%% file: main.tex
\documentclass[acmlarge]{acmart}

\usepackage{algorithm}
\usepackage{algpseudocode}
\usepackage{graphicx,subcaption}

\usepackage[utf8]{inputenc} 
\usepackage[T1]{fontenc}    
\usepackage{hyperref}       
\usepackage{url}            
\usepackage{booktabs}       
\usepackage{amsfonts}       
\usepackage{nicefrac}       
\usepackage{microtype}      
\usepackage{xcolor}         
\usepackage{amsmath}
\usepackage{amsthm}

\usepackage{prettyref}
\newrefformat{fig}{Figure~[\ref{#1}]}

\usepackage[mathscr]{euscript}

\newrefformat{app}{\ref{#1}}
\newrefformat{alg}{\ref{#1}}

\begin{document}

\title{ZeroSwap: Data-driven Optimal Market Making in DeFi}



\author{Viraj Nadkarni}
\affiliation{%
  \institution{Princeton University}
  \country{USA}
  }
\email{viraj@princeton.edu}

\author{Jiachen Hu}
\affiliation{%
  \institution{Peking University}
  \country{China}
  }
\email{nickh@pku.edu.cn}

\author{Ranvir Rana}
\affiliation{%
  \institution{Witness Chain}
  \country{USA}
  }

\author{Chi Jin}
\affiliation{%
  \institution{Princeton University}
  \country{USA}
  }
\email{chij@princeton.edu}

\author{Sanjeev Kulkarni}
\affiliation{%
  \institution{Princeton University}
  \country{USA}
  }
\email{kulkarni@princeton.edu}

\author{Pramod Viswanath}
\affiliation{%
  \institution{Princeton University}
  \country{USA}
  }
\email{pramodv@princeton.edu}

\begin{abstract}

Automated Market Makers (AMMs) are major centers of matching liquidity supply and demand in Decentralized Finance. Their functioning relies primarily on the presence of liquidity providers (LPs) incentivized to invest their assets into a liquidity pool. However, the prices at which a pooled asset is traded is often more stale than the prices on centralized and more liquid exchanges. This leads to the LPs suffering losses to arbitrage. This problem is addressed by adapting market prices to trader behavior, captured via the classical market microstructure model of Glosten and Milgrom. In this paper, we propose the first optimal Bayesian and the first model-free 
data-driven algorithm to optimally track the external price of the asset. The notion of optimality that we use enforces a zero-profit condition on the prices of the market maker, hence the name \textit{ZeroSwap}. This ensures that the market maker balances losses to \textit{informed traders} with profits from {\em noise traders}. The key property of our approach  is the ability to estimate the external market price {\em without} the need for price oracles or loss oracles. Our theoretical guarantees on the performance of both these algorithms, ensuring the stability and convergence of their price recommendations, are of independent interest in the theory of reinforcement learning. We empirically demonstrate the robustness of our algorithms to changing market conditions.
\end{abstract}

\maketitle

\section{Introduction}
\label{sec:intro}

\input{intro}

\section{Price and trader behaviour model}
\label{sec:model}
\input{model}

\section{Data-driven algorithm for unknown parameters}
\label{sec:results}

\input{results}

\section{System design for an on-chain implementation}
\label{sec:implementation}
\input{implementation}

\section{Discussion}
\label{sec:conclusion}
\input{conclusion}

\section*{Acknowledgements}

The authors thank David Krohn for fruitful discussions leading to the problem formulation. This work was supported by the National Science Foundation via grants CCF-1705007, CNS-2325477, the Army Research Office via grant W911NF2310147, a grant from C3.AI and a gift from XinFin Private Limited.

\bibliographystyle{plain}
\bibliography{references}

\appendix

\input{appendix}


\end{document}

%% file: intro.tex
Market making is an essential service that is used to satisfy liquidity demand in any financial system. Efficient market making in traditional finance involves providing bid and ask quotes for an asset that are as close to each other as possible, while accurately reflecting the price of the asset on a limit order book and thus providing a slight profit to the market maker. This efficiency is achieved by using increasingly complex models for trader behavior \cite{GLOSTEN198571,kyleModel,grossmanMiller} and then inferring the underlying hidden price from the observed trader behavior. These models have now become canonical knowledge in the domains of microeconomics and market microstructure.

More recently, the problem of market making has come to the forefront in Decentralized Finance. DeFi uses Automated Market Makers, specifically Constant Function Market Makers (CFMMs) \cite{angeris2021constant}, as an alternative to limit order books. This decreases the computational load required in satisfying trades, while also providing deep markets for infrequently traded tokens. Additionally, the liquidity required to satisfy trades is also provided in a decentralized manner by liquidity providers (LPs), who pool their tokens to help satisfy the liquidity demand. 

\color{black} Markets in DeFi come with a variety of differing characteristics. The main differences are along market depth (or liquidity), and price volatility. For instance, stablecoin trading volume ($\approx \$ 11.1$ trillion) recently surpassed that of the amount transacted via centralized services such as MasterCard and PayPal \cite{stablecoin3}. Some of the deepest markets in terms liquidity also happen to be those containing stablecoins \cite{univ3}. These markets also do not face much volatility, and their depth ensures that the price impact of retail trades is small. On the other hand, DeFi also has hundreds of tokens that do not trade frequently and hence need liquid markets. Because of lack of liquidity, such markets are volatile and the price is sensitive even to retail trades. In this work, we focus on the former type of DeFi market.

The main problem that a CFMM faces is to incentivize the LPs to pool their tokens. To do that, the CFMM needs to ensure they do not face losses on an average. However, it is common knowledge that LPs do indeed face various kinds of losses due to changing reserves \cite{loesch2021impermanent} and their lack of information about the market conditions \cite{milionis2022automated}. In the current paper, we focus on mitigating the loss that arises due to this lack of information. In particular, CFMMs with static curves often lead to LPs suffering losses due to arbitrageurs. These losses are supposed to be compensated with the fees charged on each trade. This is because centralized exchanges are characterized by high liquidity and trading volume, and lesser fees. For instance the daily trading volume on the centralized exchange Binance is around \$ 15 B, much larger compared to a volume of \$ 1.1 B on the largest decentralized exchange Uniswap \cite{tradeVolComparisons}. A less liquid exchange like Uniswap gives rise to a staler price, and hence is prone to loss due to arbitrage.
\color{black}

This arbitrage loss can be quantified for a special case as loss-versus-rebalancing (LVR) \cite{milionis2022automated}, and persists even after the introduction of fees \cite{milionis2023automated}. For a general market maker that offers to sell and buy a risky asset at the ask and bid prices respectively, the arbitrage loss is specified with respect to the external true price of the asset. The arbitrageur does a buy trade when the external price exceeds the ask price and does a sell trade when it is below the bid price. The loss to the market maker can be quantified as the difference between the two prices times the amount of asset being traded.

In traditional finance, this arbitrage loss has been modeled as the \textit{adverse selection} cost arising because of interaction with \textit{informed traders} (traders that know the external price - same as arbitrageurs). An optimal market maker exactly balances this cost with the profits arising from interaction with \textit{uninformed} traders (or \textit{noise} traders). This condition for optimality was first proposed by Glosten and Milgrom \cite{GLOSTEN198571}. In DeFi, this classification of traders has been termed as \textit{toxic} and \textit{non-toxic} order flow corresponding to informed and uninformed trade respectively \cite{nezlobinToxicity,crocToxicity}.

For CFMMs, this loss also stems from the fact that it needs to incentivize a trader to truthfully indicate their belief about the price, via their trades. This implies that, to track the external price accurately, the CFMM ends up paying the informed traders, in return for their information. This fact is also apparent from the connection between CFMMs and information eliciting market scoring rules used in prediction markets \cite{frongillo2023axiomatic}.

Naively, the loss to arbitrageurs can be minimized by simply setting the marginal price to be equal to the external price. This would need access to a \textit{price oracle}, and has indeed been tried in some market making protocols \cite{dodoOracle}. However, coupling a market maker to an oracle opens the door to frontrunning attacks \cite{frontrunningOracles} and places trust in an external centralized entity \cite{Eskandari_2021,oraclesProblem} that may itself be manipulated. To avoid this, the main constraint we impose is the absence of any access to oracles. The challenge is to infer the hidden price simply by observing the history of trades, in as sample-efficient manner as possible.

The current work formulates this challenge using the Glosten-Milgrom model of trader behaviour, which specifies the proportion of informed and uninformed traders, their sequential interaction with the market maker and the evolution of the external price. We further assume that the jumps in the hidden price happen on the same time scale as the trades, and that the jump sizes are limited, since we focus on liquid and less volatile markets such as those involving stablecoins.\color{black}  The objective of the market maker is to adaptively set the ask and bid prices so that the loss to arbitrageurs is as close to zero as possible, which is why we call the market maker \textit{ZeroSwap}. The market maker turning a profit would be undesirable since this would allow a competitor to undercut its prices and take away their order flow. In other words, the market maker should quote an efficient and competitive market price, given only the information it has in form of the trading history. Keeping this objective in mind, we make the following key contributions:

\noindent\textbf{Model-based Bayesian algorithm}. When the parameters of the trader and price behaviour model are known, we provide a Bayesian algorithm to update the ask and bid prices. We theoretically guarantee that the bid-ask spread of this algorithm is stable in presence of trades, and converges to the external market price. We empirically demonstrate that the loss to the market maker using this algorithm is zero, and it can hence be used as a benchmark for an optimally efficient market maker  
 (\prettyref{sec:Bayes}). 

\noindent\textbf{Model-free data-driven algorithm}. When the model parameters are unknown, we design a randomized algorithm which depends only on the trade history visible to the market maker. We empirically demonstrate that it tracks the hidden external price even under rapidly changing market conditions, and incurs a similar loss as a market maker that has access to an oracle. We give a first-of-its-kind theoretical guarantee that maximizing the corresponding cumulative reward ensures that the external price is tracked by the market maker efficiently; this is of independent interest in the theory of reinforcement learning  (\prettyref{sec:results}). 

\noindent\textbf{On-chain implementation}. The market-making logic driven by the reinforcement learning engine is straightforward to implement as a smart contract on a blockchain, but entirely impractical due to the associated gas fees. We specify how to implement \textit{ZeroSwap} as an application-specific {\em rollup}, in the context of the Ethereum  blockchain  (\prettyref{sec:implementation}). The actual implementation is on-going work and outside the scope of this paper. 

\section{Related work}

In this section, we reprise  relevant literature surrounding the problem formulated in this paper. Although the motivation of the problem stems from literature studying AMMs in DeFi, our formulation derives heavily from classical works in market microstructure. The data-driven algorithm we present is motivated from canonical reinforcement learning literature.

\textbf{Automated Market Makers: }Automated Market Makers, in their most popular form as Constant Function Market Makers \cite{sokDex,mohanDexPrimer}, have been known to incentivize trades that make prices consistent with an external, more liquid market \cite{Angeris_2020}. It is also known that doing this incurs a cost to the liquidity providers of the CFMMs, and a profit to the arbitrageurs \cite{evans2021optimal,Heimbach_2022,tangri2023generalizing}. This profit can be quantified as ``loss-versus-rebalancing'' in the case of a market maker with only arbitrageurs (informed traders) trading with it \cite{milionis2022automated}, and is seen to be proportional to external price volatility. Several works propose to capture this loss, either via an on-chain auction \cite{mcmenamin2022diamonds} or using auction theory to generate a dynamic ask and bid price recommendation for an AMM \cite{milionis2023myersonian}. Another recent work \cite{goyal2023finding} proposes an optimal curve for a CFMM based on the LP beliefs over prices, however, the work does not consider a dynamic model where the trader reacts based on the market maker setting their prices. Our work is related closest to \cite{milionis2023myersonian}, where a dynamic model of trading is indeed considered, and optimal ask and bid prices are derived. However, the price recommendations require the market maker to know underlying model parameters, and thus the solution is not model-free or data-driven. Additionally, we look at a competitive market maker, while \cite{milionis2023myersonian} look at the monopolistic case. Data-driven reinforcement learning algorithms to adapt CFMMs have also been used in another recent work \cite{churiwala2022qlammp}, albeit the objective there is to control fee revenue and minimize the number of failed trades.

\textbf{Optimal market making: }The trader behavior model that we use derives from the Glosten-Milgrom model \cite{GLOSTEN198571} used extensively in market microstructure literature, however we modify it to have a continuously changing external price. Several followup works \cite{das2005learning,das2008adapting} derive optimal market making rules under a modified Glosten-Milgrom framework, but they assume that underlying model parameters are known, and that external price jumps are notified to the market maker. A more data-driven reinforcement learning approach is followed in \cite{chan2001}, but the reward function they assume contains direct information about the external hidden price, while we assume no price oracle access. Another thread of optimal market making in traditional market microstructure literature deals with inventory management \cite{hoAndStoll,avellaneda2008} as opposed to the information asymmetry between traders and market makers. However, we seek to design a market maker that covers losses from information asymmetry as in the Glosten-Milgrom model, assuming no constraints on the inventory. The Glosten-Milgrom model has already been considered for AMMs in DeFi, albeit only for single trades \cite{aoyagi2020LP,angeris2020does}. 

\textbf{POMDPs and Q-Learning: }Partially Observable Markov Decision Processes (POMDPs) are used to model decision making problems where an underlying state evolves in a Markovian manner, but is invisible to the agent. Q-learning \cite{watkins1992q} is a standard model-free method that is guaranteed to learn an optimal decision making policy for Markov Decision Processes (MDPs). Here, the optimality is in terms of maximizing expected cumulative reward. We formulate the optimal market making problem as a POMDP. We then adapt the algorithm for the POMDP defined by our model, and design a reward that helps us achieve the goal of optimal market making.

%% file: model.tex
We now describe the framework used for modeling trader behavior in response to the evolution of a hidden price process and prices set by the market maker. We also state the objective that the market maker seeks to optimize, and provide the motivation behind it. The model and the objective are based on the canonical Glosten-Milgrom model \cite{GLOSTEN198571} studied extensively in market microstructure literature. In this work, we consider a \textit{discrete} time model indexed by $t$.

\noindent\textbf{External price process:} The external price process $p_{ext}^t$ of a risky asset is assumed to follow a discrete time random walk, where probability of a jump at any $t$ is given by $\sigma$. That is, we have 
\begin{align}
    p_{ext}^{t+1} = \begin{cases}
        p_{ext}^{t} + 1 \ \ &\mathrm{w.p.\ }\sigma/2\\
         p_{ext}^{t} - 1 \ \ &\mathrm{w.p.\ }\sigma/2\\
         p_{ext}^{t} \ \ &\mathrm{w.p.\ }1-\sigma
    \end{cases}\label{eq:1}
\end{align}
This process can represent either the price of the asset in a larger and much more liquid exchange, or some underlying ``true'' value of the asset. In either case, we assume that it is hidden from the market maker. We use the same notation ($\sigma$) as the continuous-time volatility for our jump probability since they both represent a qualitative measure of the change in the external price.

\noindent\textbf{Market Maker:} The market maker publishes an \textit{ask} $p_a^t$ and a \textit{bid} $p_b^t$ price in every time slot. Any trader can respectively buy and sell the asset at these prices.

\noindent\textbf{Trade actions:} We assume that the traders arrive at a constant rate of $\lambda$. This means that for time slots which are multiples of $1/\lambda$, a trader comes in to interact with the market maker by performing an action $d_t$. It can choose to either buy ($d_t = + 1$), sell ($d_t = - 1$) or do neither ($d_t = 0$). What the trader chooses to do depends on what they believe the value of the external price is. 

\noindent\textbf{Trader behavior: }We assume two types of traders - \textit{informed} and \textit{uninformed}. The informed trader is assumed to know the external price exactly, while the uninformed trader does not know it at all. The informed trader buys a unit quantity of asset if $p_{ext}^t>p_a^t$ and sells a unit quantity of asset if $p_{ext}^t<p_b^t$, thus acting as an arbitrageur between the market maker and the external market. The uninformed trader randomly buys or sells a unit quantity of asset with equal probability. We assume that the trader arriving in time slot $t$ is informed w.p. $\alpha$ and an uninformed w.p. $1-\alpha$. We make this trader model more nuanced in \prettyref{sec:sim_res}.

\noindent\textbf{Objective:} Our objective is to design an algorithm to set ask and bid prices for the market maker, such that the expected loss with respect to the external market is minimized and the market maker stays competitive. Glosten and Milgrom \cite{GLOSTEN198571} express this objective mathematically as follows.
\begin{align}
    p_a^t &= E[p_{ext}^t|\mathcal{H}_{t-1}, d_t = +1]\label{eq:gm_price1}\\
    p_b^t &= E[p_{ext}^t|\mathcal{H}_{t-1}, d_t = -1]\label{eq:gm_price2}
\end{align}
where $\mathcal{H}_{t-1} = \langle (d_i,p_a^i,p_b^i)\rangle_{i=0}^{t-1}$ is the history of trades and prices until time $t-1$.

\noindent\textbf{Interpreting the objective: }Monetary loss of the market maker is defined as 
\begin{align}
    l_t &= (p_{ext}^t-p_a^t)\mathbf{I}_{\{d_t=+1\}} + (p_b^t-p_{ext}^t)\mathbf{I}_{\{d_t=-1\}}\label{eq:monetary_loss}
\end{align}
where $\mathbf{I}_{\{.\}}$ is the indicator function, and the loss is for a unit trade of the asset.
Setting bid and ask prices as per \prettyref{eq:gm_price1} and \prettyref{eq:gm_price2} makes the expected loss of the market makers vanish, since $E[(p_{ext}^t-p_a^t)\mathbf{I}_{\{d_t=+1\}} + (p_b^t-p_{ext}^t)\mathbf{I}_{\{d_t=-1\}}] = 0$.

The market maker can thus obtain a strictly positive profit by increasing $p_a^t$ or decreasing $p_b^t$ from their values in \prettyref{eq:gm_price1} and \prettyref{eq:gm_price2}. However, doing this would make it less competitive, since any other market maker with slightly greater bid or a slightly lesser ask would offer a better price and take away the trade volume. Although we do not explicitly model other market makers, their presence is implicit in setting prices according to \prettyref{eq:gm_price1} and \prettyref{eq:gm_price2}. These equations represent ideal conditions for capital efficiency, where both the trader gets the best price possible while the market maker avoids a loss.

Also, note that the market maker incurs a loss in \textit{every trade} made by an informed trader. Thus, to make the expected loss vanish, it should learn to set prices so that the loss to informed traders is balanced by the profit obtained from uninformed traders. The equations \prettyref{eq:gm_price1} and \prettyref{eq:gm_price2} can also be interpreted as striking this balance.

\section{Bayesian algorithm for known parameters}\label{sec:Bayes}

\subsection{Details and intuition}
First, let us suppose that the market maker knows the underlying model of price evolution and trader behavior. That is, $\sigma$ (price jump probability) and $\alpha$ (trader informedness) are known. Then, the objectives specified in \prettyref{sec:model} can be achieved by an algorithm based on tracking the market maker's belief over the external price $p_{ext}^t$ and updating these beliefs using Bayes rule after each trade. Algorithm \prettyref{alg:bayes} outlines this approach. 

\begin{algorithm}[t]
\caption{A Bayesian algorithm to set ask and bid prices}\label{alg:bayes}
\begin{algorithmic}[1]
\Require Known $\alpha,\sigma \in [0,1]$
\State $t \gets 0$
\State $T \gets$ Number of total time slots
\State Prior belief over prices $b^0(p):\mathbb{Z}\rightarrow [0,1]$
\While{$t \leq T$}
\State $b_a^t(p,p_a) \gets (\alpha \mathbf{I}_{\{p>p_a\}} + \frac{1-\alpha}{2}) b^t(p)/K_1$ \Comment{Belief over prices if incoming trade is a buy}
\State $b_b^t(p,p_b) \gets (\alpha \mathbf{I}_{\{p<p_b\}} + \frac{1-\alpha}{2}) b^t(p)/K_2$\Comment{Belief over prices if incoming trade is a sell}
\State $p_a^t \gets$ Solve($p_a = \sum_{p} p b_a^t(p,p_a)$)\Comment{Solve fixed point equation to get optimal ask price}\label{eq:up1}
\State $p_b^t \gets$ Solve($p_b = \sum_{p} p b_b^t(p,p_b)$)\Comment{Solve fixed point equation to get optimal bid price}\label{eq:up2}
\State Observe incoming trader action $d_t$
\If{$d_t = +1$}
    \State Update belief $b^{t+1}(p) \gets b_a^t(p,p_a^t)$\Comment{Belief update after a buy trade}
\ElsIf{$d_t = -1$}
    \State Update belief $b^{t+1}(p) \gets b_b^t(p,p_b^t)$\Comment{Belief update after a sell trade}
\ElsIf{$d_t = 0$}
    \State Update belief $b^{t+1}(p) \gets \alpha \mathbf{I}_{\{p_a>p>p_b\}} b^t(p)/K_3$ \Comment{Belief update after no trade}
\EndIf
\State $b^{t+1}(p) \gets (1-\sigma)b^{t+1}(p) + \frac{\sigma}{2}b^{t+1}(p-1) + \frac{\sigma}{2}b^{t+1}(p+1)$ \label{eq:jump_update}\Comment{Belief update to account for price jump}
\EndWhile
\end{algorithmic}
\end{algorithm}
Algorithm \prettyref{alg:bayes} keeps track of a belief of the market maker $b^t(p)$ over prices $p\in \mathbb{Z}$.  It then hypothesizes two other distributions $b_a^t(p,p_a)$ and $b_b^t(p,p_b)$ that represent the Bayesian posterior if the incoming trade is a buy (with the ask price being $p_a$) and a sell (and the bid price being $p_b$) respectively. Note that $K_1,K_2,K_3$ are normalizing constants for the posteriors. The optimal values of the ask and bid prices are the solutions of the fixed point equations \prettyref{eq:up1} and \prettyref{eq:up2}, where we have simply restated the conditions \prettyref{eq:gm_price1} and \prettyref{eq:gm_price2}. After the trade happens according to the optimal ask and bid prices, beliefs are updated to account for the trade and the price jump. In subsequent sections, we use this algorithm as a benchmark to compare with our model-free approach. For the purposes of our experiments, we assume that the initial price $p_0$ is known, so that the prior $b^0(p)$ is such that $b^0(p_0)=1$ and is zero everywhere else. We demonstrate empirical results in comparison to other algorithms in \prettyref{sec:sim_res}.

\subsection{Theoretical guarantees}

We first present results on the spread behaviour of the Algorithm \prettyref{alg:bayes} in the case of a single jump in the external price. The assumption under this simpler case is same as those made by Glosten-Milgrom \cite{GLOSTEN198571}, that the external price jumps only once at $t=0$, with the size of the jump being drawn from a known distribution. \color{black}The special case of a single jump is especially important in blockchains where we have a batch of trades being collected as part of a block and the AMM is supposed to use them to estimate the external price of the asset when the block is released. Our theoretical results show that the jump in the price of the asset that takes place between any two blocks can be estimated with exponentially vanishing error in the number of trades present in the successive block.\color{black}

\begin{theorem}\label{thm:1}
     Let the external price $p_{ext}\sim \mathcal{D}$ jump to the value $p_{ext}^*$ only once at $t=0$, where the distribution $\mathcal{D}$ of the jump is known to the market maker. Then the Bayesian algorithm \prettyref{alg:bayes} recommends ask and bid prices  $p_a^t,p_b^t$ such that
     \begin{align}
         \lim_{t \rightarrow \infty}p_a^t-p_b^t &= 0
     \end{align}
     where the rate of convergence is exponential in $t$. Further, we also have 
     \begin{align}
         \lim_{t \rightarrow \infty}Pr[|p_a^t - p_{ext}^*| > 0] &= 0\\
          \lim_{t \rightarrow \infty}Pr[|p_b^t - p_{ext}^*| > 0] &= 0
     \end{align}
\end{theorem}

This result guarantees that the Bayesian policy indeed approaches the correct value of the hidden external price, with its spread going to zero in the limit.
The proof of the above statement is given in \prettyref{sec:toy_proof}, where we first prove that the spread goes to zero and in \prettyref{sec:theor_gen}, where we prove that the ask and bid prices converge to the correct value. Similar results were derived by Glosten-Milgrom \cite{GLOSTEN198571}, but we further prove an exponential rate of convergence.

For the more difficult case where the external price follows a random walk as described in \prettyref{sec:model}, we provide guarantees on the expected spread behaviour for different trader arrival rates $\lambda$. To our knowledge, theoretical guarantees for this case have not been given before. We find that even for a small positive rate of arrival of traders, the spread reaches a constant steady state value, thus losing any dependence on time.

\begin{theorem}\label{thm:2}
    When the external price $p_{ext}$ follows a random walk (according to \prettyref{eq:1}) with jump probability $\sigma>0$ and a known initial value $p_{ext}^0$, the dependence of expected spread on time varies with the trader arrival rate $\lambda$ as follows :
    \begin{itemize}
        \item For $\lambda=0, \sigma>0, \alpha \in (0,1)$, we have
        \begin{align}
            E[p_a^T-p_b^T] = \Theta(\sqrt{\sigma T})
        \end{align}
        \item For $\lambda>0, \sigma>0, \alpha \in (0,1)$, we have
        \begin{align}
            E[p_a^T-p_b^T] = \Theta\left(\frac{\sigma}{\lambda}\right)
        \end{align}
        \item For $\lambda>0, \alpha \in (0,1)$ with a single jump in $p_{ext}$ to a value $p_{ext}^*$, we have
        \begin{align}
            E[p_a^T-p_b^T] = \Theta(e^{-D_{\mathrm{KL}}(\mathbf{q}||\mathbf{r})T})
        \end{align}
        where $\mathbf{q} = [\frac{1-\alpha}{2}, \frac{1+\alpha}{2}], \mathbf{r} = [\frac{1+\alpha}{2}, \frac{1-\alpha}{2}]$
    \end{itemize}
\end{theorem}

 The key intuition behind the proof of \prettyref{thm:2} is recognising that the spread signifies how precisely the market maker is able to estimate the external price. The wider the spread, the more uncertain it is about $p_{ext}^t$. When there are no trades, uncertainty only increases at a square root rate with time. This is a direct consequence of the belief update that corresponds to the price jump (Line \prettyref{eq:jump_update} in Algorithm \prettyref{alg:bayes}). On the other hand, uncertainty tends to decrease after a trade because of the new information that is obtained about the external price. Firstly, we derive the rates of spread increase and decrease in the respective cases. This gives us the first part of the theorem. Secondly, we observe that if the spread is large enough, the decrease in uncertainty after a trade is always more than the increase after a price jump. This helps us prove the second part of the theorem. 
 The last part of the theorem follows from the proof of \prettyref{thm:1} directly. 
 The detailed proof has been provided in \prettyref{sec:changing_ext}.

%% file: results.tex
In the model-free case (parameters $\alpha, \sigma$ are unknown), the only information that the market maker can use are the trades coming in. The guiding principle for our approach is that when the market maker publishes prices that align with the external market price, the expected number of buy and sell trades should be the same. Any deviation from this equilibrium suggests an imbalance in buy or sell trades. The objective is to minimize the short-term ``trade imbalance'' while retaining a minimal spread. This, we conjecture, is a good proxy for solving the explicit efficient market conditions (\prettyref{eq:gm_price1} and \prettyref{eq:gm_price2}).

For formulating the problem, we utilize the Partially Observable Markov Decision Process (POMDP) framework. A POMDP is a tuple consisting of a state space \(S\), an action space \(A\), a state transition probability distribution \(P(s' | s, a): S\times S\times A \rightarrow \mathbb{R}\), an observation space \(O\), an observation probability distribution \(O(o | s', a): O\times S\times A \rightarrow \mathbb{R}\), and a reward function \(r(s, a, s'): S\times A \times S\rightarrow \mathbb{R}\). The objective of an agent is to maximize the expected cumulative reward by choosing \textit{actions} from $A$ at each time step. The underlying \textit{state} is not directly visible to the agent, but can only be inferred through \textit{observations} that depend on the state via the observation probability. The state transition probability governs how the state evolves stochastically, given the agent's last state and chosen action.

In our scenario, the POMDP allows the market maker to derive optimal actions (prices) based on a probabilistic belief over the states, given the limited observations at hand. The problem is formulated as a POMDP with a tailored reward structure, hypothesizing a policy that relates the short-term trade imbalance to the bid-ask prices for optimal outcomes.

\subsection{POMDP formulation}
 We now define the POMDP for our case as follows. The state encapsulates the external price and a history of trader actions, and is defined as $s_t = (p_{ext}^t,d_t,d_{t-1},\cdots,d_{t-H}) \in S$. An action \(a_t = (p_a^t,p_b^t) \in A\), is the tuple of ask and bid prices to be set by the market maker. An observation \(o_t = d_t \in O\), the trader's decision, being the sole observable fragment of the state for the market maker. Further, we define the policy, \(\pi_t: (O\times A)^{t-1}\rightarrow A\), which translates all preceding trade observations and price data to an ask/bid price pairing. The state and observation probabilities derived from the Glosten-Milgrom model in \prettyref{sec:model}, affirming its Markovian nature. The reward function now remains to specified, which we do in the following section.

\subsection{Reward design}\label{sec:rew_des}

The primary objective of the market maker is to find an algorithm that maximizes the expected cumulative discounted reward $E[\sum_{t=0}^\infty \gamma^t r_t]$, where $\gamma=0.99$ is a discount factor. We hypothesize that an algorithm achieving this would effectively track the concealed external price.

We now break down the individual components of the reward. To promote balanced trading, we define the \textit{trade imbalance} $n_t$ as $n_t = \sum_{\tau = t-H}^t d_{\tau}$.
Here, $H$ signifies a constant window size over which this trade imbalance is calculated. Rewarding the agent with $-n_{t}^2$ encourages a balance between the number of buy and sell trades, acting as an indirect indicator of properly tracking the external price. 

Nevertheless, an agent could easily exploit this reward by setting $p_a^t = \infty$ and $p_b^t = -\infty$, since this ensures no informed trades and only uninformed trades, maintaining the trade imbalance close to $0$. This would give us a market maker that balances the trades well, but is not competitive at all. Thus, it becomes crucial to penalize the agent for a wide spread, leading to the reward formulation:
\begin{align}\label{eq:reward}
    r_t(s_t,a_t) = - n_{t}^2 - \mu (p_a^t-p_b^t)^2
\end{align}
where $\mu$ is a constant.

Still, there exists a potential for the agent to exploit this reward by alternating its values between $p_a^t=p_b^t=\infty$ and $p_a^t=p_b^t=-\infty$ in every iteration. This would render the spread zero, attracting only uninformed traders. To counteract this, we enforce a limitation on the algorithm design, allowing it to only output finite \textit{changes in ask and bid prices} ($|p_a^t-p_a^{t-1}|, |p_b^t-p_b^{t-1}| < \Delta_{max}$, given a positive and finite $\Delta_{max}$), rather than determining the ask and bid prices directly.

\subsection{Algorithm design and intuition}\label{sec:qlalgo}
\color{black}
Algorithm \prettyref{alg:qt} shows the method to set prices in this model-free setting. It is based on the tabular Q-learning algorithm developed for MDPs \cite{watkins1992q}. We first explain why it works in the case of MDPs and argue why it is also a reasonable algorithm for our case. The key intuition in the algorithm is to keep track of a table $Q(n,a)$, where the rows represent values of trade imbalance $n$ and the column represent an action that consists of a tuple $a=(a_1,a_2)$. These are not the ask and bid prices directly, but they represent the \textit{changes in the mid price and spread} respectively. The ask and bid prices are themselves derived from the mid price and spread as shown on lines \prettyref{eq:a1} and \prettyref{eq:a2}. Each $Q(n,a)$ is supposed to be the market maker's best estimate of the expected future cumulative reward, starting with an imbalance $n$, and performing an action $a$. More formally, $Q(n,a)$ estimates $r_0(n,a) + \max_{\langle a_t\rangle} E[\sum_{t=1}^\infty Q(n_t,a_t)]$. To do this, the fixed point update equation shown on line \prettyref{eq:fpu} is followed, where $\lambda$ represents the learning rate of the algorithm. Since the algorithm obtains better estimates as time goes on, it can use those estimate to follow the optimal policy $a = \arg\max_{a'}Q(n_t,a')$ more confidently. Thus, we have a parameter $\epsilon$ that controls how many random actions are sampled for ``exploration'' as opposed to ``exploitation''. Making the probability of exploration decay with time ensures more exploitation as data is accumulated and more exploration earlier on.

The above intuition would be sufficient had the trader behavior and external price were completely deterministic one-to-one functions of the imbalance (in other words, if the problem was an MDP). But because of noise trading and the stochastic nature of the price jumps, the trade imbalance is a noisy observation of the underlying hidden external price. However, the presence of informed traders is what couples the noisy observation with the hidden price. Therefore, we conjecture, that given a non-zero informed trader proportion, it should be possible to infer the underlying external price just by observing the trade imbalance, and more generally, the trading history. While making this conjecture, we assume that the price jumps and the changes in ask and bid prices that the market maker can make are of the same scale. We explore what happens if these assumptions are violated in \prettyref{sec:limits}. 
\color{black}
\begin{algorithm}[t]
\caption{A reinforcement learning algorithm for setting ask and bid prices}\label{alg:qt}
\begin{algorithmic}[1]
\Require Algorithm parameters $\epsilon,\lambda,\mu,\gamma \in [0,1]$
\State $t \gets 0$
\State $T \gets$ Number of total time slots
\State Initialize $Q(n,a) = 0$, where $n \in \{-H,\cdots,H\}, a \in \{-1,0,+1\}\times\{-1,0,+1\}$
\State Initialize $p_m^0 = p_{ext}^0, \delta_0 = 0$
\State Initialize $n_0 = 0$
\While{$t \leq T$}
\State Sample uniformly $u \gets [0,1]$
\If{$u < \epsilon^t$}\Comment{Exploration - probability decays with time}
    \State Choose $a_t = (a_1,a_2)$ uniformly at random
\Else \Comment{Exploitation}
    \State Choose $a_t = (a_1,a_2) = \arg\max_{a'} Q(n_t,a')$\label{eq:argmax}
\EndIf
\State $p_m^t \gets p_m^{t-1}+a_1$ 
\State $\delta_t \gets \delta_{t-1}+a_2$ 
\State Set ask price $p_a^t \gets p_m^t + \delta_t$\label{eq:a1}
\State Set bid price $p_b^t \gets p_m^t - \delta_t$\label{eq:a2}
\State Observe trader action $d_{t+1}$
\State Set imbalance $n_{t+1} \gets \sum_{\tau = t+1-H}^{t+1} d_{\tau}$
\State Reward $r_t \gets n_{t}^2  - \mu (p_a^t-p_b^t)^2$
\State Update $Q(n_t,a_t)\gets Q(n_t,a_t) + \lambda (r_t + \gamma\max_{a} Q(n_t,a) - Q(n_t,a_t))$\Comment{Update estimate of expected future value of the imbalance}\label{eq:fpu}
\EndWhile
\end{algorithmic}
\end{algorithm}

\subsection{Theoretical guarantees}

The reward formulation above was based on the intuition of tracking the external price using the trade imbalance as a signal while maintaining a reasonable spread. We now justify the exact form of the reward by providing guarantees on the performance of the optimal RL policy that maximizes this reward. We do this for the simpler case of a single jump in price at $t=0$ (same as \prettyref{thm:1}).
\begin{theorem}\label{thm:3}
    In the case of a single jump in the external price $p_{ext}$ to the value $p_{ext}^*$ at $t=0$, for some constant $C$ that depends on the parameters $\alpha, H$, the optimal policy corresponding to the reward function $r_t$ \prettyref{eq:reward} is such that
    $$ R^{\pi^*} \lesssim C R^{\pi_{B}} $$
    where $\pi^*,\pi_B$ represent the optimal policy and the Bayesian policy respectively, and $R^{\pi} = E[\sum_{i=1}^T \rho_t]$ where $\rho_t := (p_{ext}^*-p_t^a)^2 + (p_{ext}^*-p_t^b)^2 $. 
\end{theorem}

The above result implies that the spread induced by the optimal policy maximizing the reward as defined in \prettyref{eq:reward} incurs an expected squared deviation from the external price that is at most a constant multiple of the same squared deviation of the Bayesian policy discussed in \prettyref{sec:Bayes}.

The key intuition behind this result is defining a \textit{risk function} $R^\pi$ that the Bayesian policy minimizes implicitly. The objective then is to prove that the risk of the optimal RL policy is only a constant multiple of the risk of the Bayesian policy. This is done by establishing a relation between the risk and the expected cumulative reward of any policy. It turns out that a policy with a higher risk has a lower cumulative reward and vice versa. Thus, we effectively show that the reward proposed in \prettyref{eq:reward} is a indeed a good proxy for the risk, which is just the squared deviation of the external price from the price at which trades occur. The full proof is given in \prettyref{sec:rl_proof}. Combining the above result with \prettyref{thm:1} immediately gives us the following corollary.
\begin{corollary}\label{cor:1}
    In the case of a single jump in the external price to the value $p_{ext}^*$ at $t=0$, the optimal policy $\pi^*$ for maximizing the reward \prettyref{eq:reward} is such that 
    \begin{align}
        \lim_{t\rightarrow\infty}|p_a^t-p_b^t| = 0\label{eq:14}
    \end{align}
    where $p_a^t,p_b^t$ are the ask and bid prices recommended by $\pi^*$. Further, the rate of convergence in \prettyref{eq:14} is exponential.
\end{corollary}

\subsection{Comparison with static curves}
When we compare how the price difference of a static curve with the external market behaves in response to trades, we see that using Algorithms \prettyref{alg:bayes} and \prettyref{alg:qt} indeed provides a major advantage. While these algorithms guarantee an exponential decay in the price difference with the number of trades $T$ (i.e. $\sim e^{-kT}$), the error goes down only as $1/T^2$ in CPMMs used in Uniswap v2 \cite{mohanDexPrimer}. 

%
\subsection{Simulation results}\label{sec:sim_res}
In this section, we test the algorithms \prettyref{alg:qt} and \prettyref{alg:bayes} on the model described in \prettyref{sec:model}. We demonstrate their robustness to market scenarios and compare their performance with previous work\footnote{All code used for simulation in this section and Appendix \prettyref{app:sim_res} can be viewed anonymously at \href{https://anonymous.4open.science/r/ZeroSwap-FC64}{\textit{this link}}}.

\noindent\textbf{Fixed market conditions}. First, we fix different values of $\alpha$ and $\sigma$, and see how well the hidden external price is tracked by the algorithms. The key metrics used to compare their performances is the deviation of the mid-price ($= \frac{p_{ask}+p_{bid}}{2}$) from the external price and the bid-ask spread. One such example, for $\alpha = 0.9, \sigma = 0.5$, is shown in the \prettyref{fig:vanilla}. Note that the algorithm learns to track the external price completely online, without any prior training required.
\begin{figure}[h]
\begin{subfigure}{0.30\textwidth}
  \centering
\hspace*{-0.4in}
 \includegraphics[width=1.25\textwidth]{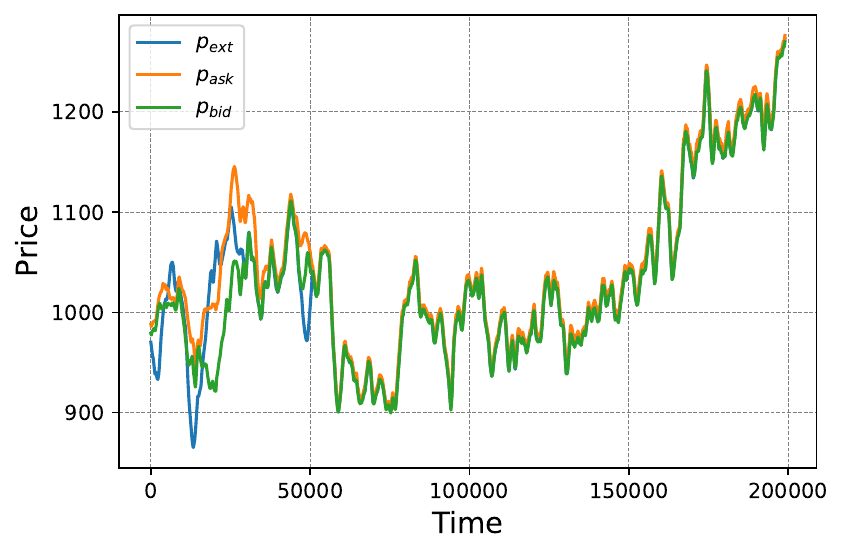}
  \caption{$p_{ask}$, $p_{bid}$ and $p_{ext}$}
  \label{fig:vanilla_ab}
\end{subfigure}
\hfill
\begin{subfigure}{0.30\textwidth}
  \centering 
  \hspace*{-0.25in}
 \includegraphics[width=1.25\textwidth]{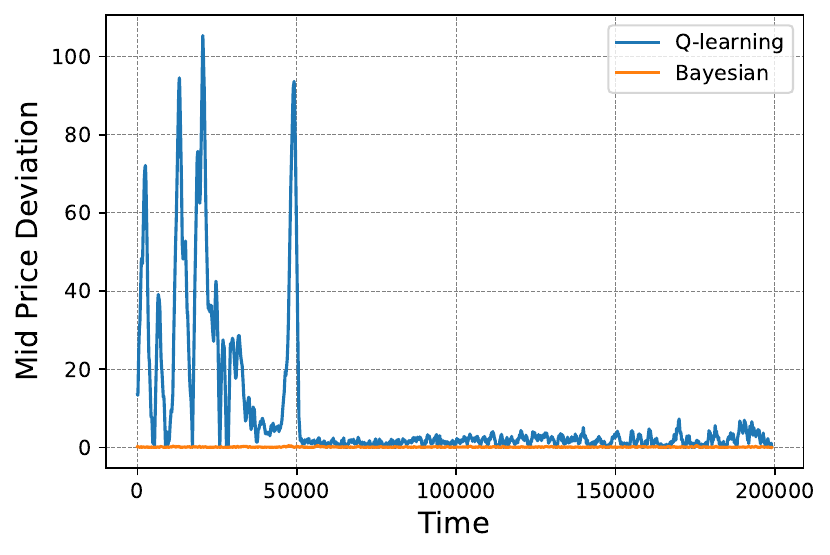} 
  \caption{Mid price deviation}
  \label{fig:vanilla_mid}
\end{subfigure}
\hfill
\begin{subfigure}{0.30\textwidth}
  \centering
  \hspace*{-0.1in}
  \includegraphics[width=1.25\textwidth]{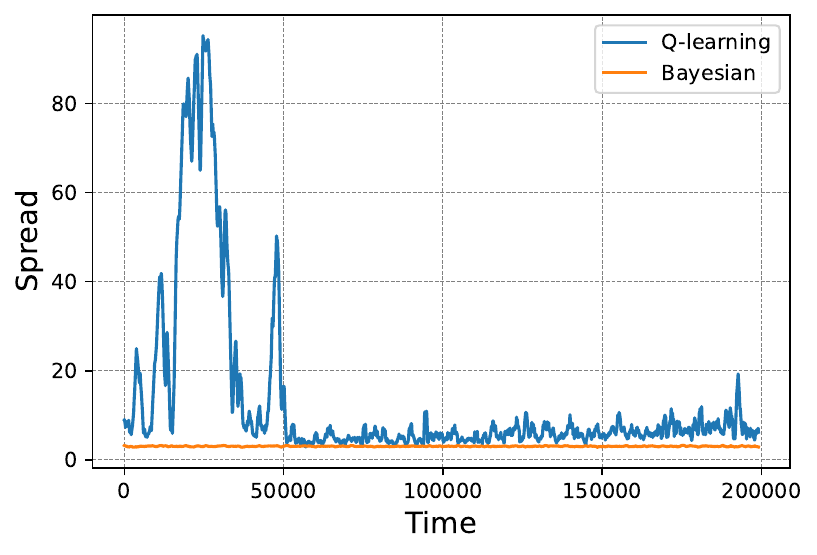}
  \caption{Spread}
  \label{fig:vanilla_sp}
\end{subfigure}
\caption{The conjectured reward for the model-free algorithm trains the agent to track the external hidden price, eventually approaching the performance of the optimal Bayesian algorithm. Figure reference in \prettyref{sec:sim_res}.}
\label{fig:vanilla}
\end{figure}
\noindent\textbf{Sudden price jumps}. Secondly, we observed what happens when there is a sudden jump in the external market price.  In that case as well, as shown in \prettyref{fig:jump} in Appendix \prettyref{app:sim_res}, the algorithm tracks the external price correctly.

\noindent\textbf{Changing market conditions}. Thirdly, we check the robustness of the algorithm to changing market conditions. This is the key to verifying its model-free nature. To do that, we vary the trader informedness $\alpha$ and underlying price volatility $\sigma$ with time by making them follow a driftless random walk in the range $[0,1]$. We find that the algorithm obtains near zero spread and mid-price deviation in this situation as well, as shown in \prettyref{fig:variable}.

\begin{figure}[hbt!]

\begin{subfigure}[t]{0.45\textwidth}
  \centering
\hspace*{-0.25in}
 \includegraphics[width=1.15\textwidth]{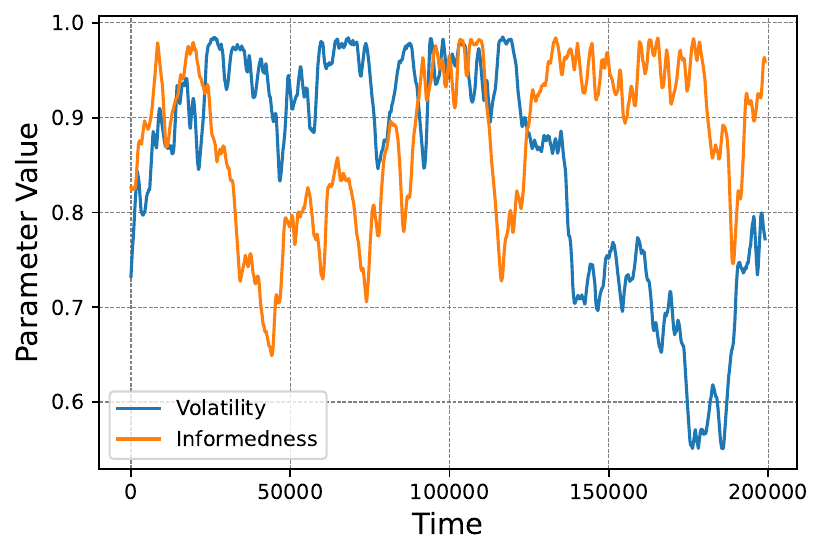}
  \caption{Erratic changes in the external market conditions : trader informedness and price volatility}
  \label{fig:variab_param}
\end{subfigure}
\hfill
\begin{subfigure}[t]{0.45\textwidth}
  \centering
\hspace*{-0.25in}
 \includegraphics[width=1.15\textwidth]{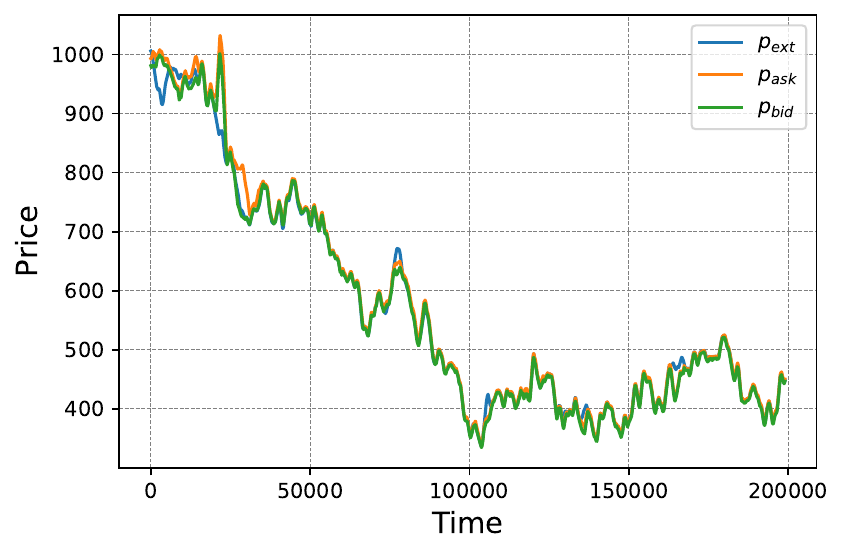}
  \caption{Ask and Bid recommended by the data-driven algorithm track the hidden external price}
  \label{fig:variab_ab}
\end{subfigure}

\caption{Even in the presence of erratic changes in the market conditions (Figure (a)), our data-driven algorithm for market making tracks the external hidden price with no prior training (Figure(b))}
\label{fig:variable}
\end{figure}

\begin{figure}[hbt!]

\begin{subfigure}[t]{0.45\textwidth}
  \centering
\hspace*{-0.25in}
 \includegraphics[width=1.2\textwidth]{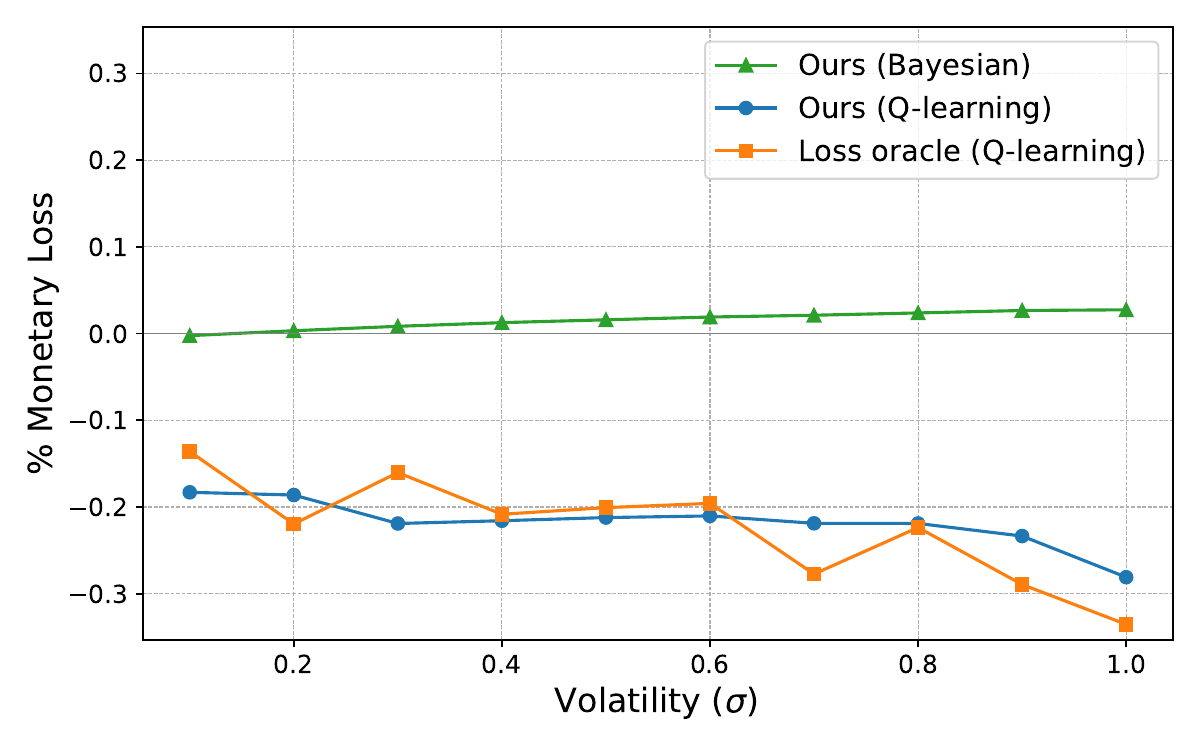}
  \caption{Percentage monetary loss per trade of our market maker is comparable with the algorithm which has access to the loss oracle}
  \label{fig:loss_avg_vs_sigma}
\end{subfigure}
\hfill
\begin{subfigure}[t]{0.45\textwidth}
  \centering
\hspace*{-0.25in}
 \includegraphics[width=1.2\textwidth]{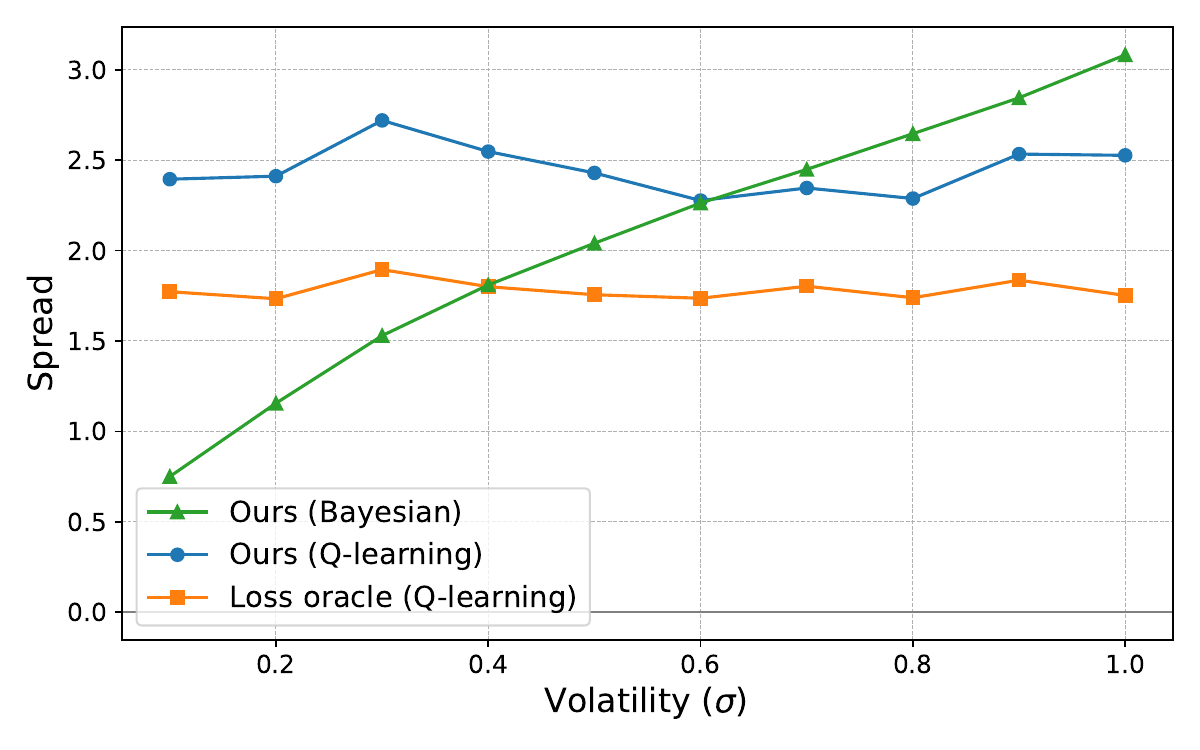}
  \caption{Larger spread observed without access to the loss oracle}
  \label{fig:spread_avg_vs_sigma}
\end{subfigure}

\caption{Algorithm \prettyref{alg:qt} gives us comparable monetary loss per trade as running the algorithm with an oracle. The Bayesian algorithm \prettyref{alg:bayes} gives loss close to zero, which is optimally efficient. All plots are averaged over values of informedness $\alpha$.}
\label{fig:monetary_loss}
\end{figure}

\noindent\textbf{Comparing monetary loss. }We compared the percentage monetary loss per trade for each of our market makers, with the algorithm in \cite{chan2001}. This work also uses Q-learning, but for a reward that has direct access to the hidden external price $p_{ext}$ and hence acts like a loss oracle. The reward function in \cite{chan2001}, is of the form $r_t = - l_t - \mu (p_a^t-p_b^t)^2$, where $l_t$ is the monetary loss as defined in \prettyref{eq:monetary_loss}. In our case, despite no access to such a loss oracle, we find that the average monetary loss per trade (around $0.2\%$) is comparable with that of \cite{chan2001} for all values of volatility $\sigma$ (\prettyref{fig:loss_avg_vs_sigma}). As expected, the Bayesian algorithm \prettyref{alg:bayes} is better than either of the others, giving us the optimally efficient zero loss. We also observe that, due to access to less information about $p_{ext}$ than \cite{chan2001}, algorithm \prettyref{alg:qt} has to resort to a larger spread (\prettyref{fig:spread_avg_vs_sigma}).

\color{black}
\noindent\textbf{Robustness of performance to block latency: }The current algorithms assume that the market maker can change the bid and ask prices immediately after every trade. However, this is only possible if the latency between blocks is lower than the time between two trades. If not, then the market maker would have to react to multiple trades in a single block. Empirically, we observe that doing this does not change the monetary loss faced by the LPs for any algorithm (\prettyref{fig:block_lat}).

\begin{figure}[hbt!]
  \centering
\hspace*{-0.4in}
 \includegraphics[width=0.75\textwidth]{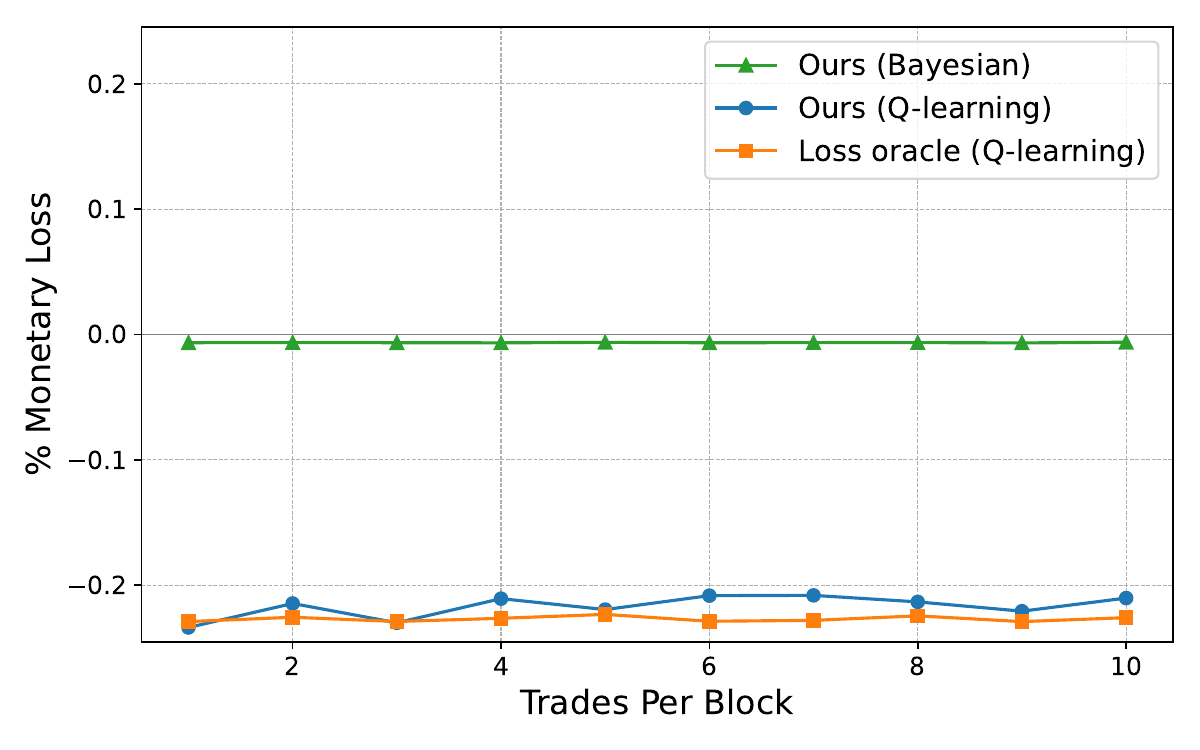}
\caption{The performance of all algorithms is robust to changes in the number of trades the algorithm processes at every time step - this is a proxy for block latency}
\label{fig:block_lat}
\end{figure}

\noindent\textbf{Augmenting the algorithm with inventory: }Assuming the same access to liquidity for the algorithms \prettyref{alg:bayes} and \prettyref{alg:qt}, we compare the monetary loss of the market makers with that of Uniswap. Instead of directly setting the ask $p_a$ and bid $p_b$ according to the operating price $p^{curve}$ of a constant product curve \cite{uniswapv2}, we set these prices as recommended by the algorithms ($p_a^{alg},p_b^{alg}$), and use the curve only as a boundary condition to avoid running out of inventory. This can be achieved by setting the ask price to be $p_{a} = max(p_a^{alg},p^{curve})$ and the bid price to be $p_{b} = min(p_b^{alg},p^{curve})$. Doing this avoids the loss to arbitrageurs and thus gives the liquidity providers a slight profit with both the algorithms we have proposed across different levels of liquidity (initial amount of the asset in the inventory). This has been shown in \prettyref{fig:liquidity_loss}.

\begin{figure}[hbt!]
  \centering
\hspace*{-0.4in}
 \includegraphics[width=0.75\textwidth]{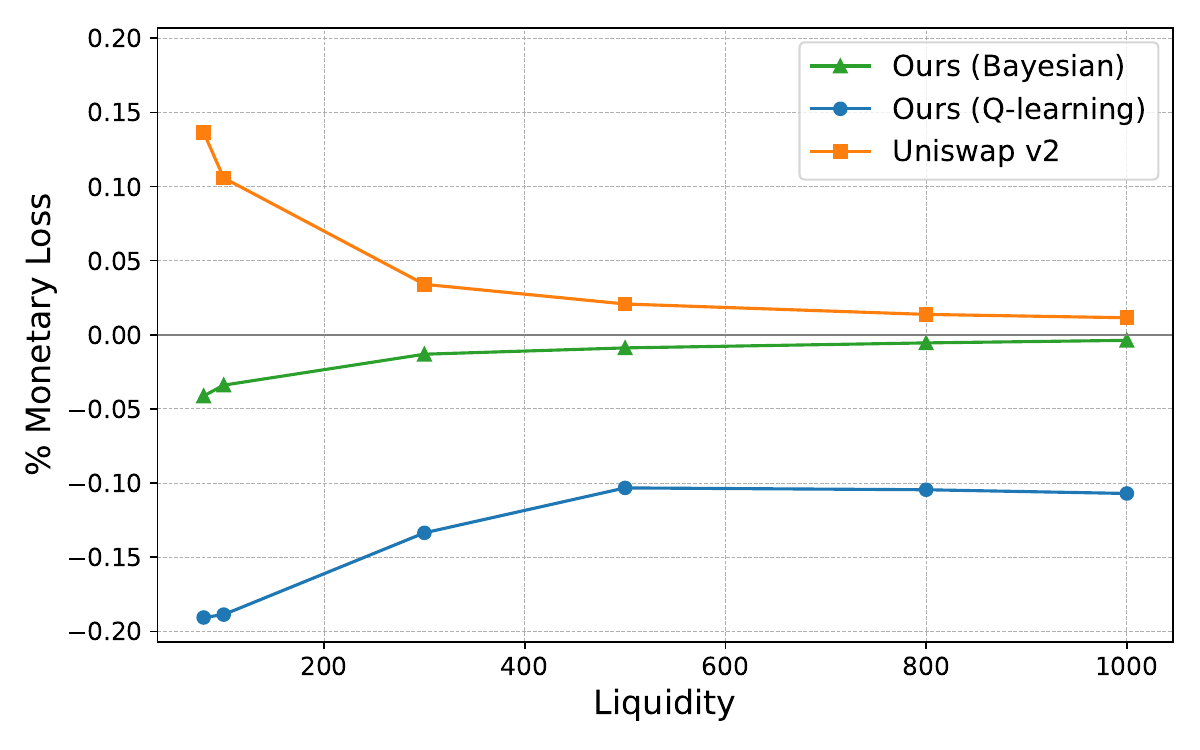}
\caption{Augmenting the constant product market maker with our algorithms avoids the arbitrage loss and incurs a slight profit to liquidity providers}
\label{fig:liquidity_loss}
\end{figure}

\subsection{Limitations of the current model}\label{sec:limits}

\begin{figure}[hbt!]

\begin{subfigure}[t]{0.3\textwidth}
  \centering
\hspace*{-0.25in}
 \includegraphics[width=1.15\textwidth]{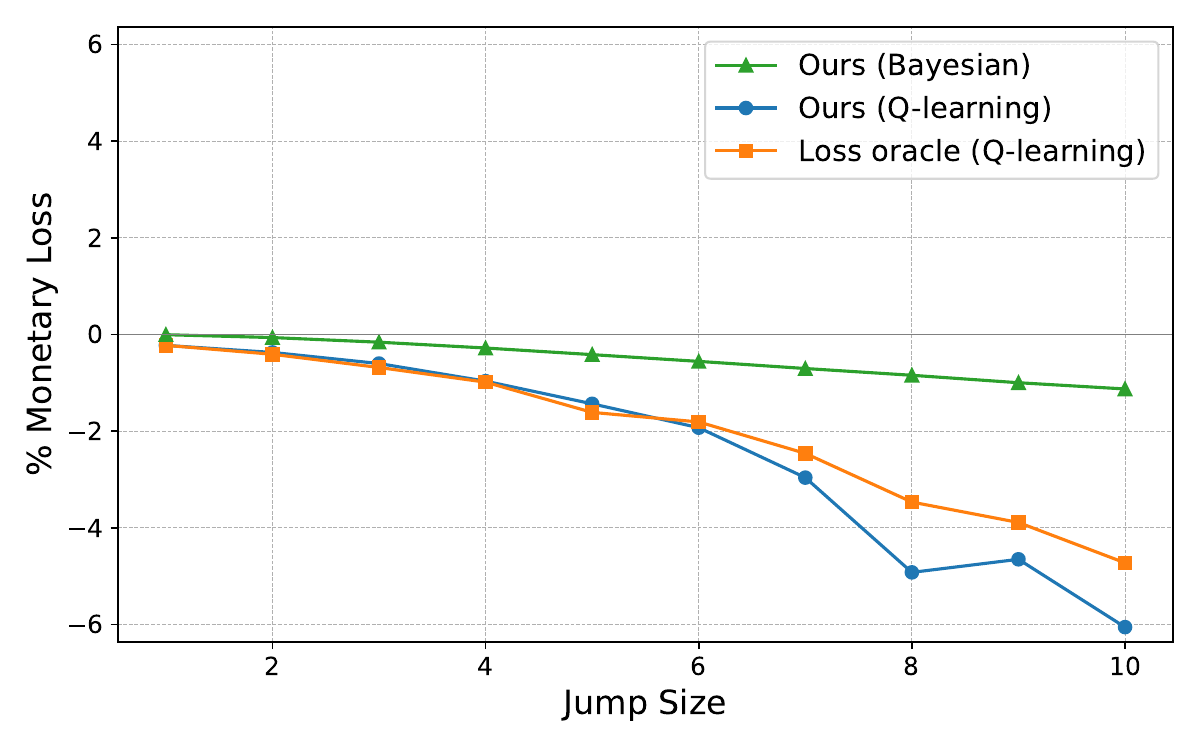}
  \caption{Increase in jump size causes inefficient pricing}
  \label{fig:lim1}
\end{subfigure}
\hfill
\begin{subfigure}[t]{0.3\textwidth}
  \centering
\hspace*{-0.25in}
 \includegraphics[width=1.15\textwidth]{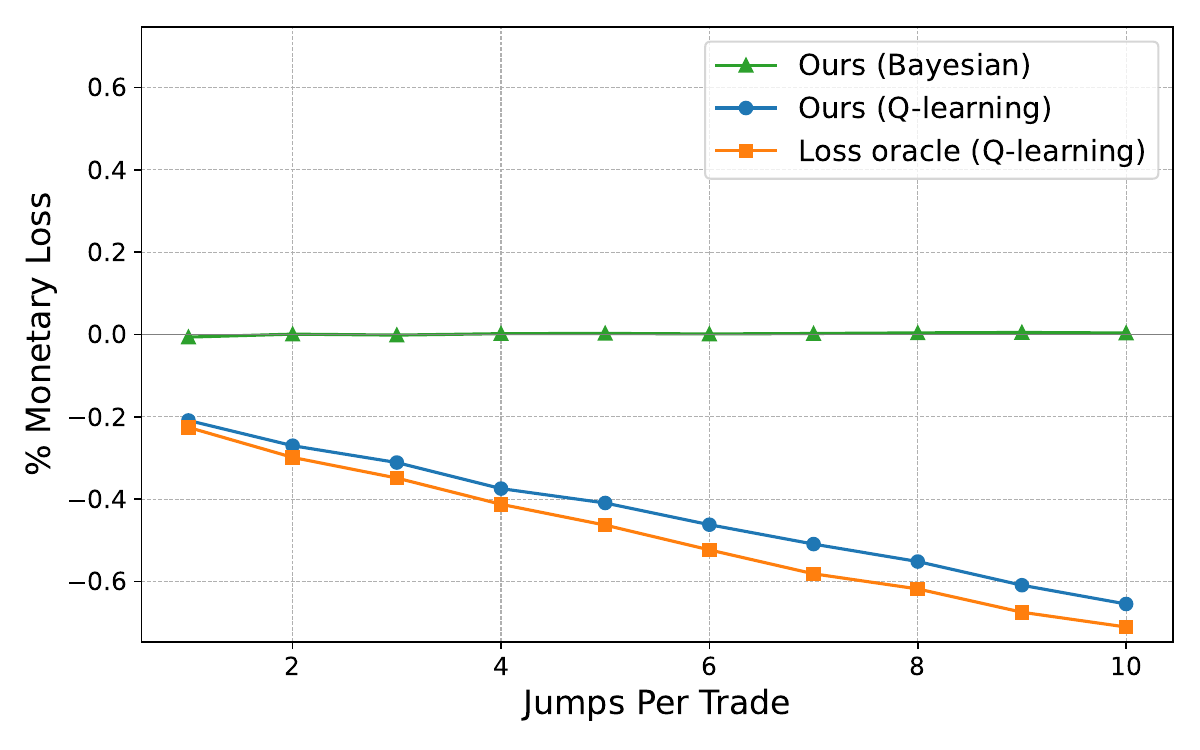}
  \caption{Increase in number of price jumps per trade causes inefficient pricing}
  \label{fig:lim1}
\end{subfigure}
\hfill
\begin{subfigure}[t]{0.3\textwidth}
  \centering
\hspace*{-0.25in}
 \includegraphics[width=1.15\textwidth]{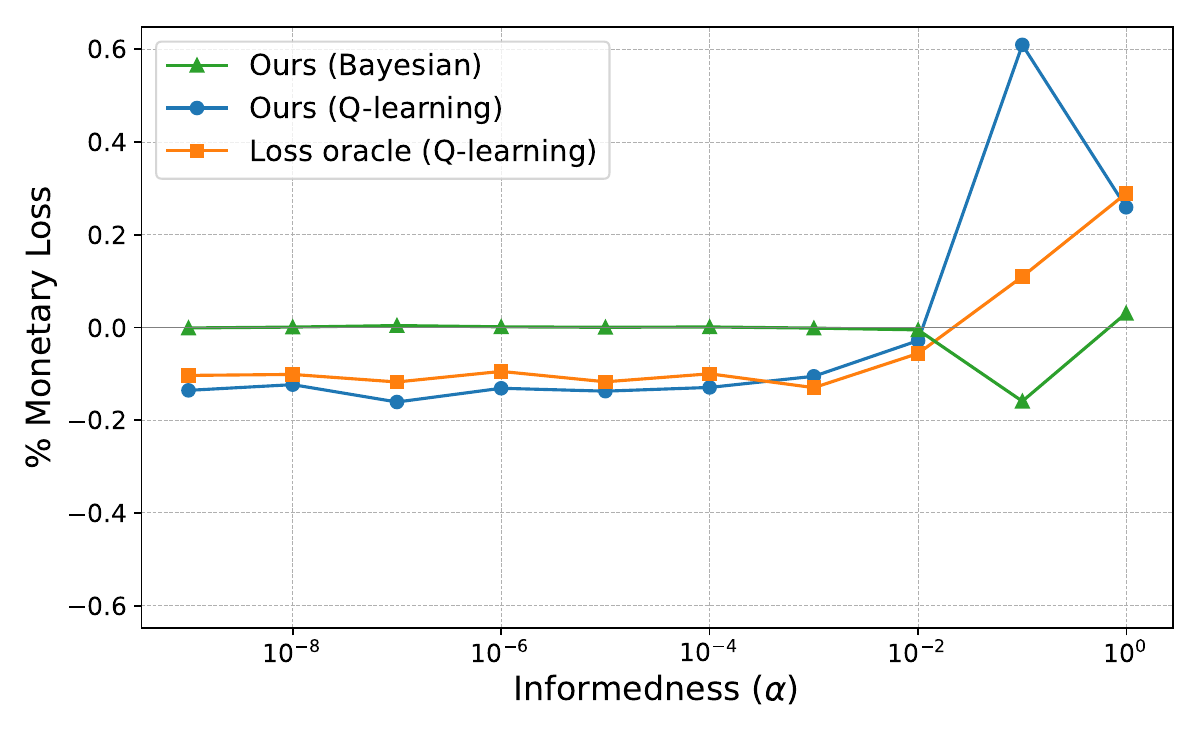}
  \caption{Algorithms are robust to changes in informedness $\alpha$}
  \label{fig:lim2}
\end{subfigure}

\caption{Limitations of the proposed algorithms}
\label{fig:limits}
\end{figure}

The main assumption in the Glosten-Milgrom model is that of the rational behavior of traders. It is possible that there are some other underlying motivations for traders and they change drastically without being known to the market maker. Even if the traders \textit{are} rational, we have assumed that the timescale of their arrival and the jumps in price are the same. However, it is possible that traders arrive at a much higher or lower frequency than jumps in price. It is also possible that the price jumps happen on a larger scale than the changes in price that the AMM algorithm is constrained to make. In all such cases, it is difficult to give theoretical or empirical guarantees on the performance of a Reinforcement Learning algorithm. These limits stem from the practical considerations of the Q-learning algorithm (such as having small action space for computational tractability). In fact, Theorem 2 guarantees that as long there is some information in the trades ($\alpha > 0$), the efficient market objective of having zero loss can always be achieved by a Bayesian market maker with a finite spread when the statistics underlying the model are known. In the most general case where the underlying distribution of price jumps and trader informedness are unknown and can be arbitrary, tracking the hidden price amounts to tracking the hidden state of a hidden markov model with unknown transition and emission statistics. This is still a fundamental open question in the study of HMMs and RL theory \cite{petrie69}. 

Having said that, in \prettyref{fig:limits}, we check the limits of the data-driven algorithm empirically. We observe there are limits to this successful tracking if we increase the jump size or the frequency of price jumps for the external hidden price. Indeed, if these variables take on larger values, the prices recommended by the algorithm become more unfair and inefficient for the traders, but stay profitable for LPs. However, the AMMs stay largely robust to changes in trader informedness.

\color{black}





%% file: implementation.tex

We propose that the data-driven market maker described in Algorithm \prettyref{alg:qt} can be implemented in a manner similar to an optimistic rollup scheme \cite{optimisticRollupsEthereum}. We now give details of the implementation, and refer to \prettyref{fig:implem} for an overview.

\noindent\textbf{Smart contract: }The main part of this implementation is the $ZeroSwap$ smart contract, which would store the latest version of the Q-table used to execute trades. The blockchain that the contract resides on can be a Layer 1 such as Ethereum, or a Layer 2 rollup, such as Arbitrum \cite{arbitrum} or Optimism \cite{optimism}. The contract performs trade execution based on \textit{solutions} posted, and also resolves any \textit{challenges} to those solutions. We explain these terms in what follows.

\noindent\textbf{Agents: }The smart contract interacts with three types of agents: \textit{traders}, \textit{validators} and \textit{challengers}. Traders wish to have their trades executed by the protocol. Validators put up stake in the protocol (i.e. lock up a specific token in the smart contract), and in return, get selected to run Algorithm \prettyref{alg:qt} off-chain and post solutions. The stake also acts as a security deposit to deter validators from misbehaving. Challengers ensure the security of the protocol by verifying the validity of the posted solutions and post challenges if they can find better solutions. 

\noindent\textbf{Trading protocol: }Firstly, traders post \textit{trade requests} as part of a block of transactions on the blockchain. These indicate their intent to buy or sell the asset from the market maker. 
Next, the {validators} collect all trade requests in a block. A chosen validator, called a \textit{proposer}, then runs one iteration of Algorithm \prettyref{alg:qt} for each trade, and posts the solution on the next block. This solution comprises of the prices at which the trades are to be executed, and an update to the on-chain Q-table. The protocol smart contract would receive this solution from a valid proposer, and execute the trades optimistically along with updating the on-chain Q-table according to the solution. We see that, in this blockchain implementation, the trades are processed in batches (corresponding to blocks). This implies that only the external price jump that happens from one block to the next matters. Thus, we have a situation where a price jump has happened, and that jump is to be inferred from a given batch of trades that was collected in a block. We know that this can be done with exponentially vanishing error (in the number of trades) as shown in Corollary \prettyref{cor:1}. \color{black}

\noindent\textbf{Challenge protocol: }The execution of trades outlined above is made secure by the presence of {challengers}. A challenger can post a challenge on-chain to be processed by the smart contract. The challenge consists of a reference to a trade request the challenger thinks was executed incorrectly, and an alternate solution pointing to the $(n_t,a)$ pair (see line \prettyref{eq:argmax} of Algorithm \prettyref{alg:qt}) that corresponds to an entry of the Q-table providing a better solution. The smart contract verifies in just one step whether this challenge is valid by querying the lookup table and comparing the values, thus checking if the $\arg\max$ operation was executed correctly by the proposer. A fault in the Q-table update on line \prettyref{eq:fpu} can be challenged in a similar manner. In this case as well, the challenger only posts an alternate solution (an $(n_t,a)$ pair) to the $\max$ operation used in the update, and the smart contract verifies its validity by looking it up in the Q-table. 

\color{black}Note that the main advantage of this protocol over existing price oracles is that the source of the data used by our algorithm is the chain itself, which is decentralized. This means that the process of challenge resolution only uses on-chain data and does not need trust in an external source. The primary objective, therefore, only boils down to recommending prices in a data driven way. \color{black}


\begin{figure}[hbt!]
  \centering
\hspace*{-0.4in}
 \includegraphics[width=0.98\textwidth]{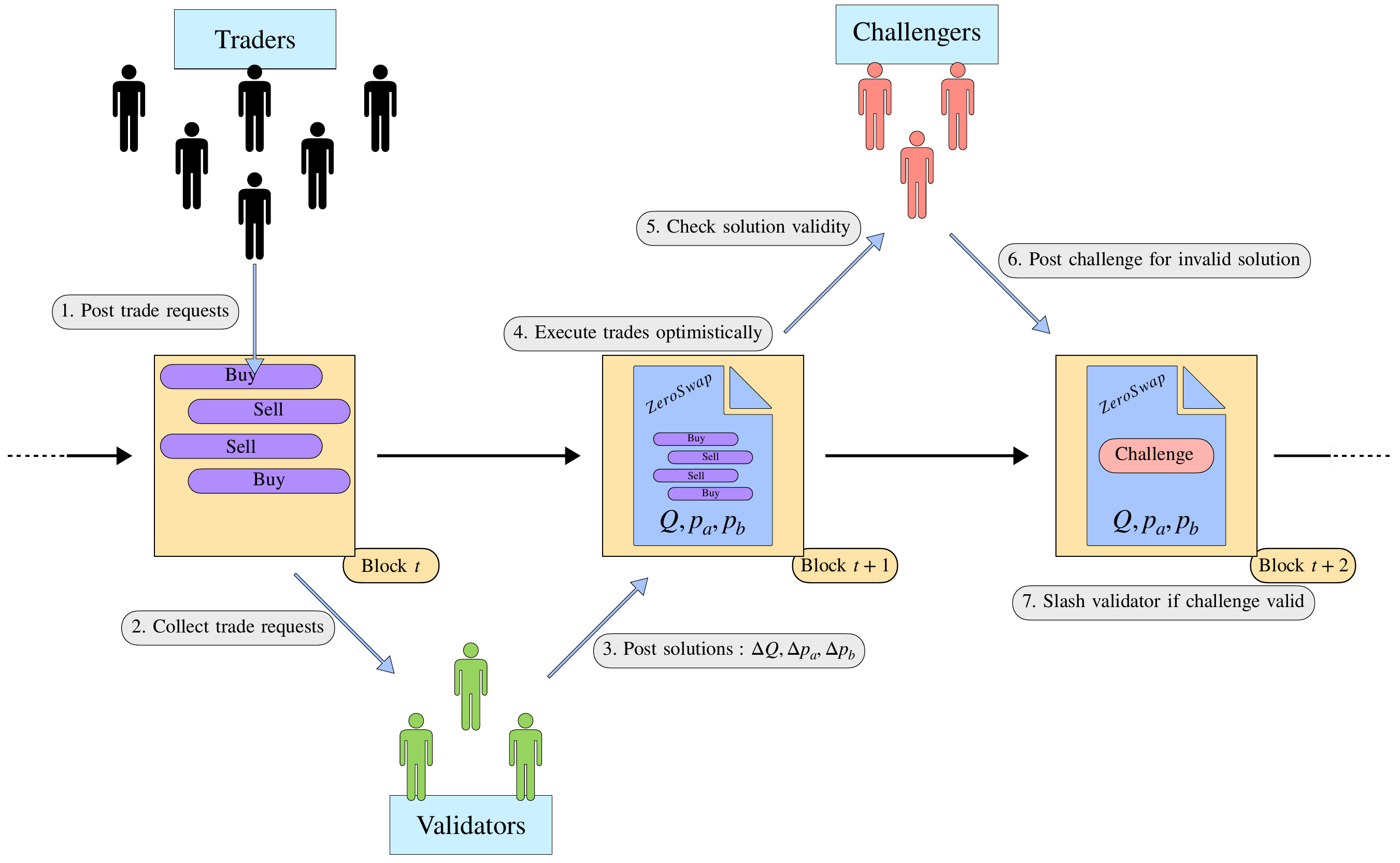}
\caption{Design for an on-chain implementation of ZeroSwap.}
\label{fig:implem}
\end{figure}

\noindent\textbf{Computational costs: }
In the smart contract, the main state variable to be tracked is the $Q-$table, which is of the size $|{O}|\times|{A}|$. In our case, the cardinality of the observation space $O$ equals the difference between the maximum and minimum trade imbalance, which is $2H$. The action space $A$ is just $\{-1,0,+1\}\times\{-1,0,+1\}$. Thus, the direct implementation of the RL model would involve tracking and updating a table with $18H$ entries. Even for a trade history size $H=100$, the $Q-$table size ($1800$) is much smaller than the liquidity vector tracked in AMMs such as Uniswap-v3 \cite{uniswapv3} for a single pool (which is of size $170,000$). Processing a single trade in Q-learning involves only modifying a single entry of the Q-table with a convex combination operation. This incurs similar computational cost as the simple addition operations required to update the liquidity vector after a trade in Uniswap-v3. Doing the $\arg \max$ operation off-chain lowers the gas usage of Q-learning considerably, since only algebraic operations remain to be done on-chain. However, note that these calculations are for only a toy model with fixed trade size and no inventory constraints.

\noindent\textbf{Incentives: }In the case that an invalid execution is detected by the smart contract through the challenges posted, the proposer whose solution was challenged would get their stake slashed. The staking infrastructure and the slashing conditions can be enforced using restaking services such as Eigenlayer \cite{eigenlayer}. This allows the validators to stake the native token of the underlying blockchain without the need to create $ZeroSwap$-specific tokens for managing incentives. 

\noindent\textbf{Potential MEV: }There is a strong incentive for the miners of the underlying blockchain to to extract MEV \cite{daian2019flash} from $ZeroSwap$. Because the Q-table and corresponding algorithm are public, the miner can calculate the optimal trade requests to frontrun and hence make a profit from other trades. This can be avoided by adding a batching operation that the validator of $ZeroSwap$ must perform before running the iteration of the Q-learning algorithm. A simple solution is to match buy and sell trades first, and then satisfy the surplus (which would be a single large buy trade or sell trade) as per the ask/bid recommendations of the algorithm \cite{cowswap,ramseyer2023augmenting}. \color{black} Furthermore, our algorithm is not susceptible to Just-In-Time liquidity attacks on LPs. This is because, for the high liquidity case we are considering, our algorithm does not take into account inventory constraints anyway so these attacks would not affect the bid and ask prices that the algorithm would recommend.\color{black}

%% file: conclusion.tex

\noindent\textbf{Exploiting information present in trades for efficiency: }In \prettyref{sec:results}, we observed that a model-free algorithm can be obtained to solve optimal market making by looking just at the incoming trades, and provided theoretical guarantees on its performance. The static curves used in AMMs today do not take into account recent past history of trades, while we propose to have a dynamically adjusting ask and bid prices to make use of the information available to us through those trades as much as possible. 

\noindent\textbf{Generalizing to variable trade size: }The model-free algorithm specified in this work does not take into account variable trade size. Our ongoing work involves deriving an algorithm for that case by using a dynamic CFMM bonding curve (instead of dynamic ask and bid prices), which is parametrized. An adaptive algorithm is then used to control the curve parameters, so as to satisfy the conditions analogous to \prettyref{eq:gm_price1} and \prettyref{eq:gm_price2} for optimal market making.

\noindent\textbf{Inventory considerations: }Our algorithms do not take into account constraints on the inventory of the market maker. The inventory preference of the LPs can be encoded as utility functions, which can be mapped to a family of bonding curves \cite{goyal2023finding,frongillo2023axiomatic}. Our method can help decide which particular curve in that family should be used to offer the most efficient price given trading history.


\noindent\textbf{Adding data-driven adaptivity to DeFi: }This work further makes a case for data-driven adaptive algorithms for DeFi in general. Future work would include bringing algorithms from reinforcement learning and stochastic control to applications such as lending protocols, treasury management and tokenomic monetary policy.

%% file: appendix.tex
\newpage

\section{Theoretical guarantees}

In this section, we provide proofs for the following results
\begin{enumerate}
    \item Bounds on the performance of the Bayesian algorithm
    
    \item Bounds on the performance of the data-driven algorithm with the conjectured reward function
    
\end{enumerate}

\subsection{Proof of \prettyref{thm:1} for a special case}\label{sec:toy_proof}



At time $t=0$, assume that the external price jumps up to $p_u$ with probability $1-\sigma$ and jumps to $p_l$ with probability $\sigma$. After this, assume that the market maker follow the Glosten-Milgrom policy using the Bayes' rule. 

At time $t$, we have a trading history $\langle d_i \rangle_{i=1}^t$ along with a history of prices $\langle (p_a^i,p_b^i) \rangle_{i=1}^t$. Define $q_i = Pr[d_i=+1|p_a^i,p_b^i,p_{ext}=p_l]$ and $r_i = Pr[d_i=+1|p_a^i,p_b^i,p_{ext}=p_u]$. Note that, given the external price and the sequence of ask and bid prices, the trades $d_i$ are independent, with their probabilities being completely determined by the GM model in terms of $\alpha$ and $\sigma$. Thus, given a sequence of trades, we have
\begin{align}
    Pr[\langle d_i \rangle_{i=1}^t | p_{ext}=p_l, ] = \prod_{i:d_i=1} q_i \prod_{i:d_i=-1} (1-q_i).
\end{align}
This gives, using Bayes' rule 
\begin{align}
    Pr[p_{ext} = p_l|\mathcal{H}_t] = \frac{\sigma \prod_{i:d_i=1} q_i \prod_{i:d_i=-1} (1-q_i)}{\sigma \prod_{i:d_i=1} q_i \prod_{i:d_i=-1} (1-q_i) + (1-\sigma)\prod_{i:d_i=1} r_i \prod_{i:d_i=-1} (1-r_i) },
\end{align}
where $\mathcal{H}_t$ consists of the history $\langle d_i \rangle_{i=1}^t, \langle (p_a^i,p_b^i) \rangle_{i=1}^t$.

The above equation tells us the market maker's posterior belief about the external price at time $t$, given the history of trades and prices. Now consider the log-likelihood ratio
\begin{align}
    \log \left(\frac{Pr[p_{ext} = p_l|\mathcal{H}_t]}{Pr[p_{ext} = p_u|\mathcal{H}_t]}\right) &= \log \left(\frac{\sigma \prod_{i:d_i=1} q_i \prod_{i:d_i=-1} (1-q_i)}{(1-\sigma)\prod_{i:d_i=1} r_i \prod_{i:d_i=-1} (1-r_i)}\right).
\end{align}

We know that $p_l \leq E[p_{ext}|\mathcal{H}_t] \leq p_u$, which implies that the GM policy always recommends bid and ask prices between $p_l$ and $p_u$. This simplifies things further, since now we have $q_i = \frac{1-\alpha}{2} := q$ and $r_i=\frac{1+\alpha}{2} := r$. This gives us
\begin{align}
    \log \left(\frac{Pr[p_{ext} = p_l|\mathcal{H}_t]}{Pr[p_{ext} = p_u|\mathcal{H}_t]}\right) &= \log \left(\frac{\sigma q^b (1-q)^s}{(1-\sigma)r^b (1-r)^s}\right)\\
    &= \log\frac{\sigma}{1-\sigma} + b \log \frac{q}{r} + s \log \frac{1-q}{1-r},
\end{align}
where $b$ and $s$ are the number of buy and sell trades respectively.

Dividing both sides by the total number of trades $b+s$ and taking the limit as $b+s \rightarrow \infty$, we get that, as $b+s\rightarrow\infty$,
\begin{align}
    \frac{1}{b+s}\log \left(\frac{Pr[p_{ext} = p_l|\mathcal{H}_t]}{Pr[p_{ext} = p_u|\mathcal{H}_t]}\right) &\longrightarrow  D_{\mathrm{KL}}(\mathbf{q}||\mathbf{r})\ \ \ \mathrm{if\ }p_{ext} = p_l\\
    &\longrightarrow  -D_{\mathrm{KL}}(\mathbf{r}||\mathbf{q})\ \ \ \mathrm{if\ }p_{ext} = p_u,
\end{align}
where $\mathbf{q} = [\frac{1-\alpha}{2}, \frac{1+\alpha}{2}], \mathbf{r} = [\frac{1+\alpha}{2}, \frac{1-\alpha}{2}]$. Thus, for a non-trivial number of informed traders ($\alpha > 0$), the KL-divergence would be always strictly positive. This implies that the posterior $(Pr[p_{ext} = p_l|\mathcal{H}_t],Pr[p_{ext} = p_u|\mathcal{H}_t])$ converges to either $(1,0)$ or $(0,1)$ depending on whether $p_{ext}$ is $p_l$ or $p_u$ respectively. This proves that both bid and ask converge to the right price as well. The only case where the convergence does not happen is when the KL divergence is exactly zero, which happens when $\alpha=0$ (no informed traders).

Suppose the actual price is $p_l$. Then, the explicit bid and ask at each time $t$ evolves as
\begin{align}
    p_a &= \frac{p_l\sigma q^{b+1}(1-q)^s + p_u (1-\sigma)r^{b+1}(1-r)^s}{\sigma q^{b+1}(1-q)^s + (1-\sigma)r^{b+1}(1-r)^s}\\
    &\rightarrow \frac{(1-\alpha) p_l + (1+\alpha)p_u  e^{-D_{\mathrm{KL}}(\mathbf{q}||\mathbf{r})t}}{(1-\alpha)+(1+\alpha)e^{-\mathrm{KL}(\mathbf{q}||\mathbf{r})t}}.
\end{align}
Similarly, the bid price evolves as 
\begin{align}
    p_b \rightarrow  \frac{(1+\alpha) p_l + (1-\alpha)p_u  e^{-D_{\mathrm{KL}}(\mathbf{q}||\mathbf{r})t}}{(1+\alpha)+(1-\alpha)e^{-\mathrm{KL}(\mathbf{q}||\mathbf{r})t}},
\end{align}
where we see that both ask $p_a$ and bid $p_b$ converge exponentially to the actual price $p_l$.

\subsection{Proof of \prettyref{thm:1} for the general case}\label{sec:theor_gen}

For a general single jump, we use propositions from \cite{GLOSTEN198571} as guidance to prove that the ask and bid converge to the true price. Assume that the true price jumps to $p_{ext}^*$, with the p.d.f. of the jump being known to the market maker as $f(p_{ext})$. We denote the ask and bid price recommendations of the Bayesian algorithm by $p_a^t, p_b^t$ respectively. 

The first lemma we prove guarantees convergence of the spread.

\begin{lemma}
    If we define $\Bar{S}_T = \frac{1}{T}\sum_{t=0}^T E[(p_a^t-p_b^t)^2]$, then for $\alpha<1$, we get
\begin{align}
    \Bar{S}_T \leq \frac{8 \mathrm{var}(p_{ext})}{(1-\alpha)^2T}.
\end{align}
\end{lemma}

\begin{proof}

First, we prove that the spread goes to zero. The key idea used here is the fact that the variance of a random variable decreases on conditioning. 

Define $p_t = E[p_{ext}|\mathcal{H}_t]$. Note that $p_t$ is a martingale w.r.t. $\mathcal{H}_{t}$, since $E[p_t|H_{t-1}] = p_{t-1}$. Thus, we have
\begin{align}
    \mathrm{var}(p_{ext}) &\geq \mathrm{var}(p_t)\\
    &= \mathrm{var}\left(\sum_{t=0}^T(p_t-p_{t-1})\right)\ \ \ \ \ \mathrm{where}\ p_0=E[p_{ext}],\ p_{-1} = 0\\
    &= \sum_{t=0}^T E[ (p_t - p_{t-1})^2] + \sum_{t=0}^T\sum_{k<t} E[(p_t-p_{t-1})(p_k-p_{k-1})]\\
    &= \sum_{t=0}^T E[ (p_t - p_{t-1})^2], 
\end{align}
since $E[(p_t-p_{t-1})(p_k-p_{k-1})] = E[E[(p_t-p_{t-1})(p_k-p_{k-1})|\mathcal{H}_{t-1}]] = E[(p_k-p_{k-1})E[(p_t-p_{t-1})|\mathcal{H}_{t-1}]] = 0$ by the martingale property.

Now, note that
\begin{align}
    (p_a^t-p_b^t)^2 &\leq 2(p_a^t-p_{t-1})^2 + 2(p_b^t-p_{t-1})^2\\
    &\leq  \frac{2(p_a^t-p_{t-1})^2}{Pr[d_t=-1|\mathcal{H}_{t-1}]} + \frac{2(p_b^t-p_{t-1})^2}{Pr[d_t=+1|\mathcal{H}_{t-1}]}\\
 &=  \frac{2(p_a^t-p_{t-1})^2Pr[d_t=+1|\mathcal{H}_{t-1}]+2(p_b^t-p_{t-1})^2Pr[d_t=-1|\mathcal{H}_{t-1}]}{Pr[d_t=-1|\mathcal{H}_{t-1}]Pr[d_t=+1|\mathcal{H}_{t-1}]}\\
 &\leq  2\frac{E[(p_t-p_{t-1})^2|\mathcal{H}_{t-1}]}{\left(\frac{1-\alpha}{2}\right)^2}.\\
 \implies E[(p_a^t-p_b^t)^2]&\leq \frac{8E[(p_t-p_{t-1})^2]}{(1-\alpha)^2}.\\
 \implies \frac{1}{T}\sum_{t=0}^T E[(p_a^t-p_b^t)^2] &\leq \frac{8 \mathrm{var}(p_{ext})}{(1-\alpha)^2T}.
\end{align}

$\qed$

\end{proof}

Next, we show that the ask and bid indeed converge to the true external price.

\begin{lemma}
    If $p_{ext}^*$ denotes the value that the external price $p_{ext}$ jumps to, then we have
\begin{align}
    Pr[|p_{ext}^* - p_a^t| \geq \epsilon] &\rightarrow 0\\
    Pr[|p_{ext}^* - p_b^t| \geq \epsilon] &\rightarrow 0,
\end{align}
as $t \rightarrow \infty$. This shows that the Bayesian policy converges to the true price eventually.
\end{lemma}
\begin{proof}
We have that 
\begin{align}
    p_a^t &= E[p_{ext}|\mathcal{H}_{t-1}, d_t = +1]\\
    &= \alpha E[p_{ext}|\mathcal{H}_{t-1}, p_{ext} > p_a^t] + (1-\alpha) p_{t-1}\\
    &\geq \alpha E[p_{ext}|\mathcal{H}_{t-1}, p_{ext} > p_{t-1}] + (1-\alpha) p_{t-1}.\\
    \implies p_a^t - p_{t-1} &\geq \alpha E[p_{ext}-p_{t-1}| \mathcal{H}_{t-1}, p_{ext} > p_{t-1}] \\
    &\geq \alpha \epsilon Pr[p_{ext}-p_{t-1}\geq \epsilon| \mathcal{H}_{t-1}, p_{ext} > p_{t-1}] \\
    &\geq \alpha \epsilon Pr[p_{ext}-p_{t-1}\geq \epsilon| \mathcal{H}_{t-1}].
\end{align}
We know that since $p_a^t \geq p_{t-1}\geq p_b^t$ and the spread goes to zero, we have $p_a^t - p_{t-1} \rightarrow 0$. Thus, for any $\epsilon>0$, we get $Pr[p_{ext}-p_{t-1}\geq \epsilon| \mathcal{H}_{t-1}]\rightarrow 0$. Similarly, we get $Pr[p_{ext}-p_{t-1}\leq -\epsilon| \mathcal{H}_{t-1}]\rightarrow 0$.

$\qed$

\end{proof}

For proving the exponential rate of decay in spread, we use equation \prettyref{eq:variance_reduction}, which shall be proven in later sections, under the assumption that $\frac{p_a^2}{\mathrm{var}(p)} \geq \epsilon > 0$ at any time. This is a valid assumption since the $E[p_a]$ is assumed to be equal to $0$ while deriving \prettyref{eq:variance_reduction}. Thus, conditioning on a buy trade would only increase the ask price. Therefore, we have that $p_a = E[p_a| d_t=+1] > 0$.

Rewriting \prettyref{eq:variance_reduction} by replacing the total variance $K_0 + \sigma T$ with $\mathrm{var}(p_T)$ gives us 
\begin{align}
    \mathrm{var}(p_{T+1}) &\leq \mathrm{var}(p_T)\left(1 - \alpha /2 \left(1-\sqrt{1-\frac{1-\alpha}{\alpha}\epsilon}\right)^2\label{eq:variance_reduction1}\right).
\end{align}
Since there are no price jumps and only trades in the time steps before $T+1$, this implies that
\begin{align}
    \mathrm{var}(p_{T+1}) &\leq \mathrm{var}(p_1)\left(1 - \alpha /2 \left(1-\sqrt{1-\frac{1-\alpha}{\alpha}\epsilon}\right)^2\label{eq:variance_reduction2}\right)^{T},
\end{align}
which confirms that the variance of the belief goes down exponentially with time. Since the expected squared spread $p_a^2$ is upper bounded by the variance, we have that the spread also decays exponentially with time.

\subsection{Proof of \prettyref{thm:2}}\label{sec:changing_ext}

When the price $p_{ext}$ follows a random walk, we observe empirically that the GM ask and bid prices manage to track $p_{ext}$ closely, but always with a non-zero spread. What we now prove is that the spread does not diverge, given that the traders bring in some useful information. We introduce a constant trader arrival rate $\lambda$, i.e. a trader arrives every $1/\lambda$ time steps.

Let $p_a^T-p_b^T = \mathcal{O}(g(T|\sigma,\alpha,\lambda))$.

In this case, we want to answer the following three questions :
\begin{enumerate}
    \item What is $g$ when $\lambda = 0$? (Empirically this is $\mathcal{O}(\sqrt{T})$)
    \item What is $g$ when $\lambda > 0$? (Empirically this is $\mathcal{O}(1)$ even for a small $\lambda$)
    \item What is $g$ when $\lambda\gg\sigma$? ($=\mathcal{O}(e^{-kT})$ from \prettyref{sec:toy_proof})
\end{enumerate}

In this section, we prove that in the absence of trades, the spread diverges at a rate of $\sqrt{T}$. Furthermore, we formalize the intuition that even sparse trading activity narrows down the support of the belief and keeps the spread from diverging.

\subsubsection{Spread behavior in the absence of trades}\label{sec:no_trad}

We derive the spread divergence rate in case of an external price following a simple random walk with jump probability $\sigma$ and no trades taking place. We denote the belief over external price at time $t$ by $f_t(.)$ 

\begin{lemma}\label{lem:3}
    In the absence of any trades, the variance of the Bayesian belief $f_t(.)$ over the external price obeys the following rule
    \begin{align}
        \mathrm{var}(f_t) = \mathrm{var}(f_{t-1}) + \sigma
    \end{align}
\end{lemma}

\begin{proof}

\begin{align}
    \mathrm{var}(f_t) &= \sum_{p} p^2 f_{t}(p)\\
    &= \sum_{p} p^2 ((1-\sigma)f_{t-1}(p) + \sigma/2 f_{t-1}(p-1) + + \sigma/2 f_{t-1}(p+1))\\
    &= (1-\sigma)\sum_{p} p^2 f_{t-1}(p) +\sigma/2\sum_p  (p-1+1)^2f_{t-1}(p-1) + \sigma/2\sum_p (p+1-1)^2 f_{t-1}(p+1)\\
    &= \left[(1-\sigma)\sum_{p} p^2 f_{t-1}(p) +\sigma/2\sum_p  (p-1)^2f_{t-1}(p-1) + \sigma/2\sum_p (p+1)^2 f_{t-1}(p+1)\right]\\
    &+ \left[\sigma/2\sum_p  f_{t-1}(p-1) + \sigma/2\sum_p  f_{t-1}(p+1)\right] \\&+ \left[\sigma/2\sum_p (p-1) f_{t-1}(p-1) + \sigma/2\sum_p  (p+1) f_{t-1}(p+1)\right]\\
    &= \left[ \mathrm{var}(f_t)\right] + \left[\sigma\right] + \left[0\right].
\end{align}

$\qed$
    
\end{proof}

We now derive an upper bound on the ask price.
\begin{align}
    E_t^2[p|d_t=+1] &\leq E_t[p^2|d_t=+1]\\
    &= \sum_p f_t(p|d_t=+1) p^2 \\
    &= \sum_p \frac{(1-\alpha)+2\alpha\mathbf{1}_{p\geq p_a^t}}{(1-\alpha)+2\alpha Pr[p\geq p_a]}f_t(p) p^2 \\
    &\leq \sum_p \frac{1+\alpha}{1-\alpha}f_t(p) p^2\\
    &= \frac{1+\alpha}{1-\alpha}\sigma t.\\
    \implies p_a^t &\leq \sqrt{\frac{1+\alpha}{1-\alpha}\sigma t}.
\end{align}
Note that the above argument is valid when $1-\alpha>0$. For $\alpha=1$, we get $p_a^t = t$.

We can get the following tighter upper bound than the above by using the definition of the ask price.

\begin{lemma}
Assume that the expected initial price $E[p_{ext}^0] = 0$. Then, the ask price of the Bayesian market maker is upper bounded as
  \begin{align}
     p_a^t &\leq \sqrt{\frac{\alpha}{2(1-\alpha)}\sigma t}.
\end{align}  
\end{lemma}
\begin{proof}

We know that
\begin{align}
    p_a^t &= E[p|d_t=+1].
\end{align}
Using the Bayes rule, and that $E[p_{ext}^0] = 0$, we can write the RHS as
\begin{align}
    p_a^t &= \frac{2\alpha}{1-\alpha+2\alpha Pr[p\geq p_a^t]}\sum_{p}\mathbf{1}_{\{p \geq p_a^t\}} p f_t(p)\\
    &\leq \frac{2\alpha}{1-\alpha+2\alpha Pr[p\geq p_a^t]}(\sum_{p}\mathbf{1}_{\{p \geq p_a^t\}}  f_t(p))^{1/2}(\sum_{p}p^2 f_t(p))^{1/2}\\
    &= \frac{2\alpha}{1-\alpha+2\alpha Pr[p\geq p_a^t]}(Pr[p\geq p_a^t])^{1/2}(\sigma t)^{1/2}.\\
    \implies (1-\alpha &+
    2\alpha Pr[p\geq p_a^t])p_a^t \leq 2\alpha\sqrt{\sigma t} \sqrt{Pr[p\geq p_a^t]}.\\
    \implies (2\alpha p_a^t) &Pr[p\geq p_a^t] - 2\alpha\sqrt{\sigma t} \sqrt{Pr[p\geq p_a^t]} + p_a^t(1-\alpha) \leq 0.\label{eq:quad_prob}
\end{align}
The roots of the quadratic expression (in $\sqrt{Pr[p\geq p_a]}$) on the LHS of \prettyref{eq:quad_prob} are
\begin{align}
  \left\{ \frac{\sqrt{\sigma t}}{2p_a^t} - \sqrt{\frac{\sigma t}{4(p_a^t)^2} - \frac{1-\alpha}{2\alpha}}, \frac{\sqrt{\sigma T}}{2p_a^t} + \sqrt{\frac{\sigma t}{4(p_a^t)^2} - \frac{1-\alpha}{2\alpha}} \right\}
\end{align}
Because the quadratic expression has a positive coefficient of the squared term, these roots must be real for the expression to ever be negative. This gives us the result.

$\qed$
\end{proof}

From the above proof, we notice that the following must also hold for the quadratic expression in $\prettyref{eq:quad_prob}$ to be $\leq 0$. We shall use this inequality in the subsequent section. In fact, we see from the above proof that we have the following lemma

\begin{lemma}\label{lem:tail_prob_bound}
If the initial belief of the market maker over the external market price is such that $E[p_{ext}^0] = 0$, then in the absence of trades, we have
    \begin{align}
    \frac{\sqrt{\mathrm{var}(p_{ext}^T)}}{2p_a} - \sqrt{\frac{\mathrm{var}(p_{ext}^T)}{4p_a^2} - \frac{1-\alpha}{2\alpha}}\leq \sqrt{Pr[p\geq p_a^T]} \leq  \frac{\sqrt{\mathrm{var}(p_{ext}^T)}}{2p_a} + \sqrt{\frac{\mathrm{var}(p_{ext}^T)}{4p_a^2} - \frac{1-\alpha}{2\alpha}} \label{eq:quadratic_bound},
\end{align}
where $\mathrm{var}(p_{ext}^T)$ is the variance of the market maker's belief at time $T$.
\end{lemma}

We now show a lower bound on the ask price, which completes the proof on the rate of growth of the spread in absence of trades.

\begin{lemma}
Assuming that the expected initial price $E[p_{ext}^0]$ is zero for simplicity. The ask price of the Bayesian market maker is lower bounded as
  \begin{align}
     p_a^t &\geq \frac{\alpha}{1+2\alpha}\sqrt{\frac{2\sigma t}{\pi}}\label{eq:lower_bound}
\end{align}  
\end{lemma}
\begin{proof}
    
\begin{align}
    p_a^t &= \frac{2\alpha}{1-\alpha+2\alpha Pr[p\geq p_a^t]}\sum_{p}\mathbf{1}_{\{p \geq p_a\}} p f_t(p).\\
   \implies p_a^t &\geq \frac{2\alpha}{1+\alpha}\sum_{p=p_a^t}^T p f_t(p)\\
    &= \frac{2\alpha}{1+\alpha}\left(\sum_{p=0}^T p f_t(p) - \sum_{p=0}^{p_a^t} p f_t(p)\right)\\
    &= \frac{2\alpha}{1+\alpha}\left(\frac{E[|p|]}{2} - \sum_{k=1}^{p_a^t} Pr[p_a^t \geq p\geq k]\right)\\
    &\geq \frac{2\alpha}{1+\alpha}\left(\frac{E[|p|]}{2} - \sum_{k=1}^{p_a^t} Pr[p\geq k]\right)\\
    &\geq \frac{2\alpha}{1+\alpha}\left(\frac{E[|p|]}{2} - \frac{p_a^t}{2}\right).\\
    \implies p_a^t &\geq \frac{\alpha}{1+2\alpha}E[|p|]. \\
    \implies p_a^t &\geq \frac{\alpha}{1+2\alpha}\sqrt{\frac{2\sigma t}{\pi}},
\end{align}
where we have used the formula for the expected absolute deviation for a simple random walk \cite{randwalk1}.

$\qed$
\end{proof}

Thus, the ask price is upper \textit{and} lower bounded by terms that grow with $\sqrt{\sigma t}$, which proves the first part of the theorem.


\subsubsection{Spread behavior in the presence of trades}\label{sec:beh_trad}

In the previous section, we proved that the spread diverges exactly at a rate of $\sqrt{T}$ when $\alpha < 1$ and at the rate of $T$ when $\alpha = 1$, in the absence of trades. In this section, we derive results on spread behavior in presence of trades. Empirically, we see that even for very sparse trading, the spread does not diverge (is $\mathcal{O}(1)$).



First, let us calculate the variance of the belief after $T+1$ time steps, where $T=1/\lambda$. That is, we have $T$ time steps where no trades occur and the $T+1^{th}$ time step where a single trade occurs. Let the ask, bid and external prices just before the trade be $p_a, p_b, p_{ext}$ respectively. We drop the $T$ superscript for simplicity. Also, assume that the variance of the initial belief distribution is $K_0$, and that the mean of the belief is $0$ for simplicity. What we aim to prove is that, if the $K_0$ is large enough, then the variance of the belief just after the trade at $T+1$ steps is less than $K_0$. Let the trade that happens at $T+1$ be denoted by $d$.

Then, the expected variance of the belief just after the last time step (when the trade happens) is given by
\begin{align}
    E[E[(p_{ext}-E[p_{ext}|d])^2|d]] &= E[(p_{ext}-E[p_{ext}|d])^2].
\end{align}
Now, we write this variance as the difference between two terms, as shown in the following lemma.
\begin{lemma}
    The expected variance of the belief just after a trade can be written as
    \begin{align}
        E[(p_{ext}-E[p_{ext}|d])^2] &= E[(p_{ext}-E[p_{ext}])^2] - E[(E[p_{ext}|d]-E[p_{ext}])^2].\label{eq:variance_formula}
    \end{align}
\end{lemma}
\begin{proof}

We first start with the expression for the variance just \textit{before} the trade, and then write it as the sum of the variance just after the trade and another non-negative term.
\begin{align}
    E[(p_{ext}-E[p_{ext}])^2] &= E[(p_{ext}-E[p_{ext}|d] + E[p_{ext}|d] - E[p_{ext}])^2]\\
    &= E[(p_{ext}-E[p_{ext}|d])^2 + E[(E[p_{ext}|d] - E[p_{ext}])^2] \\&+ 2E[(p_{ext}-E[p_{ext}|d])(E[p_{ext}|d] - E[p_{ext}])]\\
    &= E[(p_{ext}-E[p_{ext}|d])^2 + E[(E[p_{ext}|d] - E[p_{ext}])^2] \\&+ 2E[\ p_{ext}E[p_{ext}|d] + E[p_{ext}|d]E[p_{ext}] - E^2[p_{ext}|d] - p_{ext}E[p_{ext}]\ ]\label{eq:variance_proof1}.
\end{align}
The last term in the above equation can be written as
\begin{align}
 E[p_{ext}E[p_{ext}|d]] + E[E[p_{ext}|d]E[p_{ext}]] - E[E^2[p_{ext}|d]] - E[p_{ext}E[p_{ext}]].
\end{align}

 Observe that using $E[E[p_{ext}|d]] = E[p_{ext}]$, the second and fourth terms cancel out. We now group the first and third terms together, which gives us
\begin{align}
 E[p_{ext}E[p_{ext}|d]] - E[E^2[p_{ext}|d]] &= E[\ E[p_{ext}E[p_{ext}|d]|d]\ ] - E[E^2[p_{ext}|d]]\\
 &= E[E^2[p_{ext}|d]] - E[E^2[p_{ext}|d]]\\
 &= 0.
\end{align}
Thus, from \prettyref{eq:variance_proof1}, we get
\begin{align}
    E[(p_{ext}-E[p_{ext}])^2] &= E[(p_{ext}-E[p_{ext}|d])^2 + E[(E[p_{ext}|d] - E[p_{ext}])^2]. 
\end{align}

    $\qed$
\end{proof}

Now, we evaluate each of the terms on the RHS of \prettyref{eq:variance_formula}. Note that the first term is just the variance of the belief just before the trade. Since no trades have happened for the $T$ time slots, we can use \prettyref{lem:3} to get
\begin{align}
    E[(p_{ext}-E[p_{ext}])^2] &= K_0 + \sigma T\label{eq:term1}.
\end{align}
Furthermore, the second term on the RHS of \prettyref{eq:variance_formula} can be lower bounded as
\begin{align}
    E[(E[p_{ext}|d]-E[p_{ext}])^2] &\geq (p_a - 0)^2\times \left(\frac{1-\alpha}{2} + \alpha Pr[p_{ext} \geq p_a]\right) \label{eq:term2},
\end{align}
where we have one term in the expectation on the LHS, namely the case where the incoming trade $d$ is a buy order.

We now use \prettyref{eq:term1} and \prettyref{eq:term2} to write
\begin{align}
    E[(p_{ext}-E[p_{ext}|d])^2] &= E[(p_{ext}-E[p_{ext}])^2] - E[(E[p_{ext}|d]-E[p_{ext}])^2]\\
    &\leq \left[K_0 + \sigma T\right] - \left[  p_a^2 \left(\frac{1-\alpha}{2} + \alpha Pr[p_{T+1}\geq p_a]\right)  \right]\\
    &\leq K_0 + \sigma T - (1-\alpha) p_a^2/2 - \alpha/2 \left( \sqrt{\frac{K_0 +\sigma T}{2}} - \sqrt{ \frac{K_0 +\sigma T}{2} - \frac{1-\alpha} {2\alpha}p_a^2} \right)^2\label{eq:ineq2}\\
    &\leq K_0 + \sigma T - \alpha \frac{K_0 +\sigma T}{4} \left(1-\sqrt{1-\frac{1-\alpha}{\alpha}\frac{p_a^2}{K_0 +\sigma T}}\right)^2,
\end{align}
where we have used \prettyref{lem:tail_prob_bound} to obtain the inequality in \prettyref{eq:ineq2}.

Assuming the square root lower bound on the ask price as obtained in \prettyref{eq:lower_bound}, we have that $p_a \geq \beta \sqrt{K_0 + \sigma T}$, where $\beta = \frac{\alpha\sqrt{2}}{(1+2\alpha)\sqrt{\pi}}$. Substituting this, we get
\begin{align}
    E[(p_{ext}-E[p_{ext}|d])^2] &\leq K_0 + \sigma T - \alpha \frac{K_0 +\sigma T}{4} \left(1-\sqrt{1-\frac{1-\alpha}{\alpha}\beta}\right)^2.\label{eq:variance_reduction}
\end{align}
Thus, we have that
\begin{align}
E[(p_{ext}-E[p_{ext}|d])^2] &\leq K_0\ \ \ \forall K_0 \geq K^*,\\
    \mathrm{where\ }K^* = \frac{4-\alpha \gamma}{\alpha \gamma} \frac{\sigma}{\lambda} \ \mathrm{\ and\ }&\gamma = \left(1-\sqrt{1-\frac{1-\alpha}{\alpha}\beta}\right)^2,
\end{align}
where we have substituted $T = 1/\lambda$.

This can be equivalently written as
\begin{align}
    \mathrm{var}(p_{ext}^{T+1}) &\leq \mathrm{var}(p_{ext}^0) = K_0\ \ \ \forall K_0 \geq K^*.
\end{align}

Thus, the variance of the belief decreases after $T+1$ time steps when the variance of the initial belief is $> K^*$. On the other hand, we know that the variance increases after $T+1$ time steps when $K_0 = 0$. Thus, there exists some positive value of $K_0$ for which we get a ``steady-state'' constant variance, which in turn implies a constant spread.










\subsection{Proof of \prettyref{thm:3}}\label{sec:rl_proof}

We define two functions of any policy $\pi$ : the cost and the risk. The cost $J^{\pi}$ is the total negative reward accrued (where each term of the sum is of the form $n_t^2 + \mu (p_a^t-p_b^t)^2$, where $n_t$ is the windowed trade imbalance), while the risk is the actual squared deviation from the external price (of the form $(p_{ext}-p_a^t)^2 + (p_{ext}-p_b^t)^2 $). The first aim is to prove that the ratio between the risk and cost is bounded. This would imply a bound on the risk of the optimal policy (optimal w.r.t cost), which would imply that a policy that has optimal cost also has low risk. Another interesting result that might be shown here is that the cost is proportional to the ``time derivative'' of the risk.



\subsubsection{Proof for a special case}

Consider the simplest case where the external price jumps only once at the beginning, and stays constant for the next $T$ steps. The external price $p_{ext}$ jumps to $p_h$ with probability $\sigma$, and to $p_l$ with probability $1 - \sigma$. Here we assume $p_h > p_l$ and let $\Delta := p_h - p_l$.

Define the cost function at step $t$ to be the negative of the reward $c_t := -r_t = n_t^2 + \mu (p_t^a - p_t^b)^2$, and the risk function at step $t$ to be $p_t := (p_{ext}-p_t^a)^2 + (p_{ext}-p_t^b)^2 $. Let $\pi^*$ be the optimal policy of the POMDP, and $J^\pi := \mathbb{E}[\sum_t c_t], R^\pi := \mathbb{E}[\sum_t p_t]$. 

Assume the action space is limited to $(p_a, p_b) = \{(p_h, p_h), (p_l, p_l), (p_h, p_l)\}$. We categorize the three actions into three types : 

\begin{itemize}
    \item Type 1: $(p_a, p_b) = (p_h, p_l)$
    \item Type 2: $p_a = p_b = p_{ext}$
    \item Type 3: $p_a = p_b \neq p_{ext}$
\end{itemize}

Given an action sequence, we know that $d_t$ is independent from each other conditioned on the action sequence. Denote the number of actions of three types to be $k_1, k_2, k_3$ respectively. Then the expected total risk of any policy $\pi$ can be computed as 

$$R^\pi = \mathbb{E}_{\pi}\left[k_1 \Delta^2 + 2 k_3 \Delta^2\right]. $$

The computation of the expected total cost will be more difficult. Consider the definition of $n_t = \sum_{i = t - H}^t d_i$, we decompose $n_t^2$ as 

$$ n_t^2 = \sum_{i = t - H}^t d_i^2 + 2 \sum_{i < j} d_i d_j. $$

Therefore, we can check the contribution of each $d_i^2$ and $d_i d_j$ to the expected total cost as follows: 

\begin{itemize}
    \item For type I actions, we have $\mathbb{E}[d_i^2] = 1 - \alpha$, so the contribution of $d_i^2$ to $\sum_t n_t^2$ is $H(1 - \alpha)$ (ignoring some constants when $1 \leq t \leq H$ or $T - H  + 1 \leq t \leq H$). The contribution of type I action to $\mu(p_t^a - p_t^b)^2$ is $\mu \Delta^2$.
    \item For type II actions, we also have $\mathbb{E}[d_i^2] = 1 - \alpha$, so the contribution of $d_i^2$ to $\sum_t n_t^2$ is $H(1 - \alpha)$.
    \item For type III actions, we have $\mathbb{E}[d_i^2] = 1$, so the contribution of $d_i^2$ to $\sum_t n_t^2$ is $H$. The contributions of $d_i d_j$ when $j - i = l < H$ and both $i, j$ are type III actions are $(H - l + 1) \alpha^2$.
\end{itemize}

Therefore, we know the expected total cost of $\pi$ satisfies (up to some constants)

$$J^\pi \lesssim \mathbb{E}_{\pi}\left[(\mu + H(1 - \alpha)) k_1 + H(1 - \alpha)k_2 + (H + 2H^2\alpha^2)k_3\right].$$

On the other hand, we also have 

$$\mathbb{E}_{\pi}\left[(\mu + H(1 - \alpha)) k_1 + H(1 - \alpha)k_2 + Hk_3\right] \lesssim J^\pi.$$

Define $\bar{J}^\pi := J^\pi - HT(1 - \alpha)$, then 

$$\mathbb{E}_{\pi}\left[\mu k_1 + H\alpha k_3\right] \lesssim \bar{J}^\pi \lesssim \mathbb{E}_{\pi}\left[\mu k_1 + (H\alpha + 2H^2\alpha^2)k_3\right]. $$

Suppose we have a policy $\pi_{ref}$ with $R^{\pi_{ref}} = \beta$ (e.g., the Bayesian policy), then we have 

$$ \mathbb{E}_{\pi_{ref}}\left[k_1 \Delta^2 + 2 k_3 \Delta^2\right] = \beta, $$

which implies

$$ \mathbb{E}_{\pi_{ref}}[k_1 + k_3] \leq \frac{\beta}{\Delta^2}. $$

Therefore, 

$$ \bar{J}^{\pi_{ref}} \lesssim \frac{\beta \max(\mu, \alpha H (1 + 2 \alpha H))}{\Delta^2}. $$

Note that we have $\bar{J}^{\pi^*} \leq \bar{J}^{\pi_{ref}}$, so 

$$ \bar{J}^{\pi^*} \lesssim \frac{\beta \max(\mu, \alpha H (1 + 2 \alpha H))}{\Delta^2}, $$

which means 

$$ \mathbb{E}_{\pi^*}[k_1 + k_3] \lesssim \frac{\beta \max(\mu, \alpha H (1 + 2 \alpha H))}{\min(\mu, \alpha H) \Delta^2}. $$

Finally, 

$$ R^{\pi^*} \lesssim \frac{\max(\mu, \alpha H (1 + 2 \alpha H))}{\min(\mu, \alpha H)} \cdot \beta. $$

This means 

$$ \frac{R^{\pi^*}}{T} \leq \frac{\max(\mu, \alpha H (1 + 2 \alpha H))}{\min(\mu, \alpha H)} \cdot \frac{\beta}{T} + \frac{C}{T} $$

for some constant $C$. 

For the general case where the size of the initial jump is chosen from a continuous distribution, one can divide the action space of prices using an $\epsilon-$net. Doing that puts an additional constant factor which is $\mathcal{O}(1/\epsilon^2)$ on the RHS of the bound derived above, but still gives us that the cost of the optimal policy is bounded above by a constant multiple of the cost of the Bayesian policy. We give a detailed proof of this in the following section.

\subsubsection{Proof for the general case}

 Now we provide the proof for the general case discussed above. To be specific, we assume the initial jump to be bounded: $p_a, p_b \in [m_p, M_p]$. We discretize the interval $[m_p, M_p]$ up to $P := (M_p - m_p) / \epsilon$ pieces, with each piece length $\epsilon$. We require both the external price and the chosen actions at each step to be one of the endpoints of these pieces.
 This is a reasonable assumption since all market makers discretize the prices in form of \textit{ticks}, and computation is performed only in the multiples of the tick sizes over a finite interval. 

Even though the external price space gets more complicated, we can still analyze the expected cost and risk of each action, where the actions are divided into four types.

\begin{itemize}
    \item Type 1: $p_a = p_b = p_{ext}$
    \item Type 2: $p_a > p_{ext} > p_b$
    \item Type 3: $p_a \geq p_b > p_{ext}$
    \item Type 4: $p_{ext} > p_a \geq p_b$
\end{itemize}

However, the risk and cost are also related to the difference between $p_a$ and $p_b$. Thus, we use $k_{i, d}$ to denote the number of actions of type $i$ with $p_a - p_b = \lambda \epsilon$. Now we are ready to compute the risk and cost function.

\begin{lemma}
\label{lem: risk bound}
For any policy $\pi$, the risk function $R^\pi$ satisfies the following property:

\begin{align*}
R^\pi & \geq \mathbb{E}_{\pi} \left[\sum_{\lambda =1}^{P} \frac{\lambda^2 \epsilon^2 }{2} k_{2, \lambda}  + \sum_{\lambda =0}^{P} (\lambda^2 + 1) \epsilon^2 (k_{3, \lambda} + k_{4, \lambda})\right], \\
R^\pi & \leq \mathbb{E}_{\pi} \left[\sum_{\lambda=1}^{P} \lambda^2 \epsilon^2 k_{2, \lambda} + \sum_{\lambda =0}^{P} 2P^2 \epsilon^2 (k_{3, \lambda} + k_{4, \lambda}) \right].
\end{align*}

\end{lemma}

\begin{proof}
The proof is to check the risk of each type of actions.

Type 1 actions have no contributions to the risk, so we started from type 2 actions.

For a type 2 action $p_a - p_b = \lambda \epsilon$, the expected risk must be at least $\lambda^2 \epsilon^2 / 2$, and at most $\lambda^2 \epsilon^2$.

For a type 3 or type 4 action, the expected risk must be at least $(\lambda^2 + 1)\epsilon^2$, and at most $2P^2 \epsilon^2$.

The theorem is derived by simply summing them up.

$\qed$
\end{proof}

\begin{lemma}
\label{lem: cost bound}
For any policy $\pi$, the (modified) cost function $\bar{J}^\pi$ satisfies the following property:

\begin{align*}
\bar{J}^\pi & \gtrsim \mathbb{E}_{\pi} \left[\sum_{\lambda =1}^{P} \mu \lambda^2 \epsilon^2 k_{2, \lambda} + \sum_{\lambda =0}^{P} (\mu \lambda^2 \epsilon^2 + H \alpha - H \alpha^2) (k_{3, \lambda} + k_{4, \lambda})\right], \\
\bar{J}^\pi & \lesssim \mathbb{E}_{\pi} \left[\sum_{\lambda =1}^{P} \mu \lambda^2 \epsilon^2 k_{2, \lambda} + \sum_{\lambda =0}^{P} (\mu \lambda^2 \epsilon^2 + H \alpha + H^2 \alpha^2) (k_{3, \lambda} + k_{4, \lambda})\right].
\end{align*}

\end{lemma}

\begin{proof}
We still decompose $n_t^2$ as 
$$ n_t^2 = \sum_{i = t - H}^t d_i^2 + 2 \sum_{i < j} d_i d_j $$
and check the contribution of each $d_i^2$ and $2d_id_j$ for action at step $i$ with $p_a^i - p_b^i = \lambda \epsilon$.

The contribution of $d_i^2$ (i.e., $\mathbb{E}[d_i^2]$) for type 1, 2, 3, and 4 are $H(1 - \alpha), H(1 - \alpha), H, H$ respectively. 

The analysis of contributions of $2d_id_j$ is more complicated. Note that $E[d_i d_j] > 0 (i < j)$ if and only if $j - i \leq H$ and $i, j$ are both type 3 or type 4 actions. Consider an time interval $[t, t + H]$, we assume the number of type 3 actions is $n_{t, 3}$ and the number of type 4 actions is $n_{t, 4}$. The contribution of $2d_id_j$ can then be calculated as 

\begin{align*}
\sum_{t \leq i < j \leq t + H} \mathbb{E}[2 d_i d_j] & = n_{t, 3} (n_{t, 3} - 1) \alpha^2 + n_{t, 4} (n_{t, 4} - 1) \alpha^2 - 2 n_{t, 3} n_{t, 4} \alpha^2 \\
& = \left((n_{t, 3} - n_{t, 4})^2 - (n_{t, 3} + n_{t, 4})\right) \alpha^2.
\end{align*}

By some basic inequalities we have

\begin{align*}
-(n_{t, 3} + n_{t, 4}) \alpha^2 \leq \sum_{t \leq i < j \leq t + H} \mathbb{E}[2 d_i d_j] \leq (n_{t, 3} + n_{t, 4})^2 \alpha^2 \leq H(n_{t, 3} + n_{t, 4}) \alpha^2.
\end{align*}

Observe that the sum of $n_{t, 3}$ (resp. $n_{t, 4}$) over all possible $t$ is exactly $H \sum_{\lambda} k_{3, \lambda}$ (up to some constants when $t$ is smaller than $H$ or larger than $T - H$), so we have 

\begin{align*}
\sum_t \sum_{t \leq i < j \leq t + H} \mathbb{E}[2 d_i d_j] & \gtrsim -H \alpha^2 \sum_{\lambda=0}^P (k_{3, \lambda} + k_{4, \lambda}), \\
\sum_t \sum_{t \leq i < j \leq t + H} \mathbb{E}[2 d_i d_j] & \lesssim H^2 \alpha^2 \sum_{\lambda=0}^P (k_{3, \lambda} + k_{4, \lambda}).
\end{align*}

As a result, we have

\begin{align*}
J^\pi & \gtrsim \mathbb{E}_{\pi}\left[H(1 - \alpha) k_{1} + \sum_{\lambda =1}^{P} (\mu \lambda^2 \epsilon^2 + H(1 - \alpha)) k_{2, \lambda} + \sum_{\lambda =0}^{P} (\mu \lambda^2 \epsilon^2 + H - H \alpha^2) (k_{3, \lambda} + k_{4, \lambda})\right], \\
J^\pi & \lesssim \mathbb{E}_{\pi}\left[H(1 - \alpha) k_{1} + \sum_{\lambda =1}^{P} (\mu \lambda^2 \epsilon^2 + H(1 - \alpha)) k_{2, \lambda} + \sum_{\lambda =0}^{P} (\mu \lambda^2 \epsilon^2 + H + H^2 \alpha^2) (k_{3, \lambda} + k_{4, \lambda})\right].
\end{align*}

Subtracting the inequality by $HT(1 - \alpha)$ proves the theorem.

$\qed$
\end{proof}

Now we derive the final bound between $R^{\pi^*}$ and $R^{\pi_{ref}}$. Before the proof, we would like to introduce an auxiliary lemma.

\begin{lemma}
\label{lem: aux}

Given three sequence of positive real numbers $a_i, b_i, x_i (i \in [n])$ with $\sum_{i = 1}^n a_i x_i \leq c_0$ for some constant $c_0$, it holds that

$$ \sum_{i=1}^n b_i x_i \leq \beta \cdot \max_i \frac{b_i}{a_i}. $$
\end{lemma}

According to Lemma \ref{lem: risk bound}, we have 

\begin{align*}
\mathbb{E}_{\pi_{ref}} \left[\sum_{\lambda =1}^{P} \frac{\lambda^2 \epsilon^2 }{2} k_{2, \lambda}  + \sum_{\lambda =0}^{P} (\lambda^2 + 1) \epsilon^2 (k_{3, \lambda} + k_{4, \lambda})\right] \leq R^{\pi_{ref}}.
\end{align*}

By Lemma \ref{lem: aux} it holds that

\begin{align*}
\mathbb{E}_{\pi_{ref}} \left[\sum_{\lambda =1}^{P} \mu \lambda^2 \epsilon^2 k_{2, \lambda} + \sum_{\lambda =0}^{P} (\mu \lambda^2 \epsilon^2 + H \alpha + H^2 \alpha^2) (k_{3, \lambda} + k_{4, \lambda})\right] \leq \max(2\mu, \frac{H\alpha (1  + H\alpha)}{\epsilon^2}) \cdot R^{\pi_{ref}}. 
\end{align*}

Therefore, 
\begin{align*}
\bar{J}^{\pi^*} \leq \bar{J}^{\pi_{ref}} \lesssim \max(2\mu, \frac{H\alpha (1  + H\alpha)}{\epsilon^2}) \cdot R^{\pi_{ref}}
\end{align*}
according to Lemma \ref{lem: cost bound}.

Finally, we use Lemma \ref{lem: aux} and Lemma \ref{lem: risk bound} again to obtain 

\begin{align*}
R^{\pi^*} & \leq \mathbb{E}_{\pi} \left[\sum_{\lambda=1}^{P} \lambda^2 \epsilon^2 k_{2, \lambda} + \sum_{\lambda =0}^{P} 2P^2 \epsilon^2 (k_{3, \lambda} + k_{4, \lambda}) \right], \\
& \lesssim \max(\frac{1}{\mu}, \frac{2P^2\epsilon^2}{H\alpha(1 - \alpha)}) \cdot \max(2\mu, \frac{H\alpha (1  + H\alpha)}{\epsilon^2}) \cdot R^{\pi_{ref}}.
\end{align*}

Therefore, the following bound holds as long as $H \alpha (1 - \alpha) / 2P^2 \epsilon^2 \leq \mu \leq H\alpha (1 + H\alpha) / 2\epsilon^2$:

\begin{align*}
R^{\pi^*} & \lesssim \frac{2P^2 (1 + H \alpha)}{1 - \alpha} \cdot R^{\pi_{ref}} = \frac{2(M_p - m_p)^2 (1 + H \alpha)}{(1 - \alpha) \epsilon^2} \cdot R^{\pi_{ref}}.
\end{align*}
\newpage

\section{Empirical results}\label{app:sim_res}

\subsection{Results referenced in main paper}

In this section we present the empirical results referred in the main paper, particularly in \prettyref{sec:sim_res}. Additional results have been shown in \prettyref{app:sim_res}.

\subsection{General trader behavior } 

 The model of trader behavior as outlined in \prettyref{sec:model} is restrictive, in the sense that real traders lie on a spectrum of ``informedness'' about the external price, instead of being purely informed or uninformed. To remedy this, we modify the environment to have traders that see a noisy version of the hidden price. This kind of trader behavior is supported by the data collected from on-chain exchanges such as Uniswap v3 (\prettyref{fig:cex_price_dev}). That is, every trader sees an observation $p_{trader}^t = p_{ext}^t + \eta_t$ where $\eta_t$ is i.i.d noise (note that the model in \prettyref{sec:model} is just a special case of this). Then, the trader buys if $p_{trader}^t > p_a^t$, sells if $p_{trader}^t < p_b^t$ and does neither otherwise. We experimented with three types of noise - additive Gaussian, additive Laplace, and log-Normal. 

Because the model to be estimated becomes more complex, we replace the $Q-$table with a neural network \cite{mnih2013playing}. The neural network takes the trade history $(d_{t-H},\cdots,d_t)$ as a vector input instead of treating the sum ($n_t=\sum_{\tau=t-H}^t d_{\tau}$) of the trade history as a scalar input. Using this approach , we found that the market maker could handle even more sophisticated forms of trader behavior (\prettyref{fig:dqn} in Appendix \prettyref{app:sim_res}).

We observe that the algorithm trained on any one form of them is robust to a change in underlying trader distribution. In particular, \prettyref{fig:bayesian_trader} shows the ask and bid prices for an optimal Bayesian trader as per Algorithm \prettyref{alg:bayes}. We then train the DQN on traders with Gaussian noise, which reach spread and mid-price deviation values similar to the Bayesian case, as shown in \prettyref{fig:gaussian_trader}. We observe that the same agent then effectively tracks the external price when the types of noise that traders see is changed to Laplace and log-Normal in \prettyref{fig:laplace_trader} and \prettyref{fig:geomgaussian_trader} respectively.

\begin{figure}[hbt!]

\begin{subfigure}{0.30\textwidth}
  \centering
\hspace*{-0.4in}
 \includegraphics[width=1.25\textwidth]{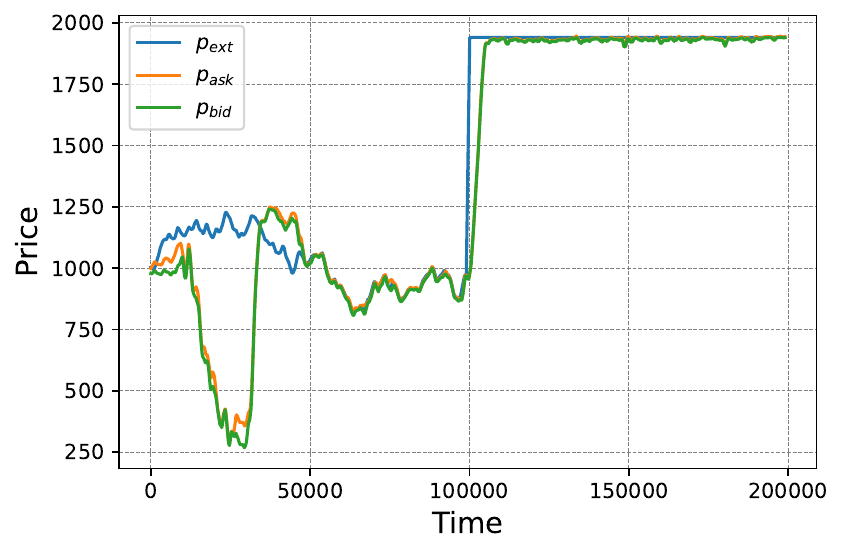}
  \caption{$p_{ask}$, $p_{bid}$ and $p_{ext}$}
  \label{fig:jump_ab}
\end{subfigure}
\hfill
\begin{subfigure}{0.30\textwidth}
  \centering 
  \hspace*{-0.25in}
 \includegraphics[width=1.25\textwidth]{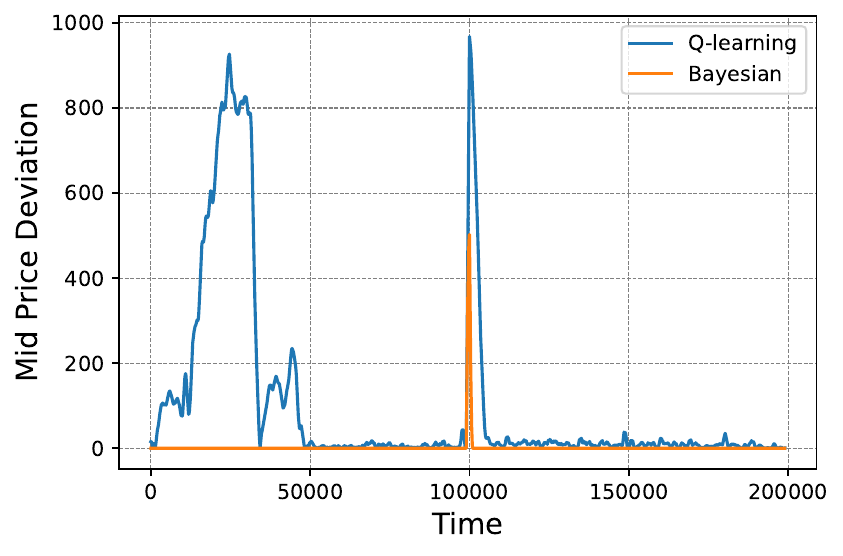} 
  \caption{Mid price deviation}
  \label{fig:jump_mid}
\end{subfigure}
\hfill
\begin{subfigure}{0.30\textwidth}
  \centering
  \hspace*{-0.1in}
  \includegraphics[width=1.25\textwidth]{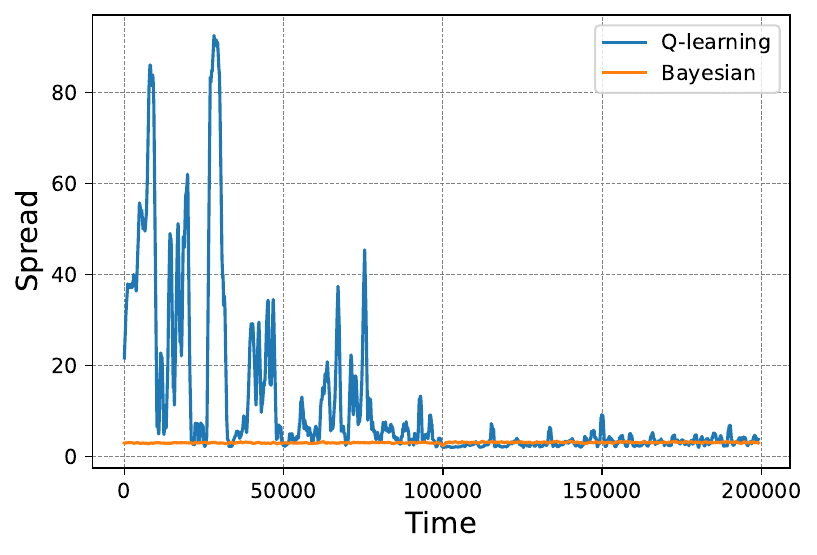}
  \caption{Spread}
  \label{fig:jump_sp}
\end{subfigure}
\caption{The conjectured reward for the model-free algorithm trains the agent to track the external hidden price, even in the presence of large price jumps. Figure reference in \prettyref{sec:sim_res}.}
\label{fig:jump}
\end{figure}

\begin{figure}[hbt!]
\begin{subfigure}[t]{0.45\textwidth}
  \centering 
  \hspace*{-0.25in}
 \includegraphics[width=1.2\textwidth]{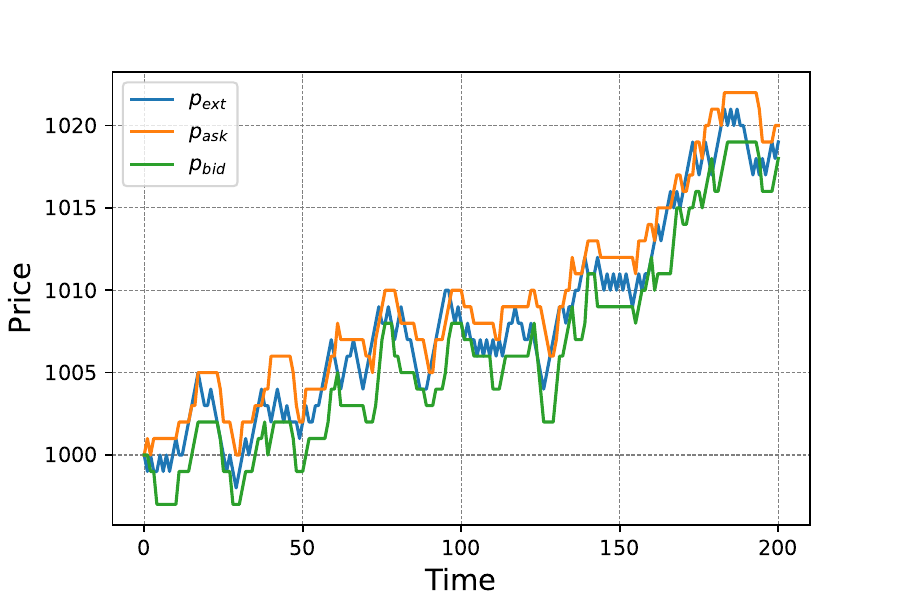} 
  \caption{Using the optimal Bayesian algorithm to track external price - with known underlying model parameters}
  \label{fig:bayesian_trader}
\end{subfigure}
\hfill
\begin{subfigure}[t]{0.45\textwidth}
  \centering
\hspace*{-0.25in}
 \includegraphics[width=1.2\textwidth]{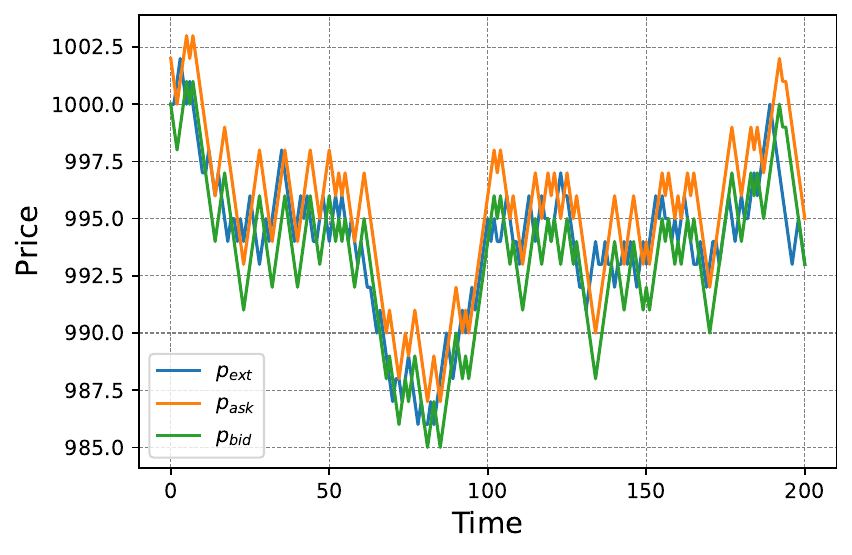}
  \caption{Trader observes external price through Gaussian noise. Market Maker based on a DQN is trained on this trader behavior - tracking is achieved with similar average spread and loss as the Bayesian model}
  \label{fig:gaussian_trader}
\end{subfigure}
\hfill
\begin{subfigure}{0.45\textwidth}
  \centering
\hspace*{-0.25in}
 \includegraphics[width=1.2\textwidth]{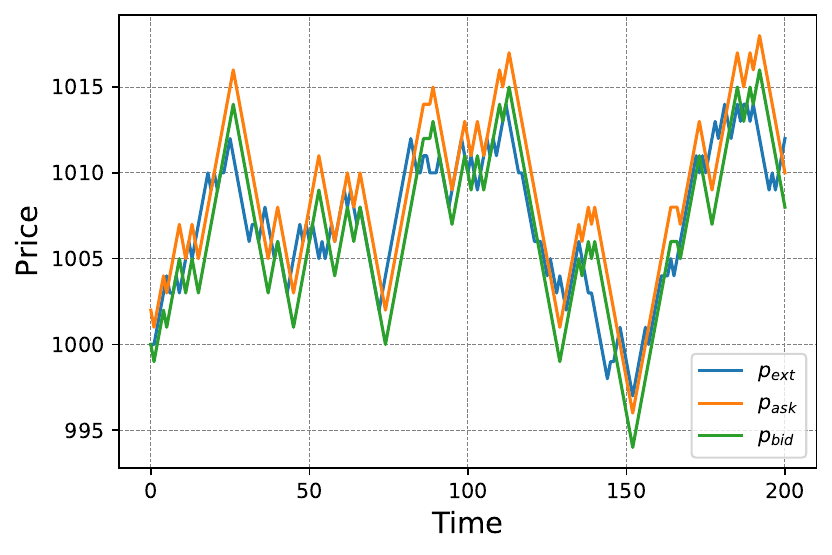}
  \caption{Trader observes external price through Laplace noise. Market maker tracks price effectively \textit{without any prior training} on this noise model.}
  \label{fig:laplace_trader}
\end{subfigure}
\hfill
\begin{subfigure}{0.45\textwidth}
  \centering 
  \hspace*{-0.25in}
 \includegraphics[width=1.2\textwidth]{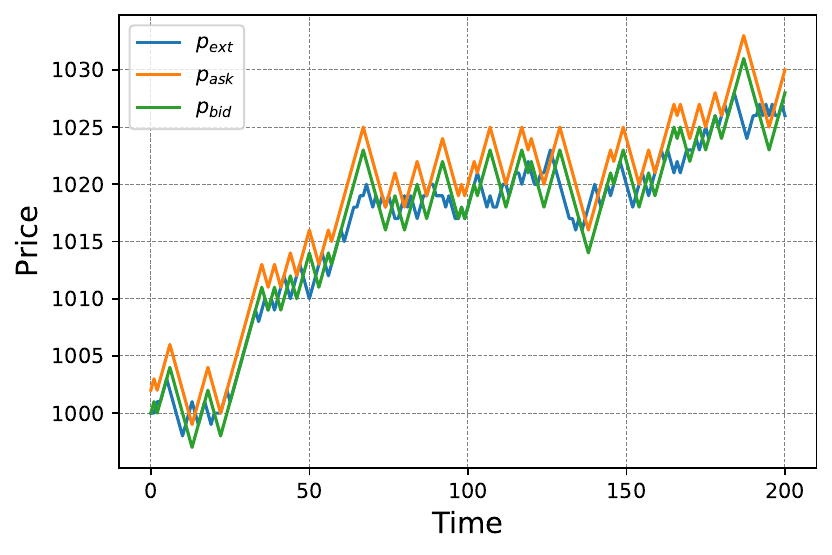} 
  \caption{Trader observes external price through log-normal noise. Market maker tracks price effectively \textit{without any prior training} on this noise model.}
  \label{fig:geomgaussian_trader}
\end{subfigure}
\caption{Algorithm \prettyref{alg:qt} is robust to changes in underlying trader behavior. Figure referenced in \prettyref{sec:sim_res}.}
\label{fig:dqn}
\end{figure}

\newpage

\subsection{Additional results}\label{app:additional}

In this section, we present results related to the ones presented in the main paper. This includes:
\begin{itemize}
\item Performance of Algorithms \prettyref{alg:bayes} and \prettyref{alg:qt} in response to different market scenarios
\item Monetary loss comparisons between Algorithms \prettyref{alg:bayes}, \prettyref{alg:qt} and the one presented in \cite{chan2001}.
\end{itemize}

\begin{figure}[hbt!]
\begin{subfigure}{0.30\textwidth}
  \centering
\hspace*{-0.4in}
 \includegraphics[width=1.25\textwidth]{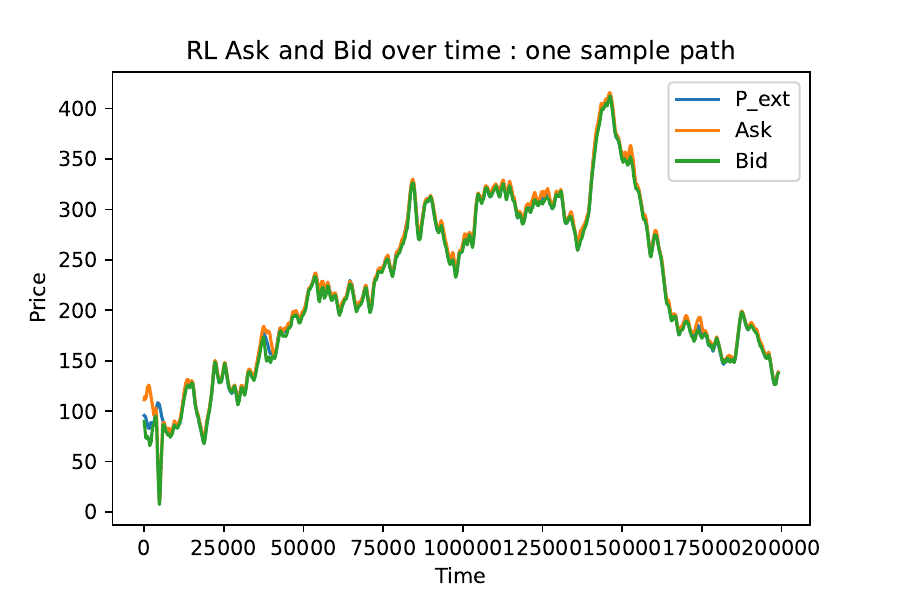}
  \caption{$p_{ask}$, $p_{bid}$ and $p_{ext}$}
  \label{fig:vanilla_ab_a}
\end{subfigure}
\hfill
\begin{subfigure}{0.30\textwidth}
  \centering 
  \hspace*{-0.25in}
 \includegraphics[width=1.25\textwidth]{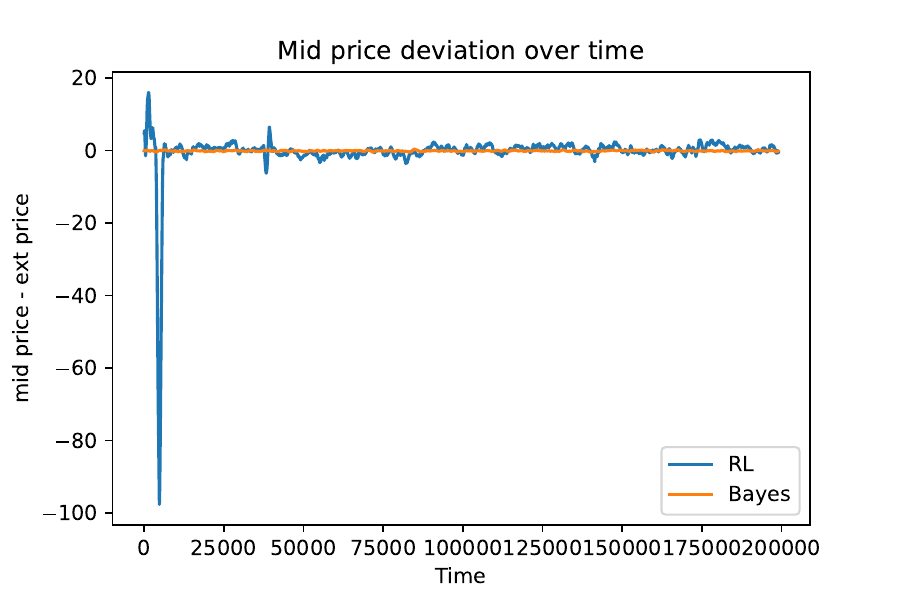} 
  \caption{Mid price deviation}
  \label{fig:vanilla_mid_a}
\end{subfigure}
\hfill
\begin{subfigure}{0.30\textwidth}
  \centering
  \hspace*{-0.1in}
  \includegraphics[width=1.25\textwidth]{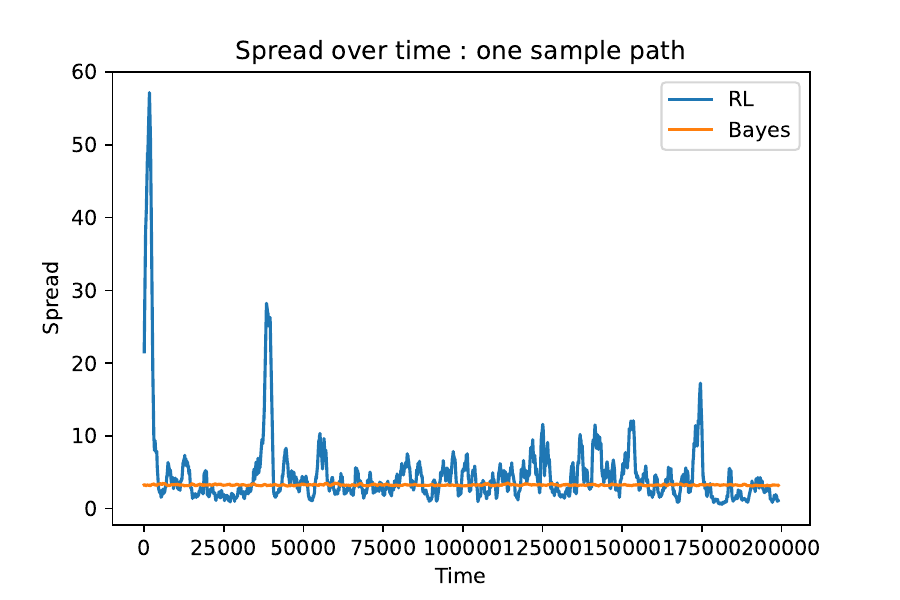}
  \caption{Spread}
  \label{fig:vanilla_sp_a}
\end{subfigure}
\caption{The conjectured reward for the model-free algorithm trains the agent to track the external hidden price, eventually approaching the performance of the optimal Bayesian algorithm.}
\label{fig:vanilla_a}
\end{figure}

\begin{figure}[hbt!]

\begin{subfigure}{0.30\textwidth}
  \centering
\hspace*{-0.4in}
 \includegraphics[width=1.25\textwidth]{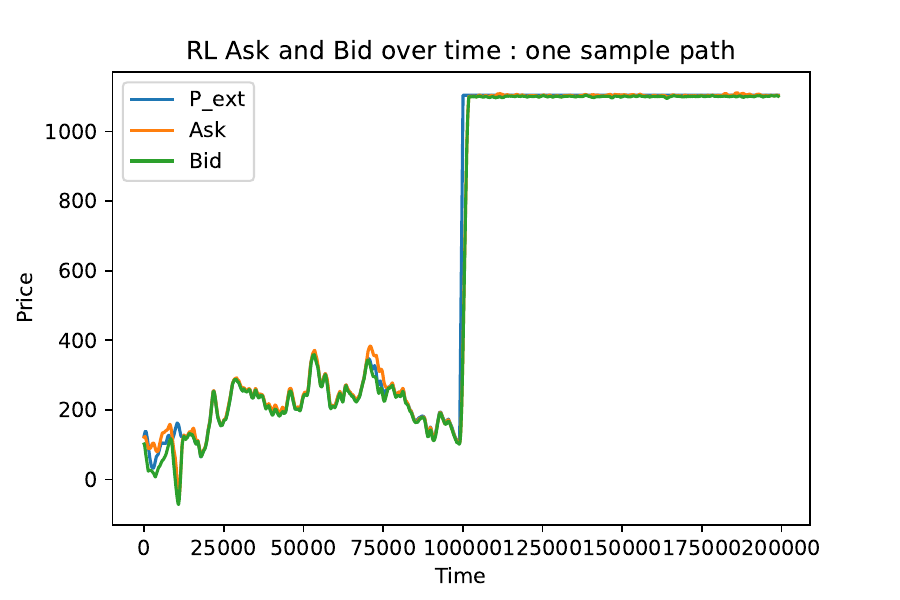}
  \caption{$p_{ask}$, $p_{bid}$ and $p_{ext}$}
  \label{fig:jump_ab_a}
\end{subfigure}
\hfill
\begin{subfigure}{0.30\textwidth}
  \centering 
  \hspace*{-0.25in}
 \includegraphics[width=1.25\textwidth]{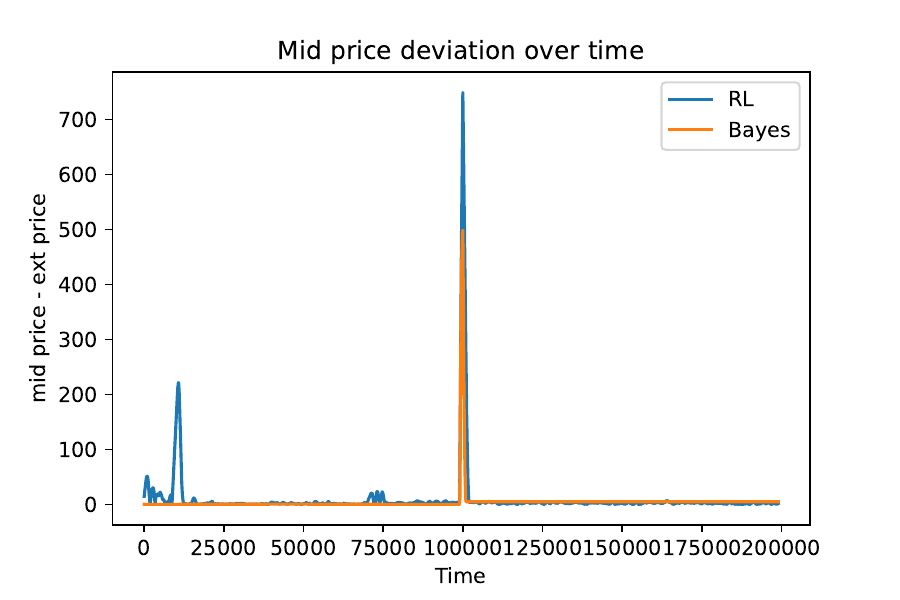} 
  \caption{Mid price deviation}
  \label{fig:jump_mid_a}
\end{subfigure}
\hfill
\begin{subfigure}{0.30\textwidth}
  \centering
  \hspace*{-0.1in}
  \includegraphics[width=1.25\textwidth]{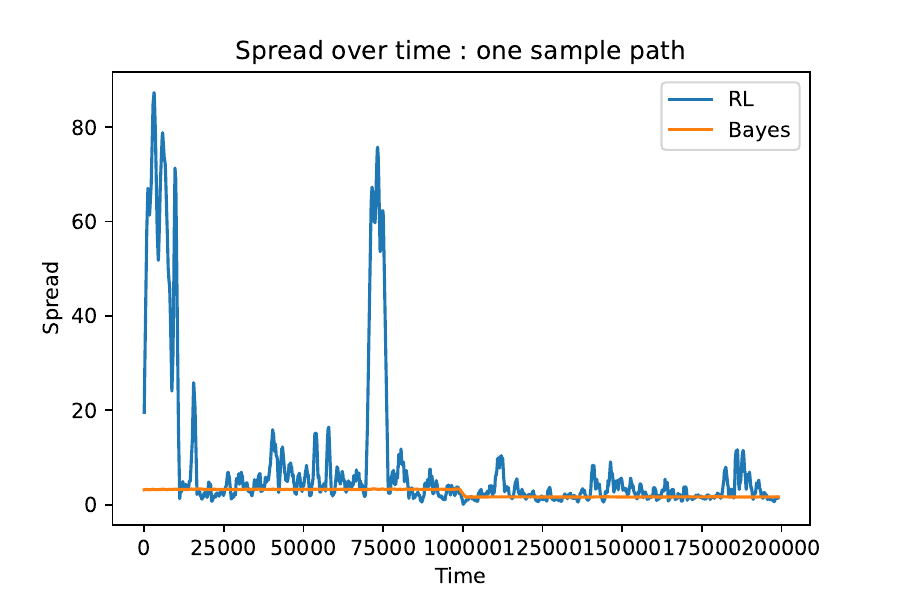}
  \caption{Spread}
  \label{fig:jump_sp_a}
\end{subfigure}
\caption{The conjectured reward for the model-free algorithm trains the agent to track the external hidden price, even in the presence of large price jumps. }
\label{fig:jump_a}
\end{figure}

\begin{figure}[hbt!]

\begin{subfigure}[t]{0.45\textwidth}
  \centering
\hspace*{-0.25in}
 \includegraphics[width=1.25\textwidth]{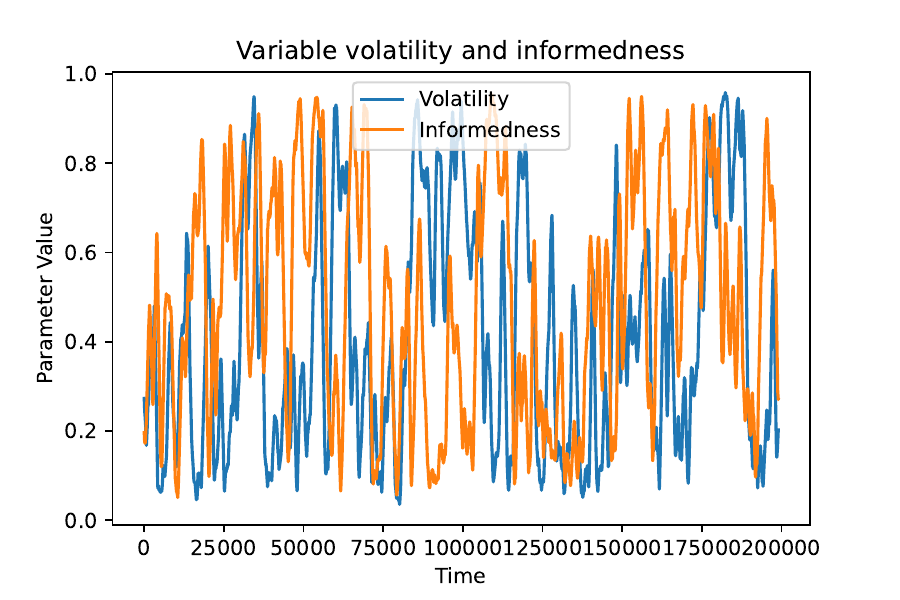}
  \caption{Erratic changes in trader informedness and price volatility}
  \label{fig:variab_param_a}
\end{subfigure}
\hfill
\begin{subfigure}[t]{0.45\textwidth}
  \centering
\hspace*{-0.25in}
 \includegraphics[width=1.25\textwidth]{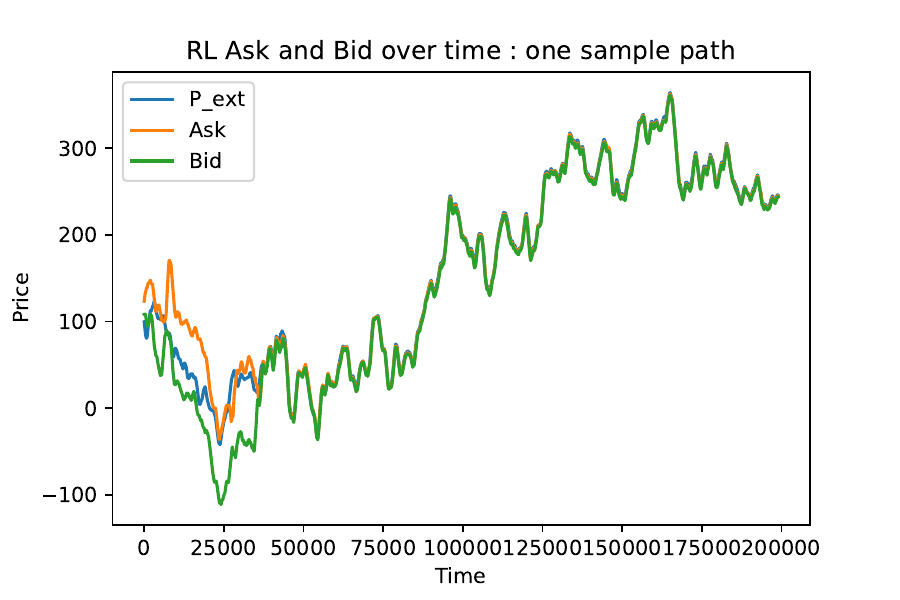}
  \caption{Ask, Bid and $p_{ext}$}
  \label{fig:variab_ab_a}
\end{subfigure}
\hfill
\begin{subfigure}[t]{0.45\textwidth}
  \centering 
  \hspace*{-0.25in}
 \includegraphics[width=1.25\textwidth]{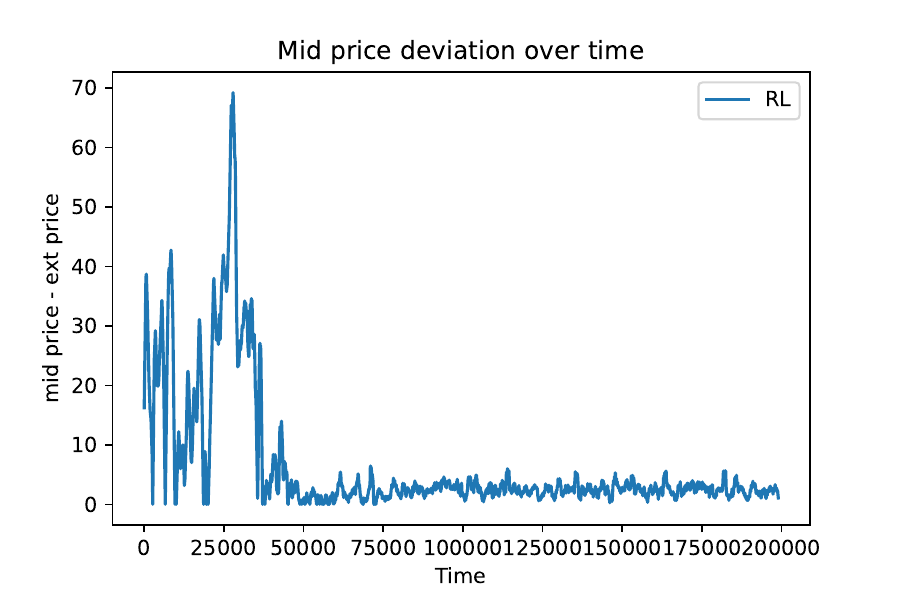} 
  \caption{Deviation from mid price}
  \label{fig:variab_mid_a}
\end{subfigure}
\hfill
\begin{subfigure}[t]{0.45\textwidth}
  \centering
  \hspace*{-0.25in}
  \includegraphics[width=1.25\textwidth]{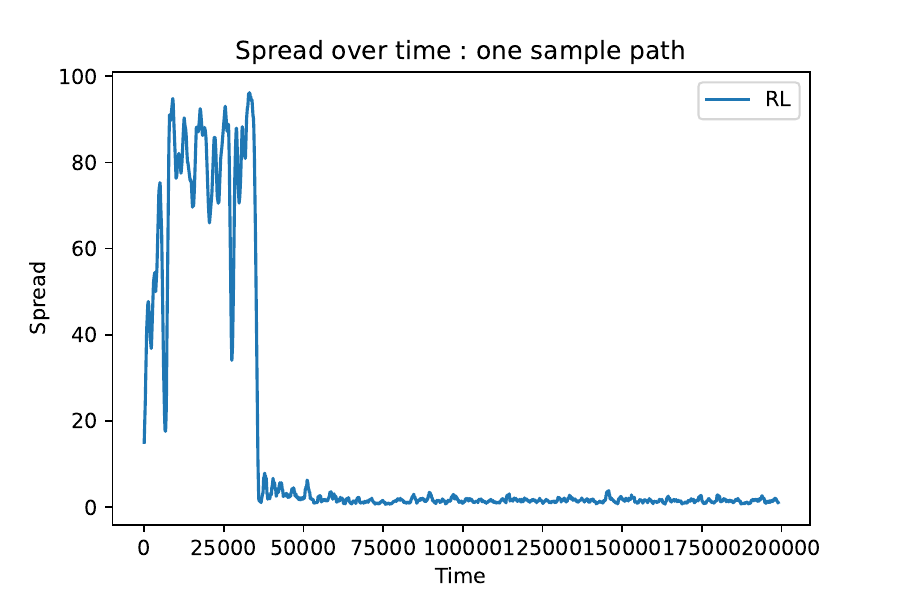}
  \caption{Spread}
  \label{fig:variab_sp_a}
\end{subfigure}
\caption{The conjectured reward trains the agent to track the external hidden price, even in the presence of erratic changes in the market conditions}
\label{fig:variable_a}
\end{figure}

\begin{figure}[hbt!]
  \centering
\hspace*{-0.4in}
 \includegraphics[width=0.75\textwidth]{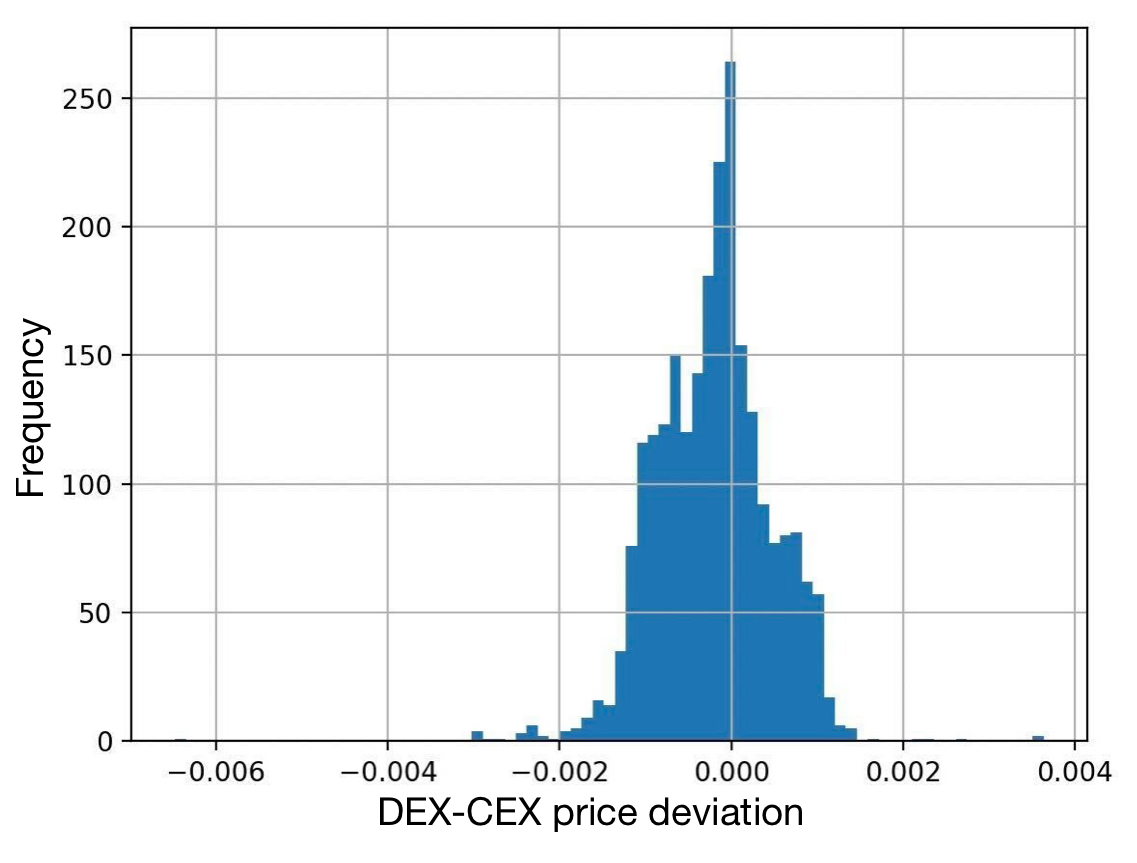}
\caption{The above histogram is that of the log deviation in price offered by Uniswap v3 (Decentralized Exchange or DEX) from Binance (Centralized Exchange or CEX) for the ETH-USDC pool. The histogram for other token pairs also has a similar plot.}
\label{fig:cex_price_dev}
\end{figure}

\begin{figure}[hbt!]

\begin{subfigure}[t]{0.45\textwidth}
  \centering
\hspace*{-0.25in}
 \includegraphics[width=1.2\textwidth]{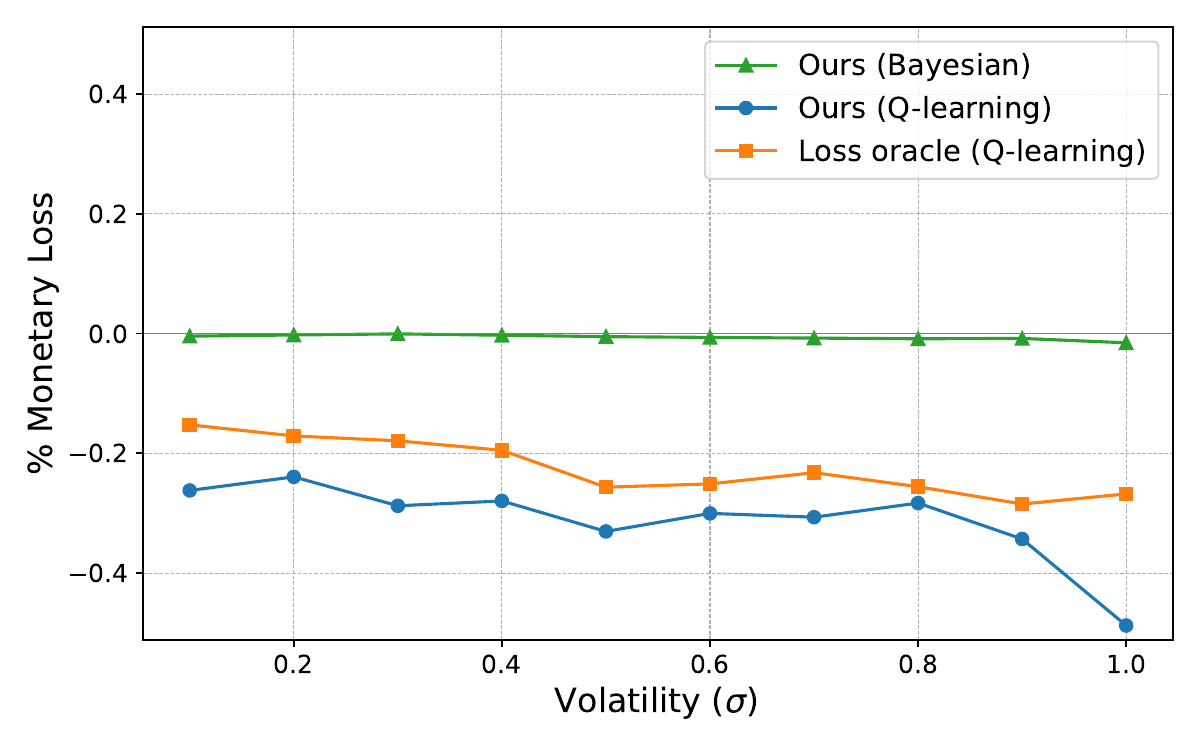}
  \caption{Percentage monetary loss per trade of our market maker is comparable with the algorithm which has access to the loss oracle}
  \label{fig:loss_avg_vs_sigma2}
\end{subfigure}
\hfill
\begin{subfigure}[t]{0.45\textwidth}
  \centering
\hspace*{-0.25in}
 \includegraphics[width=1.2\textwidth]{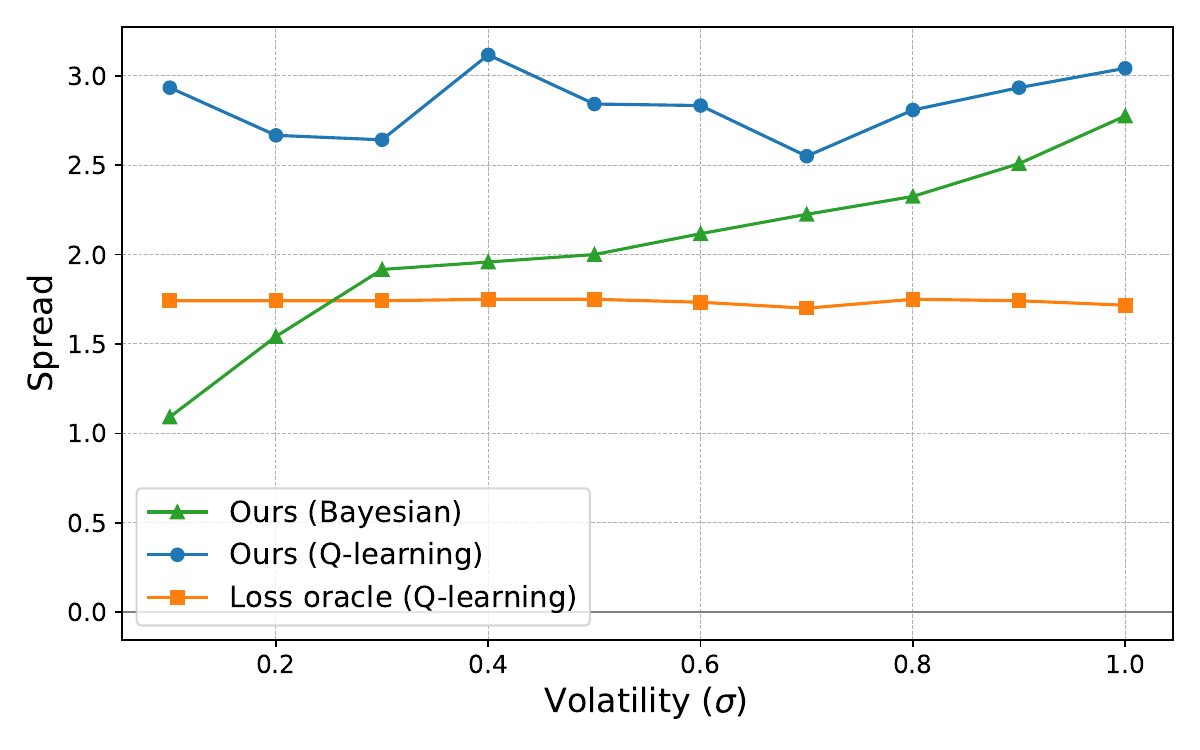}
  \caption{Larger spread observed without access to the loss oracle}
  \label{fig:spread_avg_vs_sigma2}
\end{subfigure}

\caption{Algorithm \prettyref{alg:qt} gives us comparable monetary loss per trade as running the algorithm with an oracle. The Bayesian algorithm gives loss close to zero, which is optimally efficient. The plots are against varying volatility for $\alpha=0.9$.}
\label{fig:monetary_loss2}
\end{figure}

\begin{figure}[hbt!]

\begin{subfigure}[t]{0.45\textwidth}
  \centering
\hspace*{-0.25in}
 \includegraphics[width=1.2\textwidth]{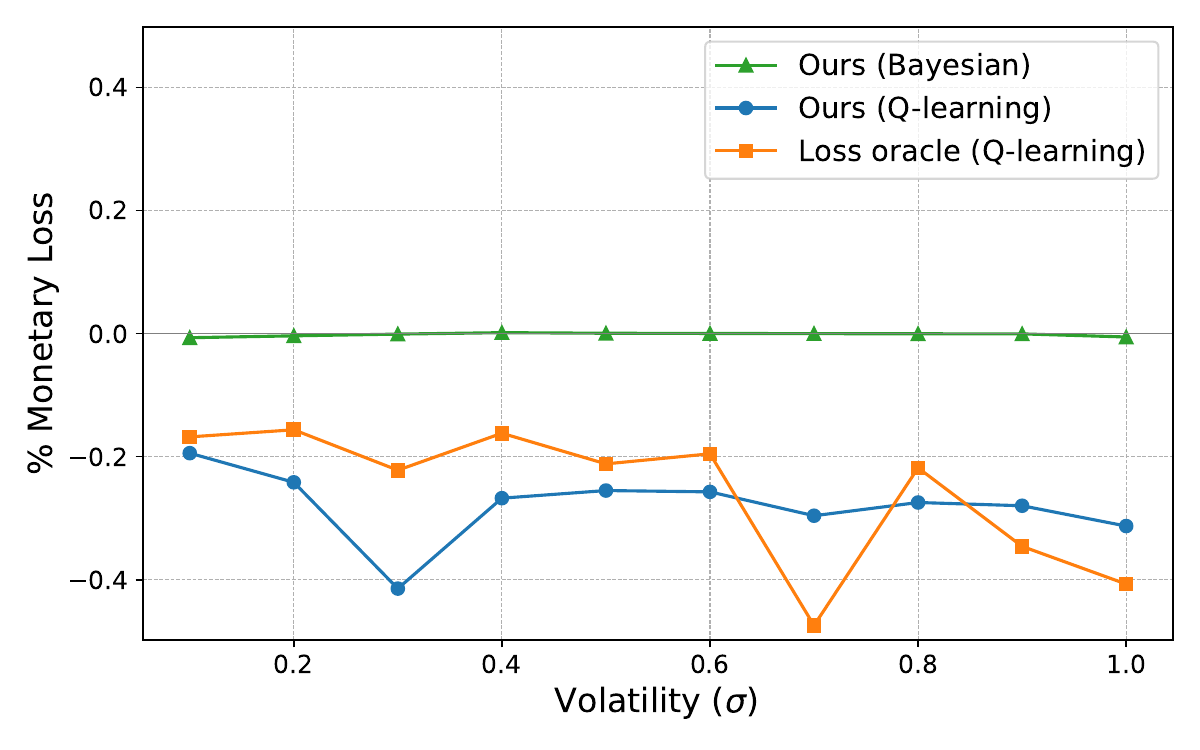}
  \caption{Percentage monetary loss per trade of our market maker is comparable with the algorithm which has access to the loss oracle}
  \label{fig:loss_avg_vs_sigma3}
\end{subfigure}
\hfill
\begin{subfigure}[t]{0.45\textwidth}
  \centering
\hspace*{-0.25in}
 \includegraphics[width=1.2\textwidth]{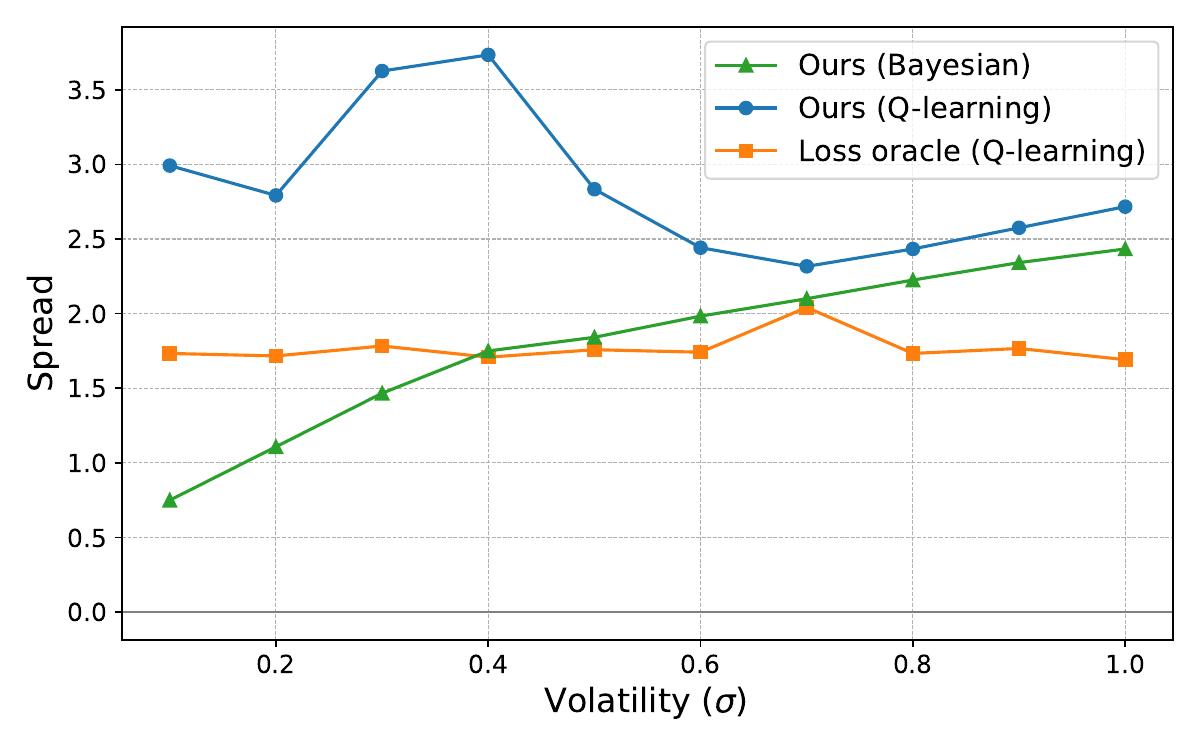}
  \caption{Larger spread observed without access to the loss oracle}
  \label{fig:spread_avg_vs_sigma3}
\end{subfigure}

\caption{Algorithm \prettyref{alg:qt} gives us comparable monetary loss per trade as running the algorithm with an oracle. The Bayesian algorithm gives loss close to zero, which is optimally efficient. The plots are against varying volatility for $\alpha=0.8$.}
\label{fig:monetary_loss3}
\end{figure}

\begin{figure}[hbt!]

\begin{subfigure}[t]{0.45\textwidth}
  \centering
\hspace*{-0.25in}
 \includegraphics[width=1.2\textwidth]{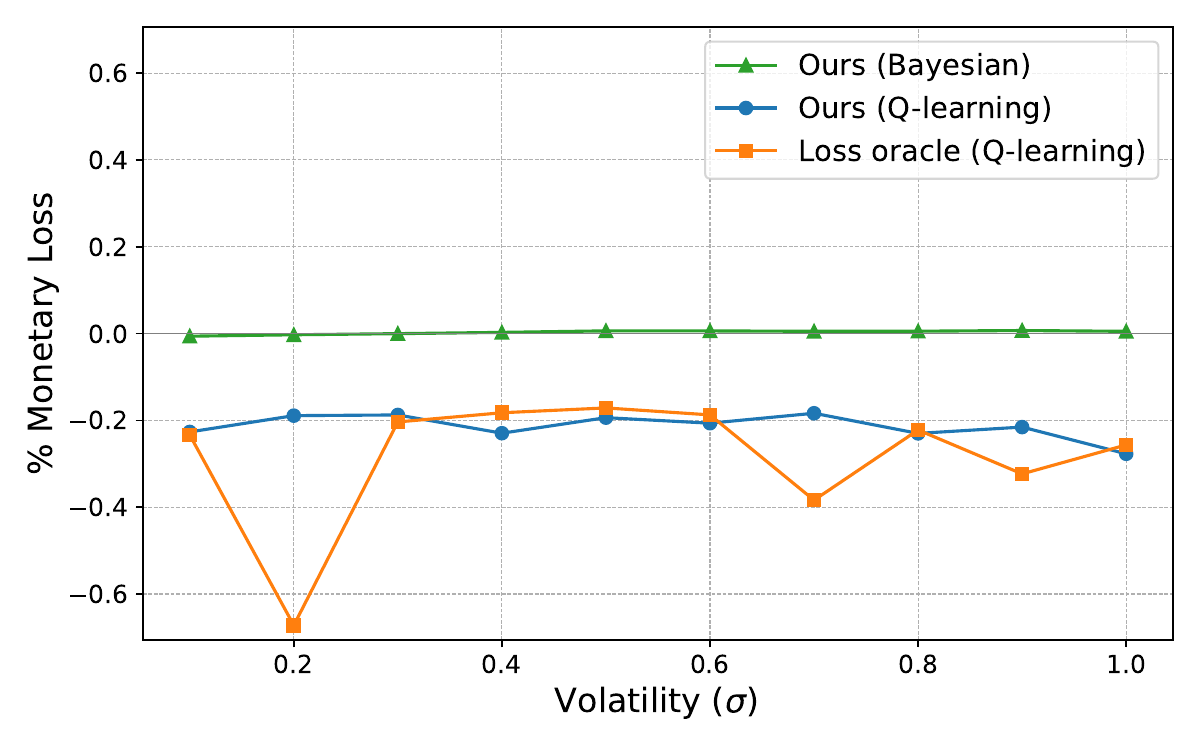}
  \caption{Percentage monetary loss per trade of our market maker is comparable with the algorithm which has access to the loss oracle}
  \label{fig:loss_avg_vs_sigma4}
\end{subfigure}
\hfill
\begin{subfigure}[t]{0.45\textwidth}
  \centering
\hspace*{-0.25in}
 \includegraphics[width=1.2\textwidth]{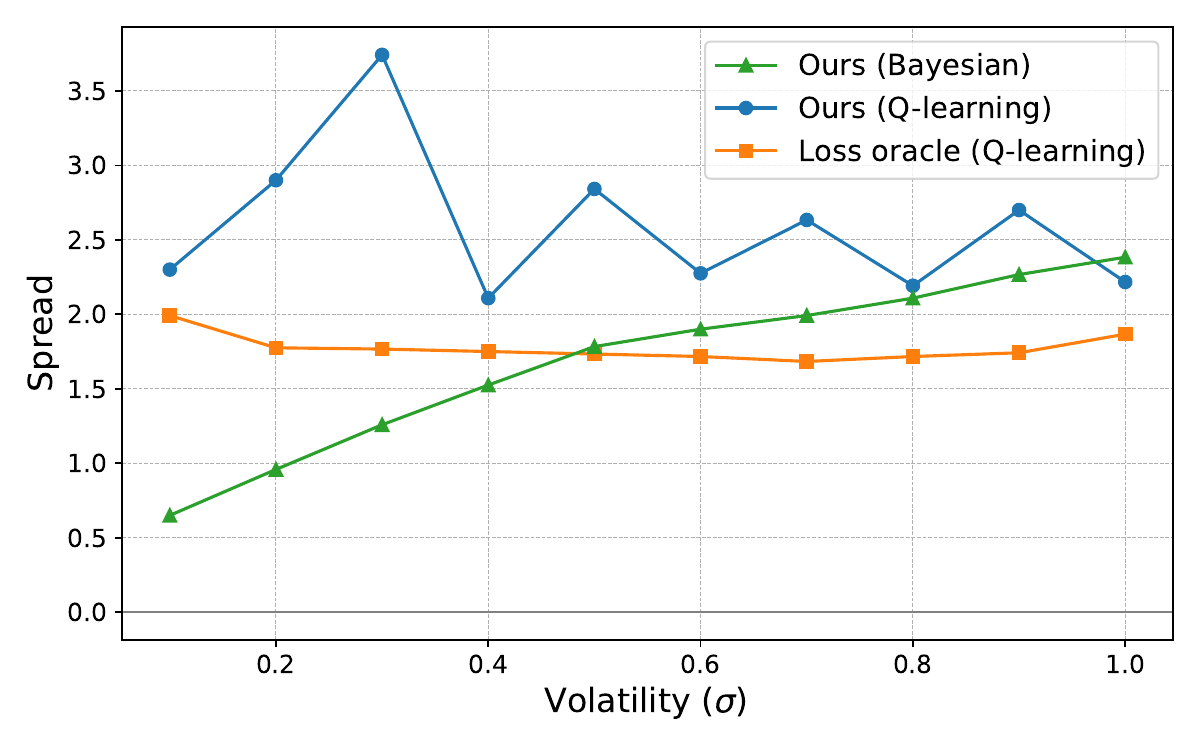}
  \caption{Larger spread observed without access to the loss oracle}
  \label{fig:spread_avg_vs_sigma4}
\end{subfigure}

\caption{Algorithm \prettyref{alg:qt} gives us comparable monetary loss per trade as running the algorithm with an oracle. The Bayesian algorithm gives loss close to zero, which is optimally efficient. The plots are against varying volatility for $\alpha=0.7$.}
\label{fig:monetary_loss4}
\end{figure}

\begin{figure}[hbt!]

\begin{subfigure}[t]{0.45\textwidth}
  \centering
\hspace*{-0.25in}
 \includegraphics[width=1.2\textwidth]{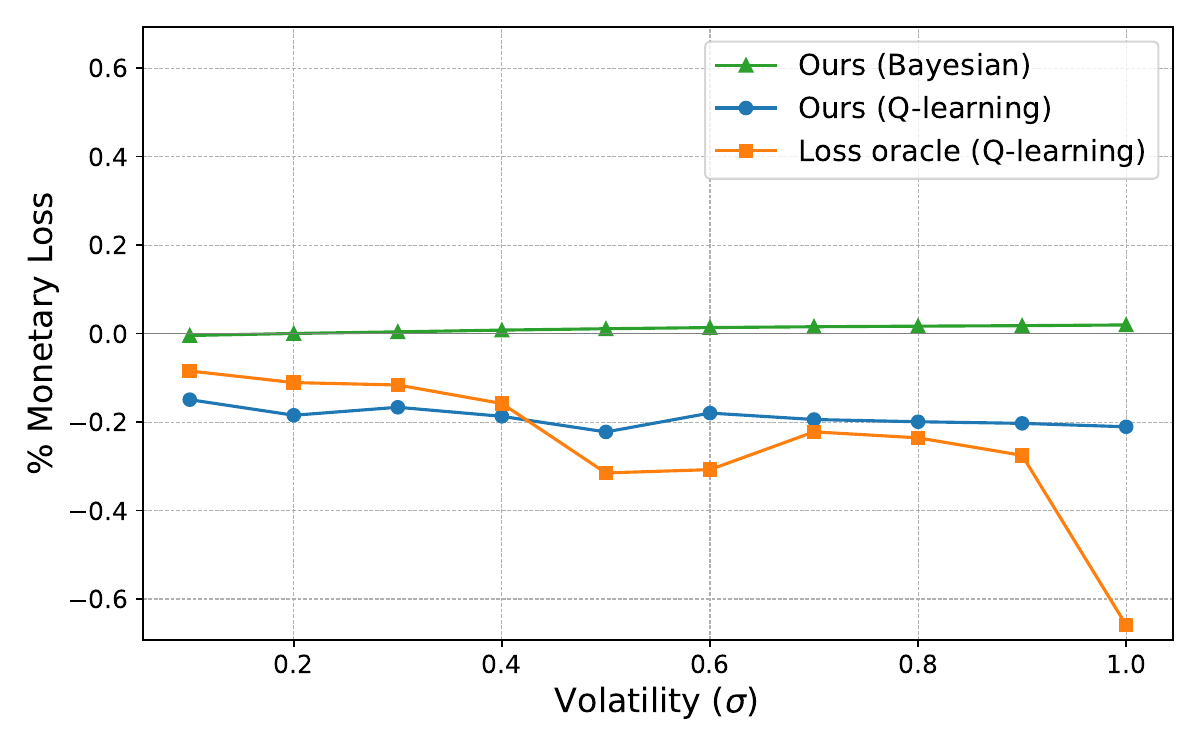}
  \caption{Percentage monetary loss per trade of our market maker is comparable with the algorithm which has access to the loss oracle}
  \label{fig:loss_avg_vs_sigma5}
\end{subfigure}
\hfill
\begin{subfigure}[t]{0.45\textwidth}
  \centering
\hspace*{-0.25in}
 \includegraphics[width=1.2\textwidth]{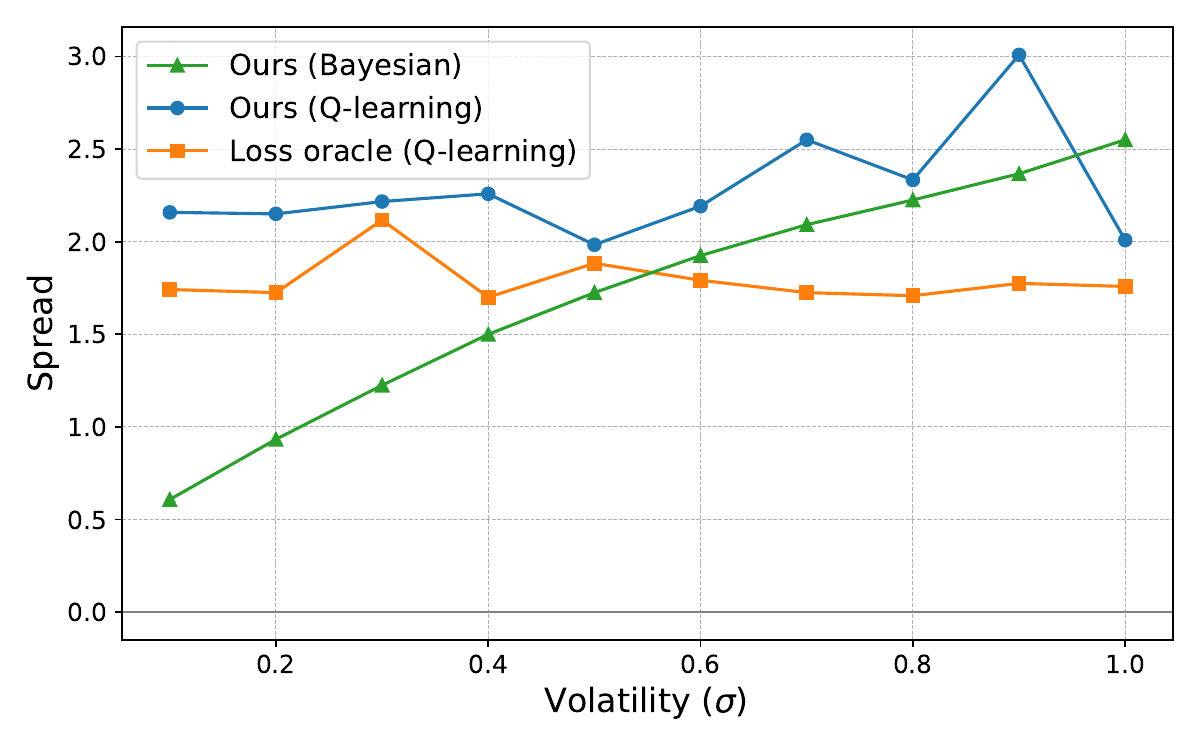}
  \caption{Larger spread observed without access to the loss oracle}
  \label{fig:spread_avg_vs_sigma5}
\end{subfigure}

\caption{Algorithm \prettyref{alg:qt} gives us comparable monetary loss per trade as running the algorithm with an oracle. The Bayesian algorithm gives loss close to zero, which is optimally efficient. The plots are against varying volatility for $\alpha=0.6$.}
\label{fig:monetary_loss5}
\end{figure}

\begin{figure}[hbt!]

\begin{subfigure}[t]{0.45\textwidth}
  \centering
\hspace*{-0.25in}
 \includegraphics[width=1.2\textwidth]{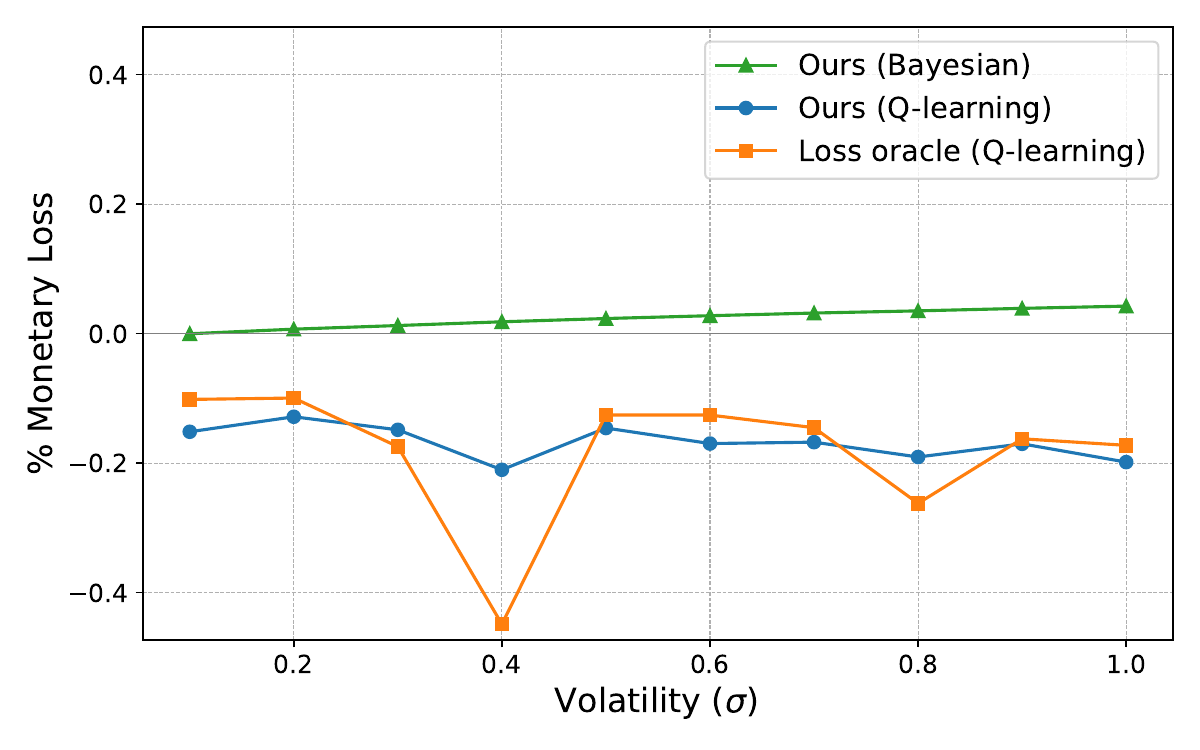}
  \caption{Percentage monetary loss per trade of our market maker is comparable with the algorithm which has access to the loss oracle}
  \label{fig:loss_avg_vs_sigma6}
\end{subfigure}
\hfill
\begin{subfigure}[t]{0.45\textwidth}
  \centering
\hspace*{-0.25in}
 \includegraphics[width=1.2\textwidth]{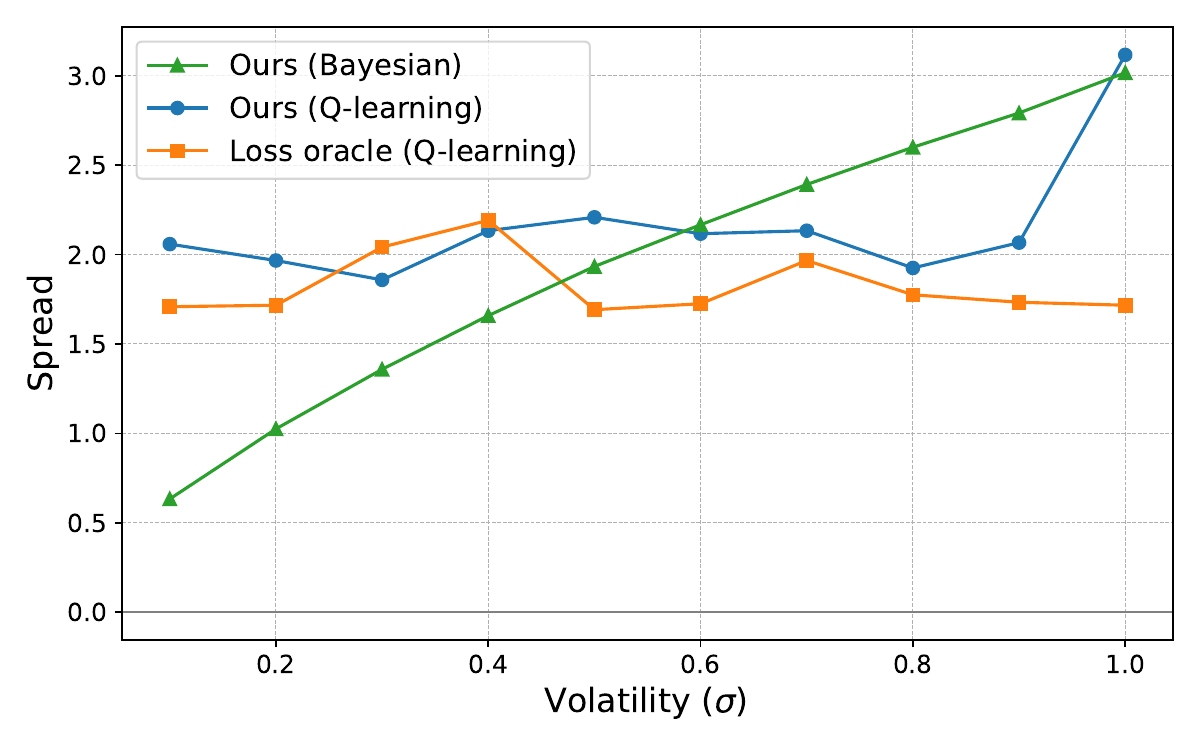}
  \caption{Larger spread observed without access to the loss oracle}
  \label{fig:spread_avg_vs_sigma6}
\end{subfigure}

\caption{Algorithm \prettyref{alg:qt} gives us comparable monetary loss per trade as running the algorithm with an oracle. The Bayesian algorithm gives loss close to zero, which is optimally efficient. The plots are against varying volatility for $\alpha=0.5$.}
\label{fig:monetary_loss6}
\end{figure}

\begin{figure}[hbt!]

\begin{subfigure}[t]{0.45\textwidth}
  \centering
\hspace*{-0.25in}
 \includegraphics[width=1.2\textwidth]{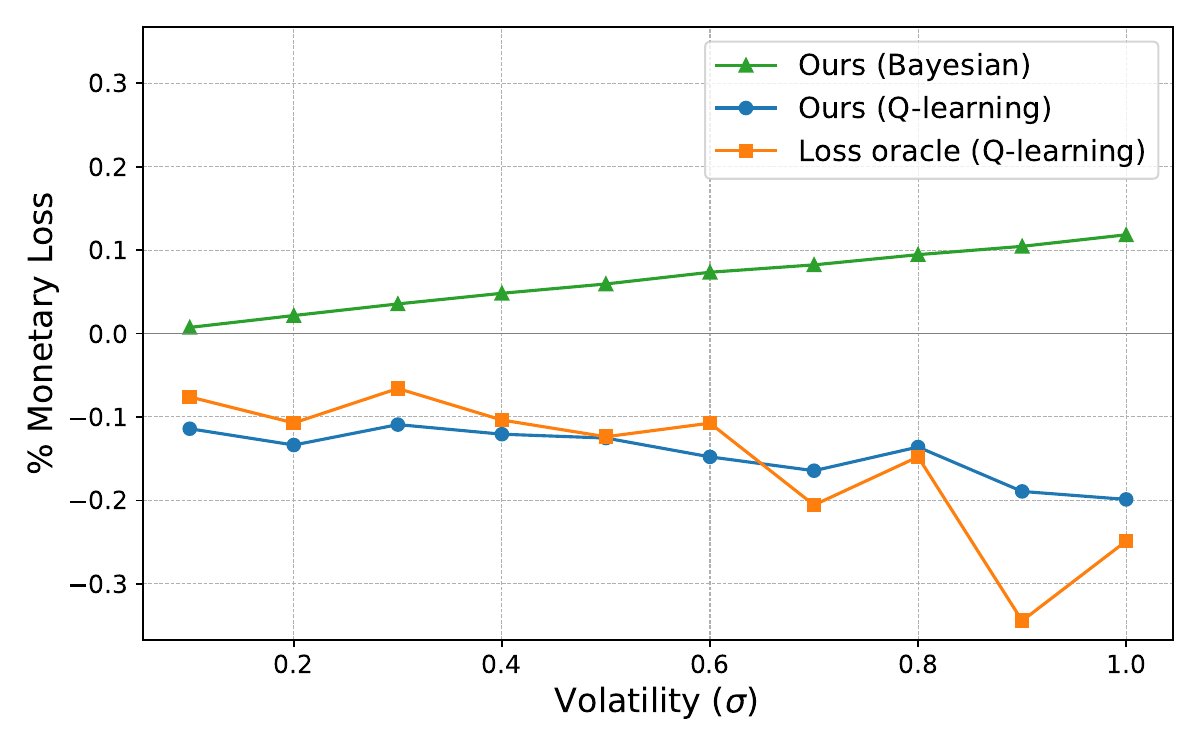}
  \caption{Percentage monetary loss per trade of our market maker is comparable with the algorithm which has access to the loss oracle}
  \label{fig:loss_avg_vs_sigma7}
\end{subfigure}
\hfill
\begin{subfigure}[t]{0.45\textwidth}
  \centering
\hspace*{-0.25in}
 \includegraphics[width=1.2\textwidth]{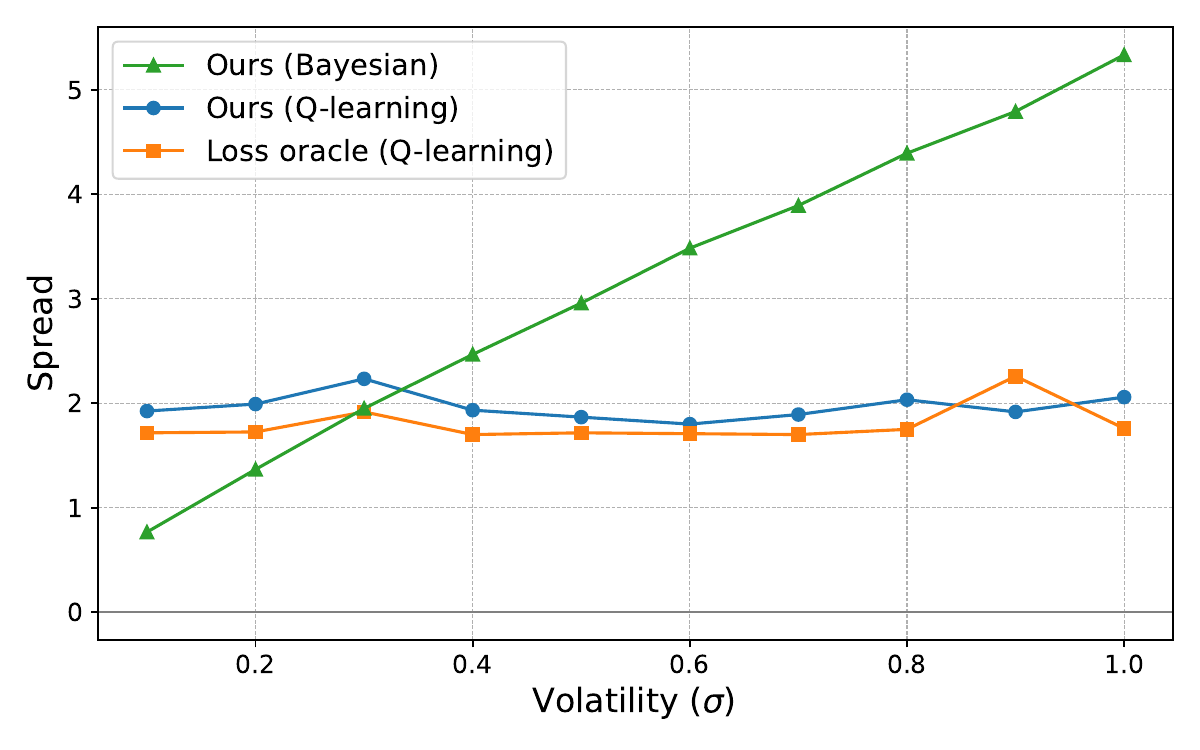}
  \caption{Larger spread observed without access to the loss oracle}
  \label{fig:spread_avg_vs_sigma7}
\end{subfigure}

\caption{Algorithm \prettyref{alg:qt} gives us comparable monetary loss per trade as running the algorithm with an oracle. The Bayesian algorithm gives loss close to zero, which is optimally efficient. The plots are against varying volatility for $\alpha=0.4$.}
\label{fig:monetary_loss7}
\end{figure}

\begin{figure}[hbt!]

\begin{subfigure}[t]{0.45\textwidth}
  \centering
\hspace*{-0.25in}
 \includegraphics[width=1.2\textwidth]{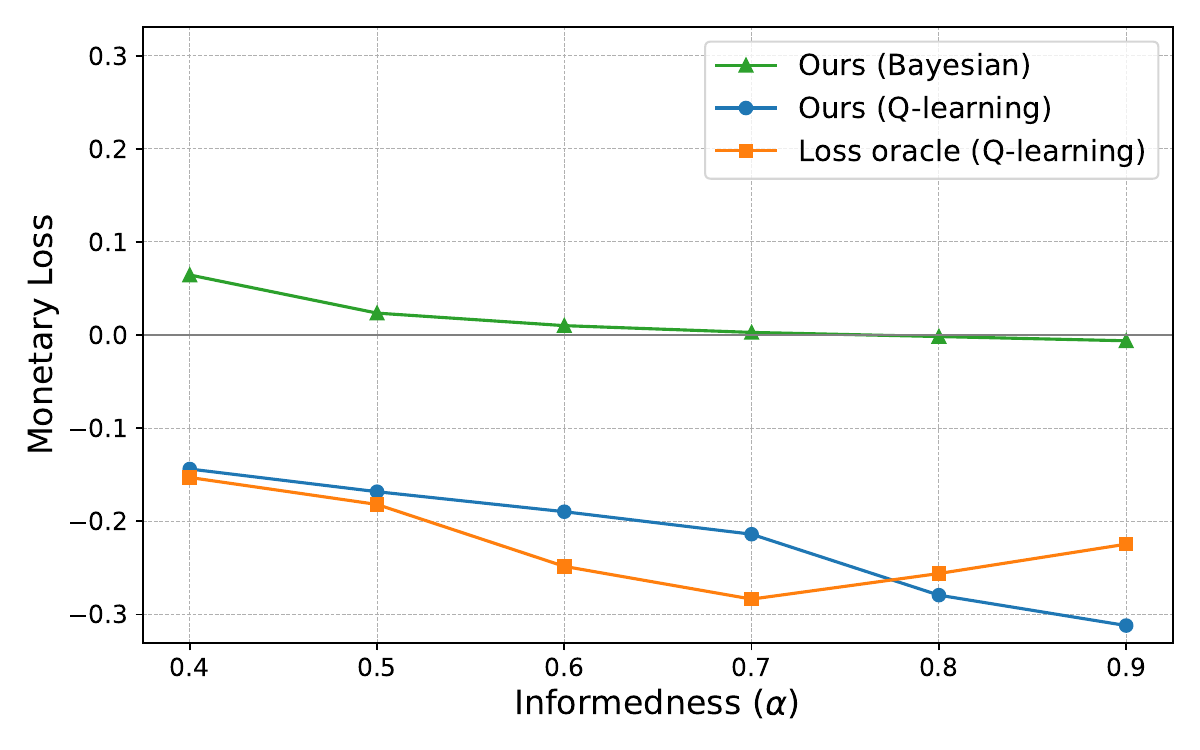}
  \caption{Percentage monetary loss per trade of our market maker is comparable with the algorithm which has access to the loss oracle}
  \label{fig:loss_avg_vs_sigma_alpha}
\end{subfigure}
\hfill
\begin{subfigure}[t]{0.45\textwidth}
  \centering
\hspace*{-0.25in}
 \includegraphics[width=1.2\textwidth]{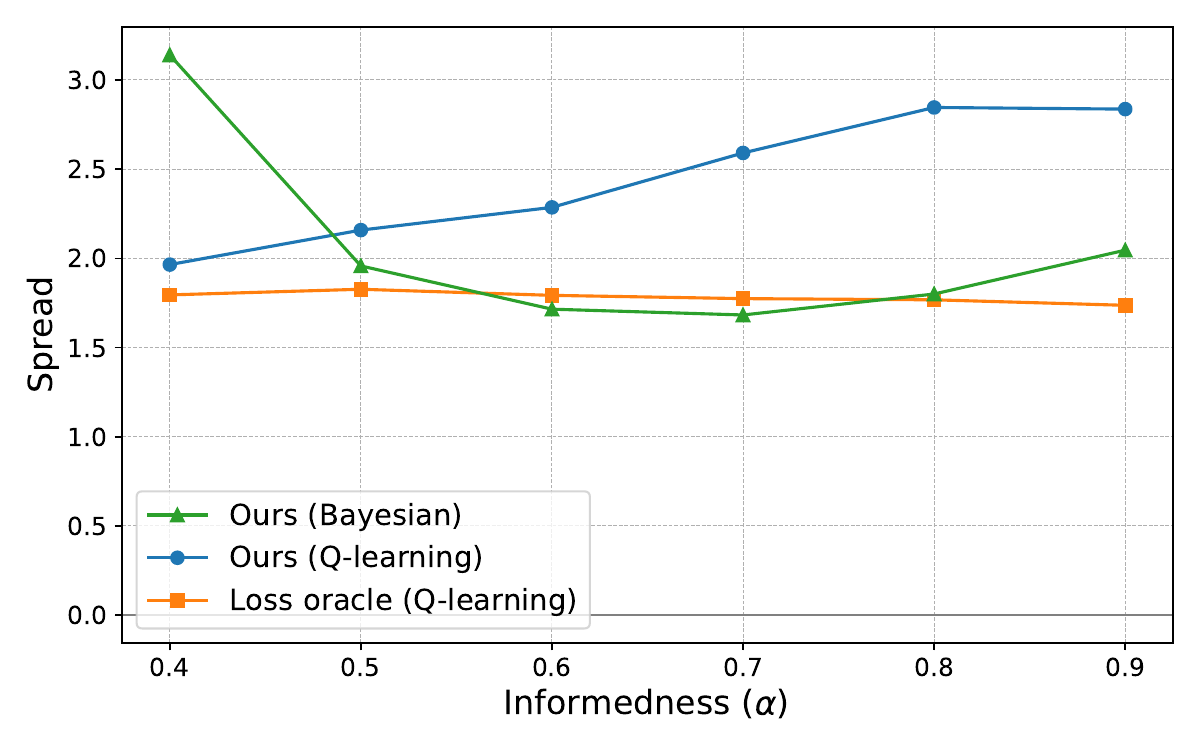}
  \caption{Larger spread observed without access to the loss oracle}
  \label{fig:spread_avg_vs_sigma_alpha}
\end{subfigure}

\caption{Algorithm \prettyref{alg:qt} gives us comparable monetary loss per trade as running the algorithm with an oracle. The Bayesian algorithm gives loss close to zero, which is optimally efficient. The plots are against varying informedness $\alpha$, and are averaged over values of volatility $\sigma$.}
\label{fig:monetary_loss_alpha}
\end{figure}